%% file: SIAM_Revised.tex
\documentclass{siamart0516}

\input{SIAM_Revised_shared}
\ifpdf
\hypersetup{
  pdftitle={\TheTitle},
  pdfauthor={\TheAuthors}
}
\fi

\begin{document}

\maketitle

\begin{abstract}
The conversion of traditional film into stereo 3D has become an important problem in the past decade.  One of the main bottlenecks is a disocclusion step, which in commercial 3D conversion is usually done by teams of artists armed with a toolbox of inpainting algorithms.  A current difficulty in this is that most available algorithms are either too slow for interactive use, or provide no intuitive means for users to tweak the output.  

In this paper we present a new fast inpainting algorithm based on transporting along automatically detected splines, which the user may edit.  Our algorithm is implemented on the GPU and fills the inpainting domain in successive shells that adapt their shape on the fly.  In order to allocate GPU resources as efficiently as possible, we propose a parallel algorithm to track the inpainting interface as it evolves, ensuring that no resources are wasted on pixels that are not currently being worked on.  Theoretical analysis of the time and processor complexity of our algorithm without and with tracking (as well as numerous numerical experiments) demonstrate the merits of the latter.

Our transport mechanism is similar to the one used in coherence transport \cite{Marz2007,Marz2011}, but improves upon it by correcting a ``kinking'' phenomenon whereby extrapolated isophotes may bend at the boundary of the inpainting domain.  Theoretical results explaining this phenomena and its resolution are presented.

Although our method ignores texture, in many cases this is not a problem due to the thin inpainting domains in 3D conversion.  Experimental results show that our method can achieve a visual quality that is competitive with the state-of-the-art while maintaining interactive speeds and providing the user with an intuitive interface to tweak the results.
\end{abstract}



\section{Introduction}

\noindent The increase in demand over the past decade for 3D content has resulted in the emergence of a multi-million dollar industry devoted to the conversion of 2D films into stereo 3D.  This is partly driven by the demand for 3D versions of old films, but additionally many current filmmakers are choosing to shoot in mono and convert in post production \cite{artOfSteroConv}.  Examples of recent films converted in whole or in part include Maleficent, Thor, and Guardians of the Galaxy \cite{Gener8Projects}.

Mathematically, 3D conversion amounts to constructing the image or video shot by a camera at the perturbed position $p+\delta p$ and orientation $O+\delta O$, given the footage at $(p,O)$.  

\vskip 2mm

\noindent {\bf Two Primary Conversion Pipelines. } There are essentially two pipelines for achieving this.  The first pipeline assumes that each frame of video is accompanied by a depth map (and hence is more applicable to footage from RGB-D cameras).  The new viewpoint is generated by ``warping'' the original footage based on the given depth map and known or estimated camera parameters - see \cite{OtherPipelineDoneWell} for an excellent recent overview.  This pipeline has applications including 3D TV and free-viewpoint rendering \cite{Compass,Compass2}.  However, it is not typically used in the movie industry - this is for a number of reasons (including, for example, the fact that much of what is being converted are old movies predating RGB-D cameras) - see \cite{CaseStudy3DConv, artOfSteroConv} for more details and discussion.

We focus in this paper on a second pipeline, which is of greater interest in film.  This pipeline does not assume that a depth map is given.  Instead it is based on teams of artists generating a plausible 3D model of the scene, reprojecting the original footage onto that model from a known or estimated camera position, and then rerendering the scene from a novel viewpoint.  Unlike the previous pipeline this one involves a step whereby teams of artists create masks for every relevant object in the original scene.  Crucially, these masks include occluded parts of objects - see Figure \ref{fig:pipeline}(c).  We go over this pipeline in detail in Section \ref{sec:pipeline}.

One thing both pipelines have in common is a hole-filling or disocclusion step whereby missing information in the form of geometry visible from $(p+\delta p, O+\delta O)$ but not from $(p,O)$ is ``inpainted''.  This step is considered one of the most technical and time consuming pieces of the pipeline \cite{artOfSteroConv}.  However, while the disocclusion step arising in the first pipeline has received a lot of attention in the literature, see for example \cite{OtherPipelineDoneWell,YoonEtAl,CriminisiVarient1,Compass2,Compass,Clamp,ModTelea,DepthAidedInpainingForNovelViewSynthesis,DepthGuidedInpaintingAlgorithmForFreeViewPointVideo,DepthImageBasedRenderingWithAdvancedTextureSynthesisfor3DVideo} to name a few, the disocclusion step arising in the second pipeline relevant to film has received far less attention.  We are, in fact, to the best of our knowledge the first paper to address it directly.  While related, these two disocclusion problems have important differences.  Most significantly, the fact that our pipeline comes with an explicit mask for every scene object - even occluded parts - and the fact that we have a full 3D model instead of just a single depth map from a single viewpoint, has two major consequences.  Firstly, while the methods above need to inpaint both the color information at the new view and the corresponding new depth map - we get the depth map at the new viewpoint for free.  This is important because most of the methods in the literature either devote quite a bit of effort to inpainting the depth map \cite{OtherPipelineDoneWell}, or else do so based on rough heuristics \cite{YoonEtAl,CriminisiVarient1,Compass2,Compass,Clamp,ModTelea,DepthAidedInpainingForNovelViewSynthesis,DepthGuidedInpaintingAlgorithmForFreeViewPointVideo,DepthImageBasedRenderingWithAdvancedTextureSynthesisfor3DVideo}, which, as noted in \cite[Sec. II.C.]{OtherPipelineDoneWell}, have a tendency to fail.  Secondly, these masks give an explicit segmentation of the scene into relevant objects both in the old viewpoint and the new one.  The methods in the other pipeline, by contrast, have access to neither.  This means that we unlike the above approaches always know which pixels to use for inpainting and do not have have to worry about (for example) inpainting a piece of the foreground into the background.  By contrast, all of the above methods have to rely on imperfect heuristics to guess based on the depth map which pixels belong to which object - see \cite[Sec. II.B.]{OtherPipelineDoneWell}.

Additionally, in our pipeline the inpainting is done by teams of artists armed a ``toolbox'' of inpainting algorithms.  These algorithms provide a starting point which artists may then touch up by hand.  Hence interactive speeds and the ability for the user to influence the results of inpainting, which may not have been a priority in the other pipeline, are important in ours.

\vskip 1mm

\noindent {\bf Image and Video Inpainting.}  Image inpainting refers to the filling in of a region in an image called the inpainting domain in such a way that the result looks plausible to the human eye.  Image inpainting methods can loosely be categorized as {\em exemplar-based} and {\em geometric}.  The former generally operate based on some procedure for copying patches of the undamaged portion of the image into the inpainting domain, either in a single pass from the boundary inwards as in Criminisi et al. \cite{Criminisi04regionfilling}, or iteratively as in Wexler et al. \cite{Spacetime} and Arias et al. \cite{Arias2011}.  The choice of which patch or patches to copy into a given area of the inpainting domain is decided using a nearest neighbor search based on a patch similarity metric.  Originally prohibitively expensive, a breakthrough was made in the PatchMatch algorithm \cite{Barnes2009}, which provides a fast {\em approximate} nearest neighbor search.  PatchMatch is used behind the scenes in the Photoshop's famous {\em Content-Aware Fill} tool.  On the other hand, geometric inpainting methods aim to smoothly extend image structure into the inpainting domain, typically using partial differential equations or variational principles.  Continuation may be achieved by either {\em interpolation} or {\em extrapolation}.  Examples of methods based on interpolation include the seminal work of Bertalmio et al. \cite{bertalmio2000image}, TV, TV-H$^{-1}$, Mumford-Shah, Cahn-Hilliard inpainting \cite{chan2005variational,Burger2009}, Euler's Elastica \cite{masnou1998level,chan2002euler}, as well as the joint interpolation of image values and a guiding vector field in Ballester et al. \cite{JointGreyVector}.  These approaches are typically iterative and convergence is often slow, implying that such methods are usually not suitable for real-time applications.  Telea's algorithm \cite{Telea2004} and coherence transport \cite{Marz2007,Marz2011} (which can be thought of as an improvement of the former) are based on {\em extrapolation} and visit each pixel only once, filling them in order according to their distance from the boundary of the inpainting domain. Unlike their iterative counterparts, these two methods are both very fast, but while possibly creating ``shocks'' - see Section \ref{sec:shocks}.  See also \cite{CarolaBook} for a comprehensive survey of geometric inpainting methods,  as well as \cite{ImageInpaintingOverview} for a recent survey of the field as a whole.

Geometric methods are designed to propagate structure, but fail to reproduce texture.  Similarly, exemplar-based approaches excel at reproducing texture, but are limited in terms of their ability to propagate structure.  A few attempts have been made at combining geometric and exemplar-based methods, such as Cao et al. \cite{Cao2011}, which gives impressive results but is relatively expensive.

Video inpainting adds an additional layer of complexity, because now temporal information is available, which is exploited by different algorithms in different ways.  For example, when inpainting a moving object in the foreground, one can expect to find the missing information in nearby frames - this type of strategy is utilized in for example \cite{RigRemoval}.  Another, more general strategy is to generalize exemplar-based image inpainting methods to video by replacing 2D image patches with 3D spacetime cubes.  This approach is taken in \cite{FastGeneric,Newson2014}, which also present a generalized patchmatch algorithm for video.  While producing impressive results, this method is also very expensive, both in terms of time and space complexity (see Section \ref{sec:numerical}).  Finally, the authors of \cite{PixMix} present a strategy for video inpainting of planar or almost-planar surfaces, based on inpainting a single frame and then propagating the result to neighboring frames using an estimated homography.
\vskip 2mm

\noindent {\bf Related Work.}  Our method is a (1st order) transport-based inpainting method for the disocclusion step in 3D conversion.  Here we review the related work on both aspects:

{\em I. Disocclusion Inpainting for 3D Conversion:}  Over the past decade there has been a considerable attention given in the literature to the design of algorithms for automatic or semi-automatic 3D conversion - at least for the first pipeline based on depth maps.  As we have already stated, the pipeline used in film, on which we focus in this work, has received little to no attention.  Nevertheless we review here briefly the work on this related problem.  In regards to the hole filling step, there is great variability in how it is handled.  At one extreme are cheap methods that inpaint each frame independently using very basic rules such as clamping to the color of the nearest useable pixel \cite{Clamp}, or taking a weighted average of closest useable pixels along a small number ($8-12$) of fixed directions \cite{Compass,Compass2}.  Slightly more sophisticated is the approach in \cite{ModTelea} which applies a depth-adapted variant of Telea's algorithm \cite{Telea2004}.  These methods are so basic that they do not appear to inpaint the depth map.  In the midrange you have a variety of methods based on first inpainting the depth map, and then applying a depth aided variant to Criminisi's method - examples include \cite{CriminisiVarient1, YoonEtAl,DepthAidedInpainingForNovelViewSynthesis,DepthGuidedInpaintingAlgorithmForFreeViewPointVideo,DepthImageBasedRenderingWithAdvancedTextureSynthesisfor3DVideo, OtherPipelineDoneWell}, see also \cite{OtherPipelineDoneWell} for an overview of the state of the art.  Unfortunately, until recently most of these approaches have been limited in the sense that too little attention has been given to the depth inpainting step, which is done based on crude heuristics, while the most of the attention is given to the subsequent color inpainting step using a depth-aided variant of Criminisi.  To our knowledge, \cite{OtherPipelineDoneWell} is the first paper to acknowledge this gap in the literature and address it with a sophisticated approach to depth inpainting.

Finally, at the most expensive extreme you have methods taking temporal information explicitly into account, such as \cite{spacetime3D} which copies spacetime patches into the inpainting domain via a process similar to Criminisi et al.

{\em II. Inpainting based on 1st-order Transport:}  There are a small number of inpainting strategies in the literature based on the idea of {\em 1st-order} transport of image values along a vector field, which is either predetermined or else calculated concurrently with inpainting (the seminal work of Bertalmio et al. \cite{bertalmio2000image} was also based on transport, but in their case the equation was third order).  While generally lower quality than their higher order counterparts, these methods have the potential to be extremely fast.  The earliest of these to our knowledge is Ballester et al. \cite{JointGreyVector}, which considers both the joint interpolation of image values and a guiding vector field, as well as the propagation of image values along a known vector field.  In the latter case, they note that their approach is equivalent to a 1st-order transport equation.  This was the first time to our knowledge that first order transport was proposed as a strategy for inpainting.  However, none of the approaches suggested in this paper are sufficiently efficient for our application.

Next, Telea et al. \cite{Telea2004} proposed filling the inpainting domain in successive shells from the boundary inwards, visiting each pixel only once and assigning it a color equal to a weighted average of its already filled neighbors, resulting in a very fast algorithm.  The connection to transport was not known until Bornemann and M\"arz showed that both Telea's algorithm and their improvement thereof, which they called coherence transport \cite{Marz2007,Marz2011}, both become 1st-order transport equations under a high-resolution and vanishing viscosity limit.  In Telea's algorithm, the user has no control over the transport direction - it was shown in \cite{Marz2007} to simply be equal to the local normal vector to the boundary of the inpainting domain.  Coherence transport attempts to improve on this by either allowing the user to supply the desired transport direction manually as a vector ${\bf g}$, or to find ``good'' values for ${\bf g}({\bf x})$ concurrently with inpainting.  However, as we will see, the algorithm actually has challenges in both respects (see below).  M\"arz went on in \cite{Marz2011} to suggest improvements to coherence tranpsort based on a carefully selected fill order, and then went on in \cite{Marz2015} to explore in depth an issue first raised in \cite{JointGreyVector} - how to make sense of the well-posedness of a 1st-order transport equation on a bounded domain with Dirichlet boundary conditions, where integral curves of the transport field may terminate at distinct points with incompatible image values. 

\vskip 2mm

\noindent {\bf Our Contribution.}  Our contributions are multiple.  Firstly, while any of the disocclusion algorithms from the previous section could be adapted to our pipeline, none of them are designed to take advantage of its particular characteristics.  In particular, none of them are designed to take advantage of the scene segmentation available in our pipeline, and with the possible exception of the recent high quality approach \cite{OtherPipelineDoneWell}, this is likely to lead to needless ``bleeding'' artifacts when pixels from ``the wrong object'' are used for inpainting.  See Figure \ref{fig:careless2}(c) as well as the discussion in \cite[Sec. II.C]{OtherPipelineDoneWell}.  Our first contribution is to define an inpainting algorithm designed to take advantage of this extra information explicitly, which we do by making use of a set of ``bystander pixels'' not be used for inpainting (Section \ref{sec:bystander}).

Secondly, even if the above methods were to be adapted to our pipeline, what appears to be missing is an algorithm suitable for the ``middle ground'' of cases where Telea's algorithm and coherence transport are inadequate, but exemplar-based approaches are needlessly expensive.  In particular, because the inpainting domains in 3D conversion tend to be thin ``cracks'' (see Figure \ref{fig:pipeline}), there are many situations in which one can safely ignore texture.  

Thirdly, we acknowledge that in the movie industry inpainting is typically done by teams of artists who are happier if they have the ability to influence the results of inpainting, and have designed our algorithm with this in mind.

Fourth, our method is a transport based algorithm inspired by the coherence transport algorithm \cite{Marz2007,Marz2011}, but improving upon it by correcting some of its shortcomings.  Both methods proceed by measuring the orientation of image isophotes in the undamaged region near the inpainting domain and then extrapolating based on a transport mechanism.  However, in the case of coherence transport both of these steps have problems.  Firstly, the procedure for measuring the orientation ${\bf g}$ of isophotes in the undamaged region is inaccurate and leads to ``kinking'' in the extrapolation.  See Figure \ref{fig:modStructureTensor} as well as Sections \ref{sec:guidefield}, \ref{sec:2ndKink} for a discussion of this problem and our resolution.  Second, once fed a desired transport direction ${\bf g}$ (which may or may not be accurate based on the last point), coherence transport instead transports along a direction ${\bf g}^*$ such that ${\bf g}^* \neq {\bf g}$ unless ${\bf g}$ points in one of a small number of special directions.  The result is a secondary ``kinking'' effect of extrapolated isophotes (see Figures \ref{fig:connectLine}, \ref{fig:ghostReal}, \ref{fig:coherenceTheory}).  This behaviour, which the authors of \cite{Marz2007,Marz2011} appear unaware of (the theory in \cite{Marz2007} does not account for it), is explored in Section \ref{sec:refraction} and rigorously analyzed in Section \ref{sec:continuum}.  We present an improved transport mechanism overcoming this problem, as well as a theoretical explanation of its origin and resolution - see Theorem \ref{thm:teaser}.  However, our ability to transport along these additional directions come at a price in the sense that our method introduces some blurring into extrapolated edges.  This blurring can be significant for low resolution images and wide inpainting domains, but otherwise is appears to be minimal - see Section \ref{sec:signalDegredation}.  Additional details on the similarities and differences between our method and coherence transport \cite{Marz2007,Marz2011} are presented in Section \ref{sec:approach}.

In this paper we present a fast, geometric, user guided inpainting algorithm intended for use by artists for the hole-filling step of 3D conversion of film.  We have designed our algorithm with two goals in mind:
\begin{itemize}
\item The method retains interactive speeds even when applied to the HD footage used in film.
\item Although the method is automatic, the artist is kept ``in the loop'' with a means of possibly adjusting the result of inpainting that is {\em intuitive} (that is, they are not simply adjusting parameters).
\end{itemize}

The first of these goals is accomplished via an efficient GPU implementation based on a novel algorithm for tracking the boundary of the inpainting domain as it evolves.  Since our method only operates on the boundary of the inpainting domain on any given step, knowing where the boundary is means that we can assign GPU processors only to boundary pixels, rather than all pixels in the image.  For very large images ($\sqrt{N} \gg p$, where $N$ denotes the number of pixels in the inpainting domain, and $p$ denotes the number of available processors), our tracking algorithm leads to a time and processor complexity of $T(N,M)=O(N\log N)$, $P(N,M)=O(\sqrt{N+M})$ respectively (where $N+M$ is the total number of pixels in the image), versus $T(N,M)=O((N+M)\sqrt{N})$, $P(N,M)=O(N+M)$ without tracking - see Theorem \ref{thm:complexity} and Theorem \ref{thm:complexity2}.  Moreover, for moderately large problems ($\sqrt{N} \lessapprox p$ and $N+M \gg p$) the gains are larger - $T(N,M)=O(\sqrt{N}\log N)$ with tracking in this case.

The second goal is accomplished by providing the user with automatically computed splines showing how key image isophotes are to be extended.  These splines may be edited if necessary.  In this regard our algorithm is not unlike Sun et al. \cite{Sun2005} and Barnes et al. \cite{Barnes2009}, both of which allow the user to similarly promote the extension of important structures by drawing them onto the image directly.  However, both of these approaches are exemplar-based, the former of is relatively expensive and the latter, while less expensive, is limited to linear edges.  As far as we know our method is the first {\em geometric} method to give the user this type of control over the results of inpainting.

Our method - which we call Guidefill - is intended as a practical tool that is fast and flexible, and applicable to many, but not all, situations.  It is not intended as a black box capable of providing the correct result in any situation given enough time.  Our method was originally designed for the 3D conversion company Gener8 and a version of it is in use by their stereo artists.

Similarly to many state of the art 3D conversion approaches we treat the problem frame by frame.  While an extension that uses temporal information would be interesting (and is a direction we would like to explore in the future), it is outside of the scope of this paper.

\vskip 2mm

\noindent {\bf Organization:}  In Section \ref{sec:pipeline} we go over a 3D conversion pipeline commonly used in film.  Section \ref{sec:alternativePipelines} also goes over the alternative pipeline commonly appearing in the literature, highlighting some of its potential drawbacks.  Next, in Section \ref{sec:approach} we present our proposed method as part of a broader class of shell-based algorithms, highlighting issues with earlier methods and how ours is able to overcome them, as well as issues with the class of methods as a whole.  Sections \ref{sec:2ndKink} and \ref{sec:refraction} in particular focus on two kinking issues associated with coherence transport and how Guidefill overcomes them, in the latter case through the introduction of what we call ``ghost pixels''.  Pixel ordering strategies for Guidefill are compared and constrasted with other strategies in the literature in Section \ref{sec:ordering}.  Two separate GPU implementations are sketched in section \ref{sec:GPU}.  Section \ref{sec:continuum} is devoted to a continuum analysis of our algorithm and others like it.  It enables us to rigourously explain some of the strengths and shortcomings of both Guidefill and coherence transport.  Our analysis is different from the analysis by M\"arz in \cite{Marz2007} - we consider a different limit and uncover new behaviour.  In Section \ref{sec:complexity} we analyze the time complexity and processor complexity of our method as a parallel algorithm.  In Section \ref{sec:numerical} we show the results of our method applied to a series of 3D conversion examples.  Results are compared with competing methods both in terms of runtime and visual quality.  At the same time, we also validate the complexity analysis of Section \ref{sec:complexity}.  Finally, in Section \ref{sec:conclusion} we draw some conclusions.

\vskip 2mm

\noindent {\bf Notation.}

\begin{itemize}
\item $h = \mbox{ the width of one pixel}$.
\item $\field{Z}^2_h := \{ (nh,mh) : (n,m) \in \field{Z}^2 \}.$
\item Given ${\bf x} \in \field{R}^2$, we denote by $\theta({\bf x}) \in [0,2\pi)$ the counter-clockwise angle ${\bf x}$ makes with the x-axis.
\item $\Omega = [a,b] \times [c,d]$ and $\Omega_h = \Omega \cap \field{Z}^2_h$ are the continuous and discrete image domains.
\item $D_h = D^{(0)}_h \subset \Omega_h$ is the (initial) discrete inpainting domain.
\item $D^{(k)}_h \subseteq D^{(0)}_h$ is the inpainting discrete inpainting domain on step $k$ of the algorithm. 
\item $B_h \subset \Omega_h \backslash D_h$ is the set of ``bystander pixels'' (defined in Section \ref{sec:bystander}) that are neither inpainted nor used for inpainting.
\item $u_h : \Omega_h \backslash (D_h \cup B_h) \rightarrow \field{R}^d$ is the given image (video frame).
\item $D \subset \Omega :=  \{ {\bf x} \in \Omega : \exists {\bf y} \in D_h \mbox{ s.t. } \|{\bf y}-{\bf x}\|_{\infty} < h \}$ is the continuous inpainting domain.
\item $D^{(k)}$ is the continuous inpainting domain on step $k$ of the algorithm, defined in the same way as $D$.
\item $B \subset \Omega$ is the continuous bystander set, defined in terms of $B_h$ in the same way as $D$.
\item ${\bf g} : D_h \rightarrow \field{R}^2$ is the guide field used to guide the inpainting.
\item $A_{\epsilon,h}({\bf x})$ denotes a generic discrete (but not necessarily lattice aligned) neighborhood of radius $\epsilon$ surrounding the pixel ${\bf x}$ and used for inpainting.\footnote{That is,  $A_{\epsilon,h}({\bf x}) \subset \Omega$ is a finite set and $\|{\bf y}-{\bf x}\| \leq \epsilon$ for all ${\bf y} \in A_{\epsilon,h}({\bf x})$} 
\item $B_{\epsilon,h}({\bf x}) = \{ {\bf y} \in \Omega_h : \|{\bf x}-{\bf y}\| \leq \epsilon \}$, the choice of $A_{\epsilon,h}({\bf x})$ used by coherence transport.
\item $\tilde{B}_{\epsilon,h}({\bf x}) = R(B_{\epsilon,h}({\bf x}))$, where $R$ is the rotation matrix taking $(0,1)$ to ${\bf g}({\bf x})$, the choice of $A_{\epsilon,h}({\bf x})$ used by Guidefill.
\item $\mathcal N({\bf x})=\{ {\bf x}+{\bf y} : {\bf y} \in \{-h,0,h\} \times \{-h,0,h\}, {\bf y} \neq {\bf 0}\}$ is the eight-point neighborhood of ${\bf x}$.
\item Given $A_h \subset \field{Z}^2_h$, we define the discrete (inner) boundary of $A_h$ by $$\partial A_h := \{ {\bf x} \in A_h : \mathcal N({\bf x}) \cap \field{Z}^2_h \backslash A_h \neq \emptyset \}.$$  For convenience we typically drop the word ``inner'' and refer to $\partial A_h$ as just the boundary of $A_h$.
\item Given $A_h \subset \field{Z}^2_h$, we define the discrete {\em outer} boundary of $A_h$ by $$\partial_{\mbox{outer}} A_h := \{ {\bf x} \in \field{Z}^2_h \backslash A_h : \mathcal N({\bf x}) \cap A_h \neq \emptyset \}.$$
\item $\partial_{\mbox{active}} D^{(k)}_h \subseteq \partial D^{(k)}_h$ is the active portion of the boundary of the inpainting domain on step $k$ of the algorithm.  That is
$$\partial_{\mbox{active}} D^{(k)}_h = \{ {\bf x} \in \partial D^{(k)}_h : \mathcal N({\bf x}) \cap (\Omega_h \backslash (D^{(k)}_h \cup B_h )) \neq \emptyset \}.$$
In other words, $\partial_{\mbox{active}} D^{(k)}_h$ excludes those pixels in $\partial D^{(k)}_h$ with no readable neighboring pixels.
\end{itemize}

\section{A 3D Conversion Pipeline for Film} \label{sec:pipeline}

\begin{figure}
\centering
\begin{tabular}{ccc}
\subfloat[]{\includegraphics[width=.3\linewidth]{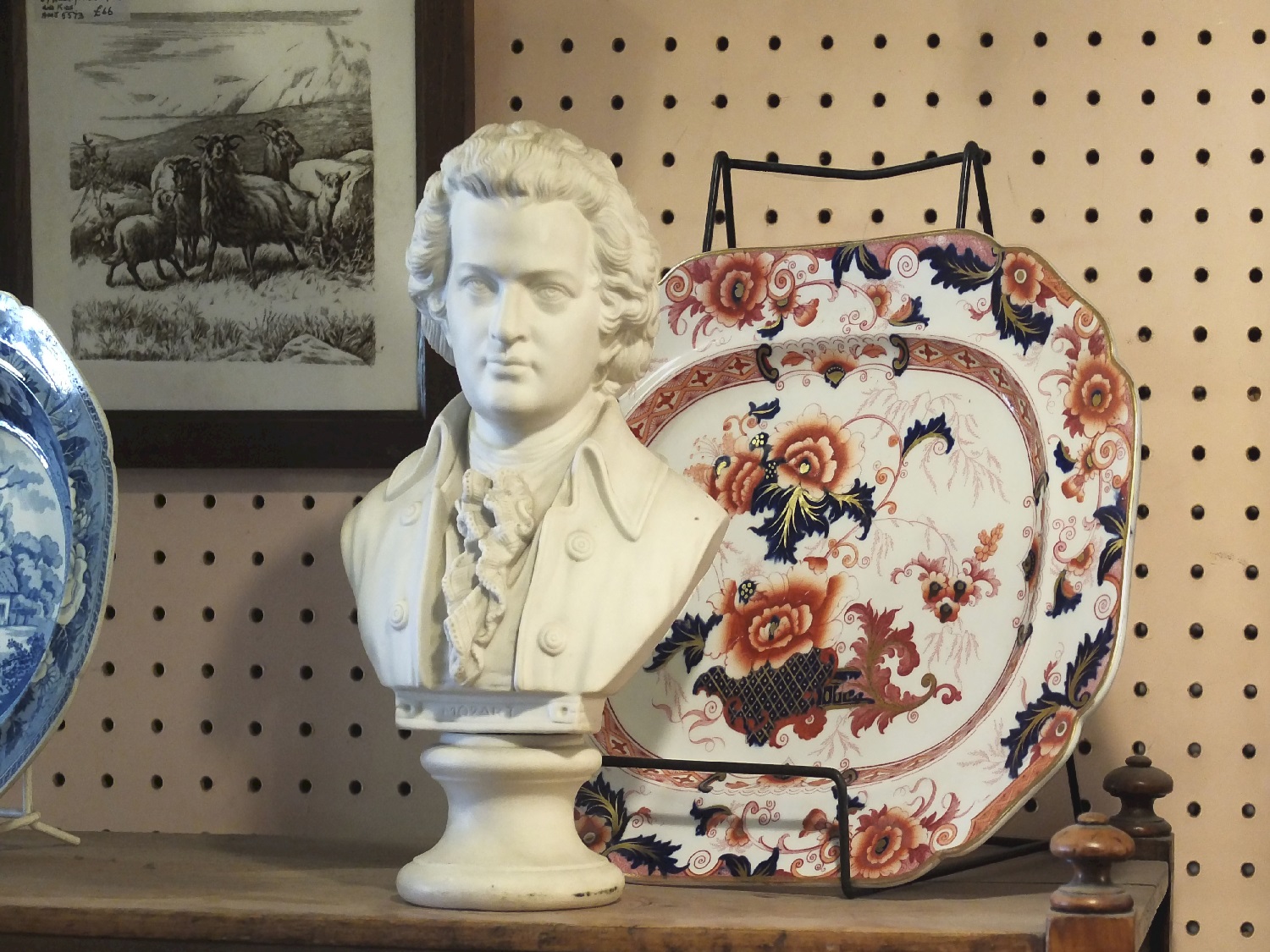}} & 
\subfloat[]{\includegraphics[width=.3\linewidth]{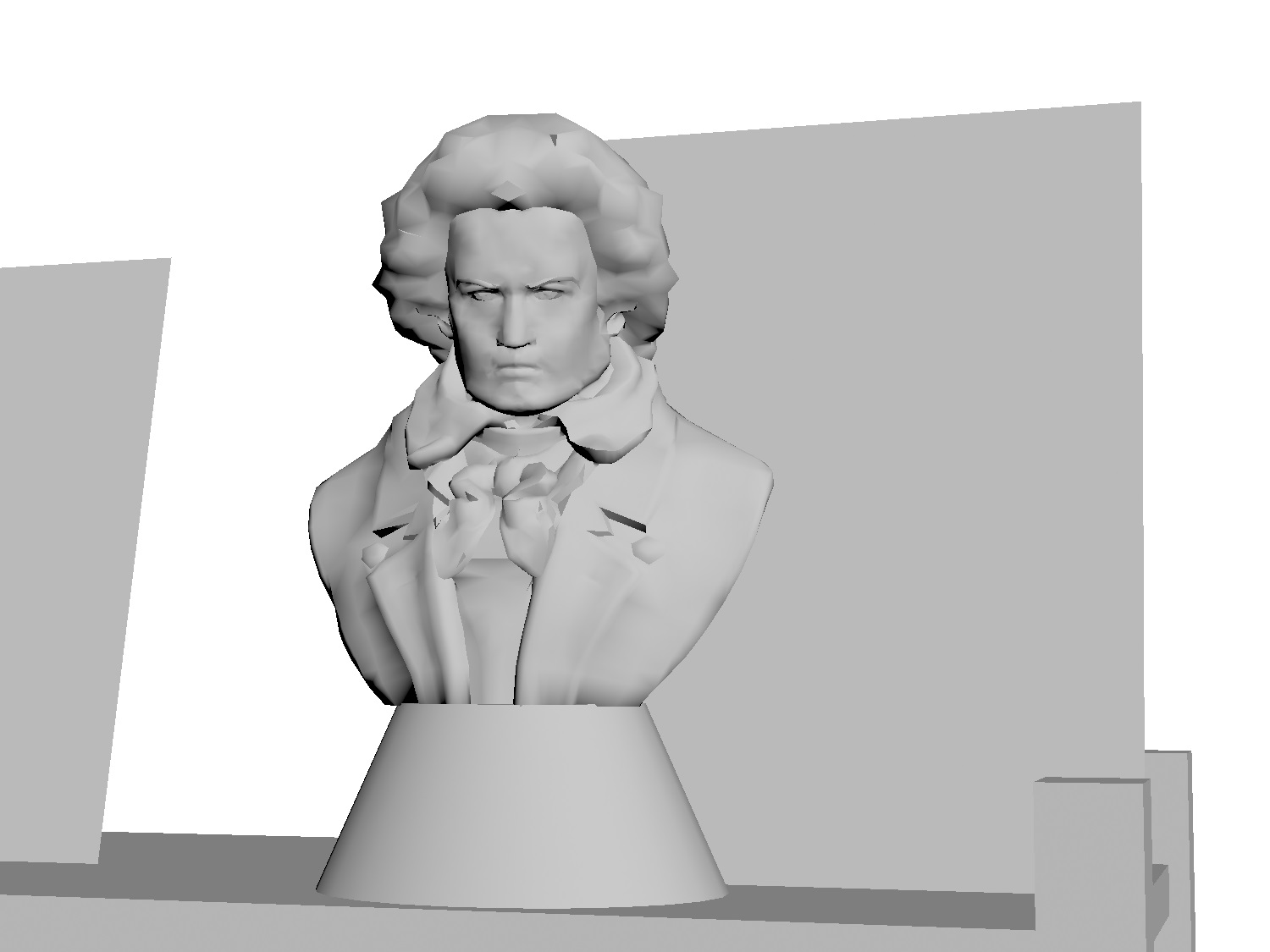}} &
\subfloat[]{\includegraphics[width=.3\linewidth]{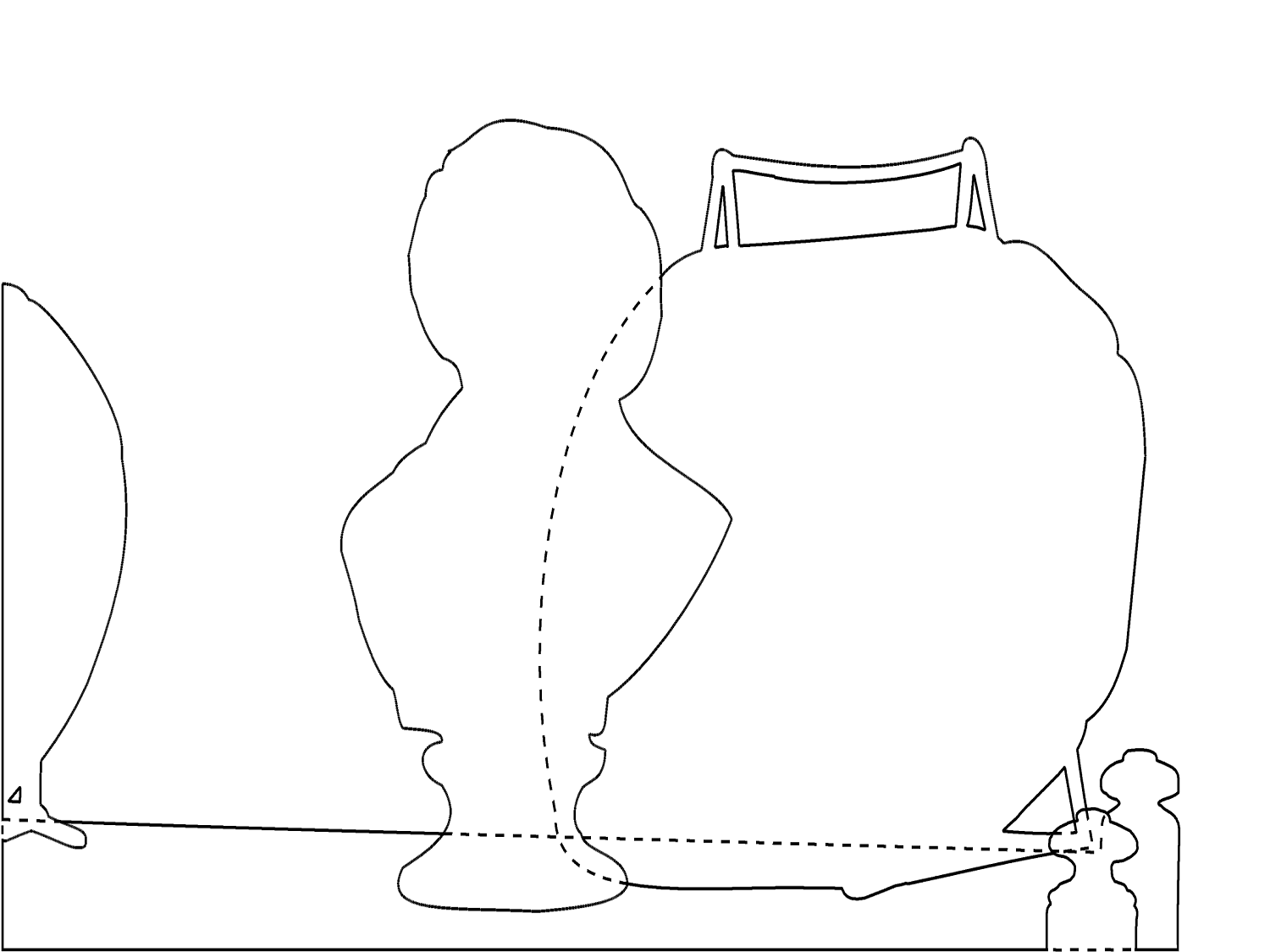}} \\
\subfloat[]{\includegraphics[width=.3\linewidth]{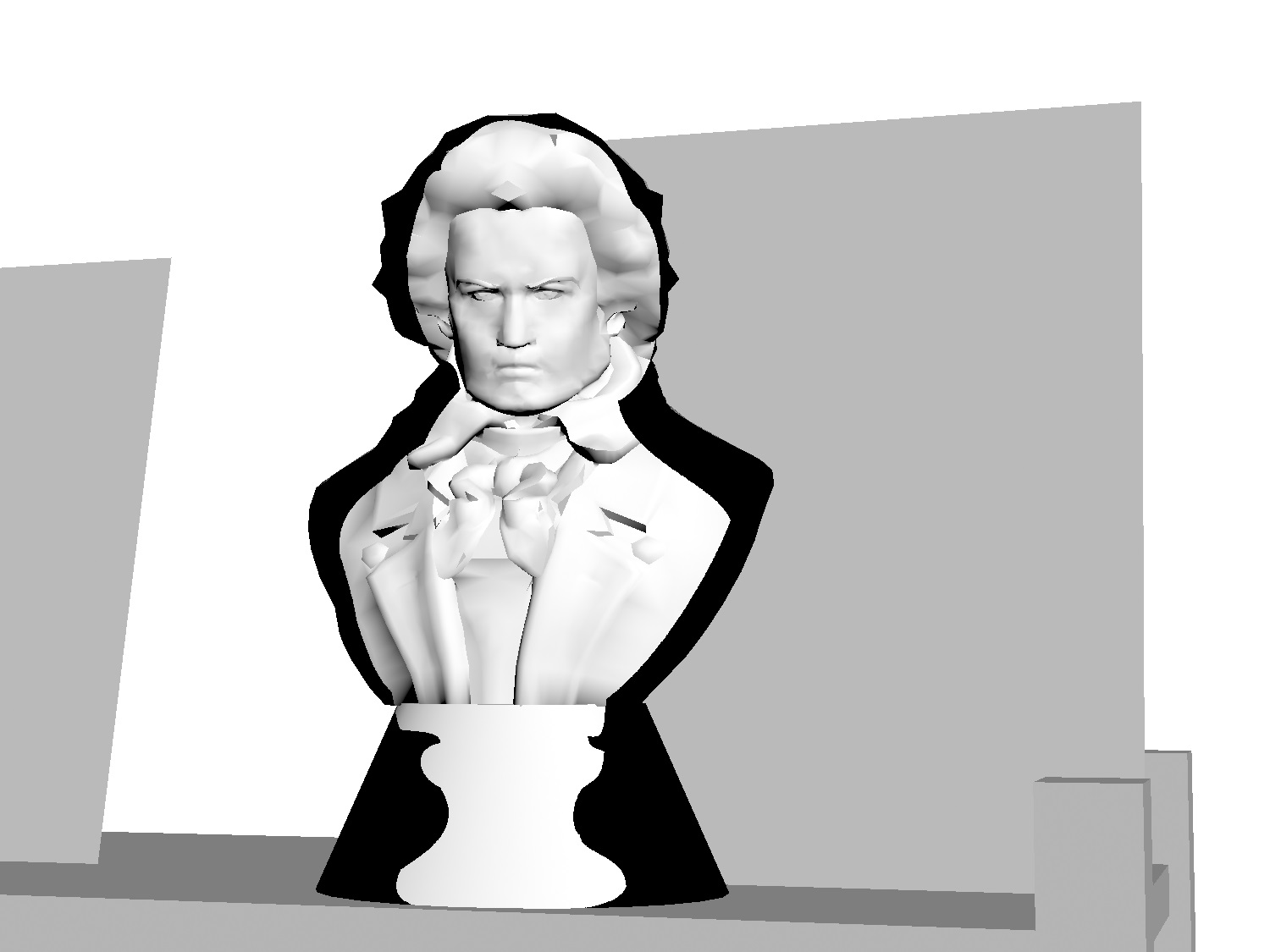}} &
\subfloat[]{\includegraphics[width=.3\linewidth]{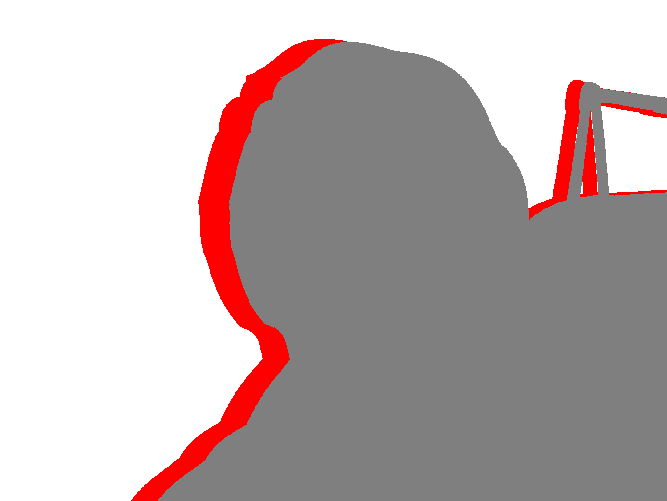}} &
\subfloat[]{\includegraphics[width=.3\linewidth]{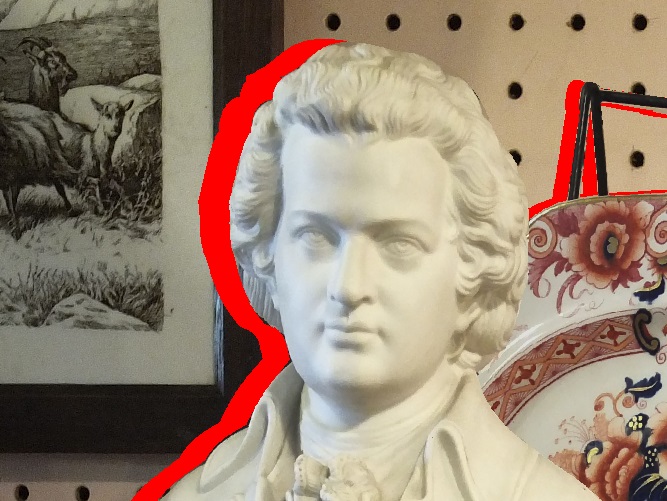}} \\
\end{tabular}
\caption{Intermediate data generated in a 3D conversion pipeline prior to inpainting: (a) original image, (b) rough 3D geometry, (c) object masks including occluded areas, (d) projection of an object mask onto the corresponding object geometry, (e)  example labeling of pixels in the new view according to object and visibility (in this case the object in question is the wall, white pixels are visible from both viewpoints, red are visible from the new viewpoint but occluded in the original view, grey are occluded in both views), (f) the generated new view with red ``cracks'' requiring inpainting.}
\label{fig:pipeline}
\end{figure}

Here we briefly review a 3D conversion pipeline commonly used in film - see for example \cite{CaseStudy3DConv} for a more detailed description. The pipeline relevant to us involves three main steps (typically done by separate teams of specialized artists) which must be completed before inpainting can proceed:
\begin{enumerate}
\item If camera data (including position, orientation and field of view) is not known, it must be estimated.  This process is often called ``match-move'' and is typically done with the aid of semi-automatic algorithms based on point tracking \cite{MatchMove,MatchMove2}.  
\item Accurate masks must be generated for all objects and for every frame, including occluded areas (See Figure \ref{fig:pipeline}(c)) .  This is typically done to a subpixel accuracy using editable B\'ezier splines called ``roto''.  These masks play three important roles:
\begin{enumerate} 
\item generating the depth discontinuities visible from the new viewpoint(s).
\item generating the scene segmentation in the old viewpoint.
\item generating the scene segmentation in the new viewpoint(s).
\end{enumerate}
These masks need to be as accurate as possible \cite{artOfSteroConv}.
\item A plausible 3D model of the scene must be generated (see Figure \ref{fig:pipeline}(b) for an example).  This will effectively be used to generate the ``smooth'' component of the depth map as viewed from the new viewpoint(s) and does not have to be perfect.  It is however very important that each object's mask generated in the previous step fit entirely onto its geometry when projected from the assumed camera position, as in Figure \ref{fig:pipeline}(d).  For this reason 3D geometry is typically designed to be slightly larger than it would be in real life \cite{CaseStudy3DConv}.
\item For each object, a multi-label mask must be generated assigning a label to each pixel in the new view as either
\begin{itemize}
\item belonging to the object and visible from the original viewpoint, or
\item belonging to the object and occluded in the original viewpoint, but visible in the new viewpoint, or
\item belonging to the object and occluded in both the original and new viewpoints, or
\item belonging to another object.
\end{itemize}
See Figure \ref{fig:pipeline}(e) for an example where the four labels are colored white, red, grey, and black respectively, and the object in question is the background.
\end{enumerate}
Once these components are in place, the original footage, clipped using the provided masks, is projected onto the geometry from the assumed camera position and orientation.  The new view is then generated by rendering the 3D scene from the perspective of a new virtual camera.  This new view, however, contains disoccluded regions - formerly hidden by geometry in the old view - which must be inpainted (see Figure \ref{fig:pipeline}(f)).  Inpainting then proceeds on an object by object basis, with each object inpainted separately.

\subsection{Bystander Pixels} \label{sec:bystander}

In most image inpainting algorithms it is assumed that all pixels in $\Omega_h \backslash D_h$ may be used for inpainting.  However, for this application each object is inpainted separately, so some of the pixels in $\Omega_h \backslash D_h$ belong to other objects (according to the labelling in step 4) and should be excluded.  Failure to do so will result in ``bleeding'' artifacts, where, for example, a part of background is extended into what is supposed to be a revealed midground object - see Figure \ref{fig:careless2}(c).

\begin{figure}
\centering
\begin{tabular}{cccc}
\subfloat[Detail from ``Bust'':  A complex hole involving several objects at multiple depths.]{\includegraphics[width=.21\linewidth]{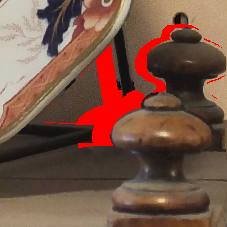}} &
\subfloat[Segmentation of the new view available to our pipeline.]{\includegraphics[width=.21\linewidth]{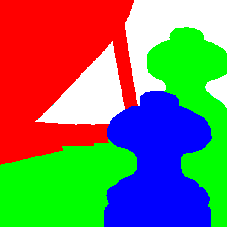}} &
\subfloat[midground structure cut off by ``bleeding'' of background into midground, when (b) is not taken into account.]{\includegraphics[width=.21\linewidth]{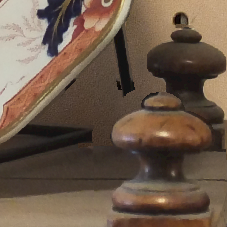}} &
\subfloat[Our result.]{\includegraphics[width=.21\linewidth]{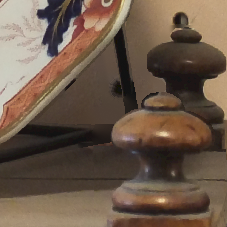}} \\
\end{tabular}
\caption{{\bf Importance of the pixel labeling step:}  Unlike our pipeline, which has an explicit scene segmentation (b) available to it from the new viewpoint, the depth map based pipeline does not have this information and must rely on heuristics.  As noted in \cite{OtherPipelineDoneWell}, these heuristics tend to fail for complex holes involving multiple objects at different depths, such as (a).  Most methods in the literature (especially those based on scanlines such as \cite{YoonEtAl,CriminisiVarient1,ModTelea}) with the exception of \cite{OtherPipelineDoneWell} itself (a very recent paper designed to cope with these situations) would struggle to correctly inpaint this hole and would likely produce artifacts similar to (c), where the midground structure is cut off by ``bleeding'' of background into the midground.  Our pipeline does not have this problem as it is able to take advantage of the segmentation in (b). }
\label{fig:careless2}
\end{figure}

Pixels which are neither inpainted nor used as inpainting data are called ``bystander pixels'', and the set of all such pixels is denoted by $B_h$.  Pixels in $\Omega_h \backslash (D_h \cup B_h)$ are called ``readable''.

\subsection{An Alternative Pipeline} \label{sec:alternativePipelines}

Here we briefly review the depth-map based pipeline that has so far received the most attention in the literature.  We will go over some of the heuristics employed and give a simple example to show how these heuristics can fail.  Please also see \cite{OtherPipelineDoneWell}, which covers the same issues we raise but in more detail, and aims at overcoming them.

The general setup is that we have an initial image/video frame $u_0$ with an accompanying depth map $d_0$ take from a known camera position, and we wish to know the image/video frame $u_0'$ from a new virtual camera position.  The key idea is that of a warping function $\mathcal W$, constructed from the known camera positions and parameters, that determines where a pixel ${\bf x}$ in $u_0$ at depth $d_0({\bf x})$ ``lands'' in $u'_0$.  $u'_0$ and $d'_0$ are then constructed by applying $\mathcal W$ to all pixels in $u_0$, $d_0$ (note that some care may be required as in general $\mathcal W({\bf x},d_0({\bf x}))$ may lie between pixel centers). It is typically assumed that the camera positions are related by a translation orthogonal to the optical axis and parallel to the horizon so that $\mathcal W$ is a simple horizontal translation.  The result is a new image $u_0'$ and depth map $d'_0$ with ``gaps'' due to disocclusion.  

The main disadvantage of this approach is that it has access to neither a depth map of the new view nor a segmentation thereof, whereas we have both.  When confronted with a complex hole as in Figure \ref{fig:careless2}(a), our pipeline also has access to Figure \ref{fig:careless2}(b), and hence while it may not know what RGB values are meant to go in the hole, it at least knows which object they are meant to belong to.  Without this information, algorithms in this pipeline instead have to make guesses based on heuristics.  On common approach is to first inpaint the depth map based on heuristics, then then use the inpainted depth map to guess which pixels belong to which objects.  For depth map inpainting, a very common heuristic, used in for example \cite{YoonEtAl,CriminisiVarient1}, is to divide the inpainting domain in horizontal scanlines.  Each scanline is then filled with a constant depth value that may be the endpoint with the greater depth \cite{CriminisiVarient1}, or the minimal extrema of depth patch statistics centered at the endpoints of the scanline as well as their inverse images under the warping function $\mathcal W$ \cite{YoonEtAl}.  In \cite{ModTelea}, the authors do not  inpaint the depth map, but divide the inpainting domain into horizontal scanlines as usual, declaring the endpoint with greater depth ``background'' and hence useable for inpainting, while discarding the other endpoint.  These approaches will work for most of the hole in Figure \ref{fig:careless2}(a), but all of them will incorrectly cut off the vertical plate leg as in Figure \ref{fig:careless2}(c).   Another approach, used in for example \cite{DepthGuidedInpaintingAlgorithmForFreeViewPointVideo}, is to inpaint using a modified variant of Criminisi that assigns higher priority to pixels with greater depth.  This approach is also likely to fail to extend either leg of the plate, since as an object lying in the midground it will be given a lower priority than background pixels.

In fact, of the approaches currently in the literature, the only one likely to give the correct result in this case is \cite{OtherPipelineDoneWell}, which was designed to address this gap in the literature by incorporating an explicit structure propagation step.  By contrast, our algorithm, taking advantage of the information in Figure \ref{fig:careless2}(b), produces the result in Figure \ref{fig:careless2}(d).

\section{Proposed Approach} \label{sec:approach}

Guidefill is a member of an extremely simple class of inpainting algorithms which also contains coherence transport \cite{Marz2007,Marz2011} and Telea's algorithm \cite{Telea2004}.  These methods fill the inpainting domain in successive shells from the boundary inwards, with the color of a given pixel due to be filled computed as a weighted average of its already filled neighbors.  The averaging weights $w_{\epsilon}$ are chosen to scale with proportionally with the size of the neighborhood $A_{\epsilon,h}({\bf x})$.  That is 
\begin{equation} \label{eqn:scalinglaw}
w_{\epsilon}({\bf x},{\bf y})=\hat{w}\left(\frac{{\bf x}}{\epsilon},\frac{{\bf y}}{\epsilon}\right)
\end{equation}
for some function $\hat{w}(\cdot,\cdot)$ with both arguments in the unit ball.  See \eqref{eqn:weight} for the weights used by coherence transport and Guidefill, which are of the form \eqref{eqn:scalinglaw}.  Note that we will sometimes write $w_r$ or $w_1$ in place of $w_{\epsilon}$ - in this case we mean \eqref{eqn:scalinglaw} with $\epsilon$ replaced by $r$ or $1$ on the left hand side.    As the algorithm proceeds, the inpainting domain shrinks, generating a sequence of inpainting domains $D_h = D^{(0)}_h \supset D^{(1)}_h \supset \ldots \supset D^{(K)}_h = \emptyset$.  At iteration $k$, only pixels belonging to the current active boundary $\partial_{\mbox{active}} D^{(k)}_h$ are filled, however, $\partial_{\mbox{active}} D^{(k)}_h$ need not be filled in its entirety - certain pixels may be made to wait until certain conditions are satisfied before they are ``ready'' to be filled (see Section \ref{sec:ordering} for discussion and \eqref{eqn:ready} for a definition of ``ready'').  Algorithm 1 illustrates this with pseudo code.
\begin{algorithm}
\caption{Shell Based Geometric Inpainting}
\begin{algorithmic} \label{alg:general}
\STATE $u_h$ = image
\STATE $D^{(0)}_h$ = initial inpainting domain
\STATE $\partial_{\mbox{active}} D^{(0)}_h $= initial active inpainting domain boundary
\STATE $B_h$ = bystander pixels
\FOR{$k=0,\ldots$}
\IF{$D^{(k)}_h = \emptyset$}
\STATE break
\ENDIF
\FOR{${\bf x} \in \partial_{\mbox{active}} D^{(k)}_h$}
\STATE compute $A_{\epsilon,h}({\bf x})$ = neighborhood of ${\bf x}$.
\STATE compute non-negative weights $w_{\epsilon}({\bf x},{\bf y}) \geq 0$ for $A_{\epsilon,h}({\bf x})$.
\IF{ready({\bf x})}
\STATE 
\vspace{-5mm}
\begin{equation*} 
u_h({\bf x})  = \frac{\sum_{{\bf y} \in A_{\epsilon,h}({\bf x}) \cap \Omega \backslash ( D^{(k)} \cup B_h) }w_{\epsilon}({\bf x},{\bf y})u_h({\bf y})  }{\sum_{{\bf y} \in A_{\epsilon,h}({\bf x}) \cap \Omega \backslash ( D^{(k)} \cup B_h) }w_{\epsilon}({\bf x},{\bf y})} \phantom{This text will be invisible invi} 
\end{equation*}
\vspace{-1mm}
\ENDIF
\ENDFOR
\STATE $F = \{ {\bf x} \in \partial_{\mbox{active}} D^{(k)}_h : \mbox{ready}({\bf x})\}$
\STATE $D^{(k+1)}_h = D^{(k)}_h \backslash F$
\STATE $\partial_{\mbox{active}} D^{(k+1)}_h = \{ {\bf x} \in \partial D^{(k+1)}_h : \mathcal N({\bf x}) \cap (\Omega_h \backslash (D^{(k+1)}_h \cup B_h )) \neq \emptyset \}.$
\ENDFOR
\end{algorithmic}
See \eqref{eqn:ready}  for a definition of the $\mbox{ready}$ function for Guidefill.  Coherence transport and Guidefill use the neighborhoods $A_{\epsilon,h}({\bf x}) = B_{\epsilon,h}({\bf x})$, $A_{\epsilon,h}({\bf x}) = \tilde{B}_{\epsilon,h}({\bf x})$ respectively - see Figure \ref{fig:ballRotate}.  They also both use the same weights \eqref{eqn:weight}.  Note that in coherence transport, $\partial_{\mbox{active}} D^{(k)}_h = \partial D^{(k)}_h$ as there are no bystander pixels.
\end{algorithm}
While basic, these methods have the advantage of being cheap and highly parallelizable.  When implemented on the GPU the entire active boundary of the inpainting domain can be filled in parallel.  If done carefully, this yields a very fast algorithm suitable for very large images - see Section \ref{sec:GPU}.

Guidefill is inspired in part by coherence transport \cite{Marz2007,Marz2011}.  Coherence transport operates by adapting its weights in order to extrapolate along isophotes in the undamaged portion of the image when they are detected, and applying a smooth blur when they are not.  While relatively fast and achieving good results in many cases,  it has a number of drawbacks:
\begin{enumerate}
\item Users may need to tune parameters in order to obtain a good result.
\item Extrapolated isophotes may ``kink'' due to inaccurate computation of the guidance direction ${\bf g}$ (see Figure \ref{fig:modStructureTensor} and Section \ref{sec:2ndKink}).
\item Even if ${\bf g}$ is computed correctly, extrapolated isophotes may still ``kink'' if ${\bf g}$ does not belong to a finite set of special directions (see Figures \ref{fig:connectLine}, \ref{fig:ghostReal}, \ref{fig:coherenceTheory} and Sections \ref{sec:refraction}, \ref{sec:continuum}).
\item The method is a black box with no artist control.
\item The quality of the result can be strongly influenced by the order in which pixels are filled - see Figure \ref{fig:goodbadOrder}.  This is partially addressed in \cite{Marz2011}, where several methods are proposed for precomputing improved pixel orderings based on non-euclidean distance functions.  However, these methods all either require manual intervention or else have other disadvantages - see Section \ref{sec:ordering}.
\end{enumerate}
Guidefill is aimed at overcoming these difficulties while providing an efficient GPU implementation (the implementation of coherence transport in \cite{Marz2007,Marz2011} was sequential, despite the inherent parallelizability of the method), in order to create a tool for 3D conversion providing intuitive artist control and improved results.

\subsection{Formation of shocks} \label{sec:shocks}

\begin{figure}
\centering
\begin{tabular}{ccc}
\subfloat[An inpainting problem with incompatible boundary conditions.  The inpainting domain $D_h$ is in grey, and the skeleton $\Sigma$ is drawn in black.]{\includegraphics[width=.3\linewidth]{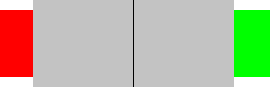}} &
\subfloat[Inpainting using Guidefill.  A shock is formed on the skeleton set $\Sigma$ shown in (a).]{\includegraphics[width=.3\linewidth]{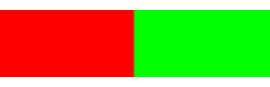}} &
\subfloat[Inpainting by solving the second order elliptic equation $-\epsilon \Delta u + u_x =0$ with $\epsilon = 10^{-7}$.  Shocks are prevented, but the solution (using GMRES) becomes much more expensive.]{\includegraphics[width=.3\linewidth]{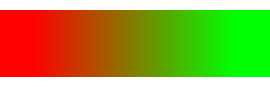}} \\
\end{tabular}
\caption{{\bf Creation of shocks by Algorithm 1:} When Algorithm 1 is used to inpaint problems with incompatible boundary conditions, such as the problem illustrated in (a) of inpainting a stripe that is red on one end and green on the other, the result may contain shocks as in (b).  These shocks can be understood by adopting the framework proposed in \cite{Marz2007,Marz2015}, where the output of Algorithm 1 under a high resolution and vanishing viscosity limit is shown to be equivalent to the solution of a first order transport equation on $D \backslash \Sigma$, where $\Sigma$ is a set of measure zero containing any potential shocks.  Ballester et al. \cite{JointGreyVector} suggested overcoming this problem by adding a diffusive term $-\epsilon \Delta u$ to the transport equation and taking $\epsilon \rightarrow 0$.  As can be seen in (c), in this case the formation of shocks is prevented, but the algorithm becomes much more expensive, in this case requiring several minutes for GMRES \cite{GMRES} to converge on a $200 \times 200$px inpainting domain (Guidefill by contrast took only $60$ms).}
\label{fig:shocks}
\end{figure}

A disadvantage of the shell based approach in Algorithm 1 is the potential to create a shock in the middle of $D_h$, where image values propagated from initially distant regions of $\partial D_h$ meet - see Figure \ref{fig:shocks}(a)-(b) for a simple example.  Since M\"arz \cite[Theorem 1]{Marz2007} showed that Algorithm 1 (with $B_h = \emptyset$ and under other assumptions) is related in a high resolution and vanishing viscosity limit to a first order transport equation
\begin{equation} \label{eqn:1storder}
{\bf c}({\bf x}) \cdot \nabla u = 0 \quad \mbox{ in } D, \quad u = \varphi \quad \mbox{ on } \partial D,
\end{equation}
this problem was arguably anticipated by Ballester et al. \cite{JointGreyVector}, who noted that above boundary value problem does not have an obvious solution.  Indeed, the integral curves of ${\bf c}({\bf x})$, each with a beginning and endpoint on $\partial D$, may have incompatible values of $\varphi$ at those endpoints.  To resolve this issue they suggested, among other things, adding a diffusive term $-\epsilon \Delta u$ to \eqref{eqn:1storder} to make it well posed, and then taking $\epsilon \rightarrow 0$.  See Figure \ref{fig:shocks}(c), where we solve the resulting nonsymmetric linear system with $\epsilon = 10^{-7}$ using GMRES (the Generalized Minimum RESidual method for nonsymmetric linear systems) \cite{GMRES}.  In a series of papers \cite{Marz2007,Marz2011,Marz2015} M\"arz took a different approach and instead showed that \eqref{eqn:1storder} is well posed on $D \backslash \Sigma$, where $\Sigma$ is a set of measure zero containing any potential shocks, and related to a distance map prescribing the order in which pixels are filled.

In our case this issue is less significant as we only specify boundary data on $\partial_{\mbox{active}} D_h \subset \partial D_h$.  Indeed, as long as the integral curves of ${\bf c}({\bf x})$ do not cross and always have one endpoint on $\partial_{\mbox{active}} D_h$ and the other on $\partial D_h \backslash \partial_{\mbox{active}} D_h$, we avoid the issue altogether.  However, there is nothing about our framework that explicitly prevents the formation of shocks, and indeed they do sometimes occur - see for example Figure \ref{fig:bust}(f).

\subsection{Overview} \label{sec:overview}

\begin{figure}
\centering
\begin{tabular}{ccc}
\subfloat[Automatically generated splines.]{\includegraphics[width=.3\linewidth]{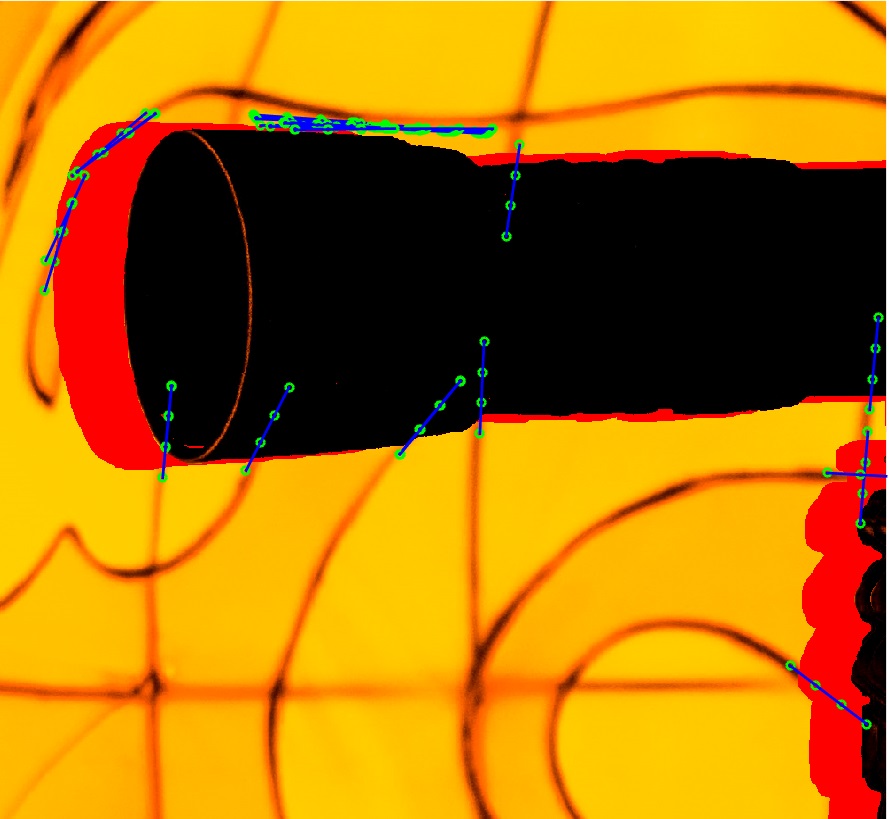}} & 
\subfloat[After user adjustment.]{\includegraphics[width=.3\linewidth]{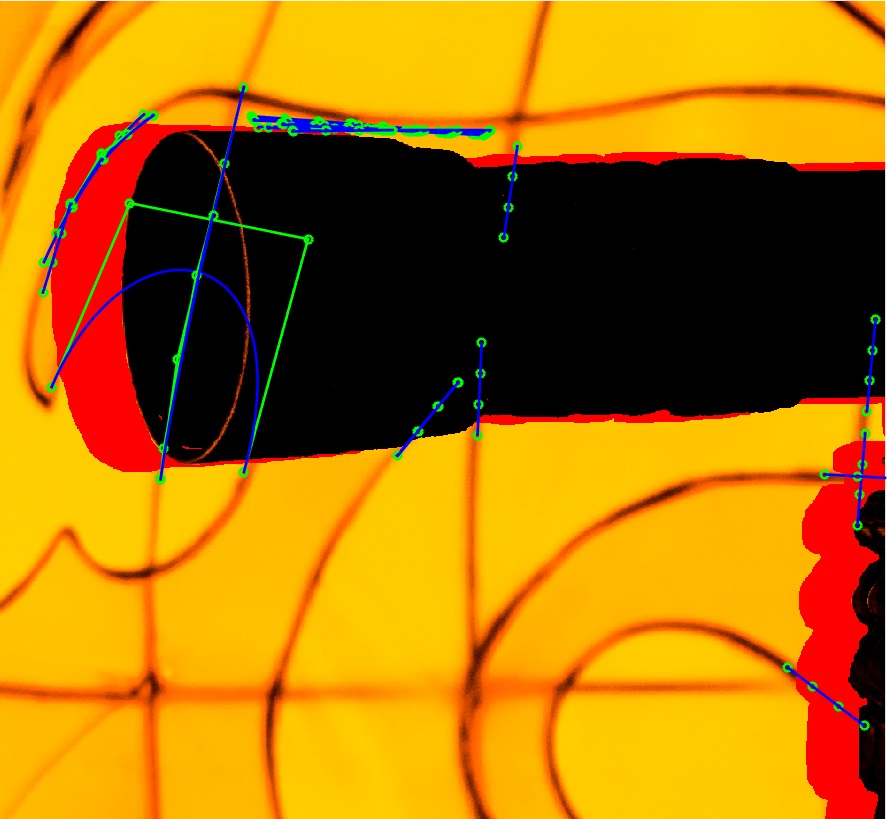}} &
\subfloat[Closeup of the resulting guide field.]{\includegraphics[width=.302\linewidth]{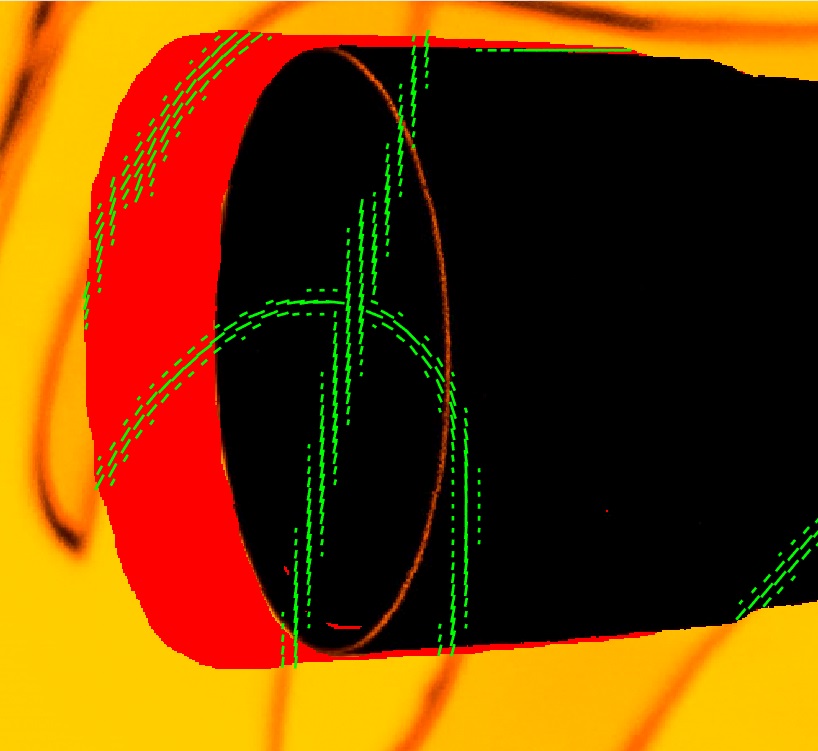}} \\
\end{tabular}
\caption{Generating the guide field ${\bf g}$ (c) based on splines automatically generated by Guidefill (a) and edited by the user (b).}
\label{fig:guidefield}
\end{figure}

The main idea behind Guidefill is to generate, possibly based on user input, a suitable vector field ${\bf g}: D_h \rightarrow \field{R}^2$ to guide the inpainting process, prior to inpainting.  The vector field ${\bf g}$, which we call the ``guide field'', is generated based on a small set of curves carrying information about how key image edges in $\Omega_h \backslash (D_h \cup B_h) $ should be continued into $D_h$.  These curves provide an intuitive mechanism by which the user can influence the results of inpainting (see Figure \ref{fig:guidefield}).

Coherence transport also utilizes a vector field ${\bf g}({\bf x})$, but it is calculated concurrently with inpainting.  Precomputing the guide field ahead of time is an advantage because the guide field contains information that can be used to automatically compute a good pixel ordering, avoiding artifacts such as Figure \ref{fig:goodbadOrder}.  At step $k$ of our algorithm, given any pixel ${\bf x} \in \partial_{\mbox{active}} D^{(k)}_h$ due to be filled, our algorithm decides based on ${\bf g}({\bf x})$ whether to allow ${\bf x}$ to be filled, or to wait for a better time.  Our test amounts to checking whether or not enough pixels have already been inpainted in the area pointed to by ${\bf g}({\bf x})$, and is discussed in greater detail in Section \ref{sec:ordering}.

The method begins with the user either drawing the desired edges directly onto the image as B\'ezier splines using a GUI, or else by having a set of splines automatically generated for them based on the output of a suitable edge detection algorithm run on $\Omega_h \backslash (D_h \cup B_h)$.  In the latter case the user may either accept the result or else use it as a starting point which they may improve upon by editing and/or removing existing splines as well as drawing new ones.  This is illustrated in Figure \ref{fig:guidefield}.

Next, the idea is to choose ${\bf g}({\bf x})$ to be ${\bf 0}$ when ${\bf x}$ is far away from any splines (e.g. more than a small number of pixels, around ten by default), and ``parallel'' to the splines when ${\bf x}$ is close.  Details are provided in Section \ref{sec:guidefield}.

The purpose of the guide field is to ensure that the inpainting will tend to follow the splines wherever they are present.  To accomplish this, at step $k$ of our algorithm a given pixel ${\bf x} \in \partial_{\mbox{active}} D^{(k)}_h$ due to be inpainted is ``filled'' by assigning it a color equal to a weighted average of its already filled neighbors, with weights biased in favor of neighboring pixels ${\bf y}$ such that ${\bf y}-{\bf x}$ is parallel to ${\bf g}({\bf x})$.  This is accomplished using the weight function
\begin{equation} \label{eqn:weight}
w_{\epsilon}({\bf x},{\bf y}) = \frac{1}{\|{\bf y}-{\bf x}\|}\exp\left(-\frac{\mu^2}{2\epsilon^2}({\bf g}^{\perp}({\bf x}) \cdot ({\bf y}-{\bf x}))^2\right),
\end{equation}
(introduced in coherence transport \cite{Marz2007}) where $\mu > 0$ is a positive parameter and $\epsilon > 0$ is the radius of the neighborhood $A_{\epsilon,h}({\bf x})$.  However, whereas the sum in coherence transport is taken over the filled portion of the discrete ball $A_{\epsilon,h}({\bf x})=B_{\epsilon,h}({\bf x})$ aligned with the image lattice, we sum over the available ``pixels'' within a {\em rotated} ball $A_{\epsilon,h}({\bf x})=\tilde{B}_{\epsilon,h}({\bf x})$ aligned with the local guide direction ${\bf g}({\bf x})$ - see Figure \ref{fig:ballRotate} for an illustration.  The color $u_h({\bf x})$ is then computed using the formula in Algorithm 1, taking $A_{\epsilon,h}({\bf x})=\tilde{B}_{\epsilon,h}({\bf x})$ and using weights \eqref{eqn:weight}.  Coherence transport ``fills'' a pixel using exactly the same formula, except that now $A_{\epsilon,h}({\bf x})=B_{\epsilon,h}({\bf x})$.

\begin{figure}
\centering
\begin{tabular}{ccc}
\subfloat[$A_{\epsilon,h}({\bf x})=B_{\epsilon,h}({\bf x})$]{\includegraphics[width=.28\linewidth]{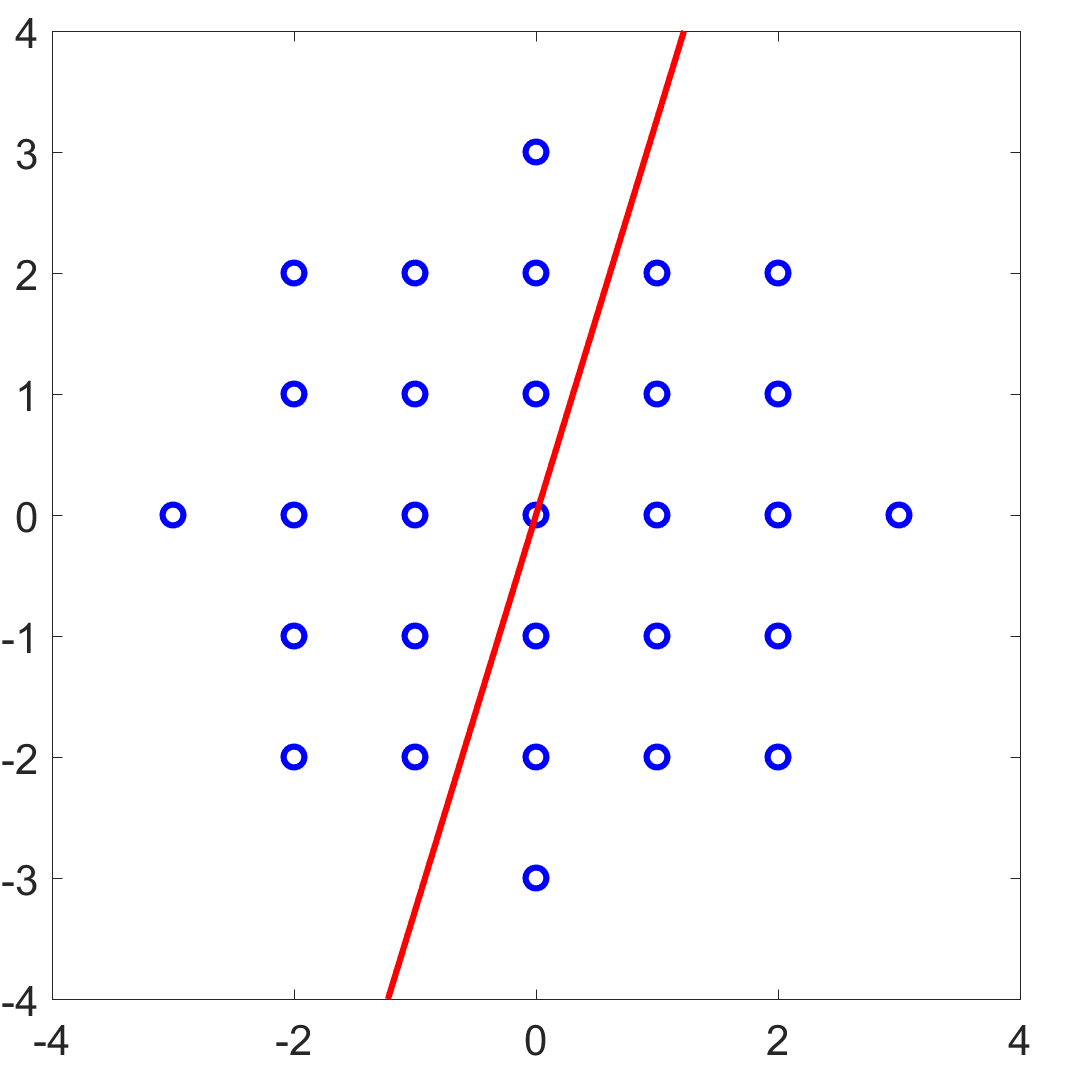}} & 
\subfloat[$A_{\epsilon,h}({\bf x})= \tilde{B}_{\epsilon,h}({\bf x})$]{\includegraphics[width=.28\linewidth]{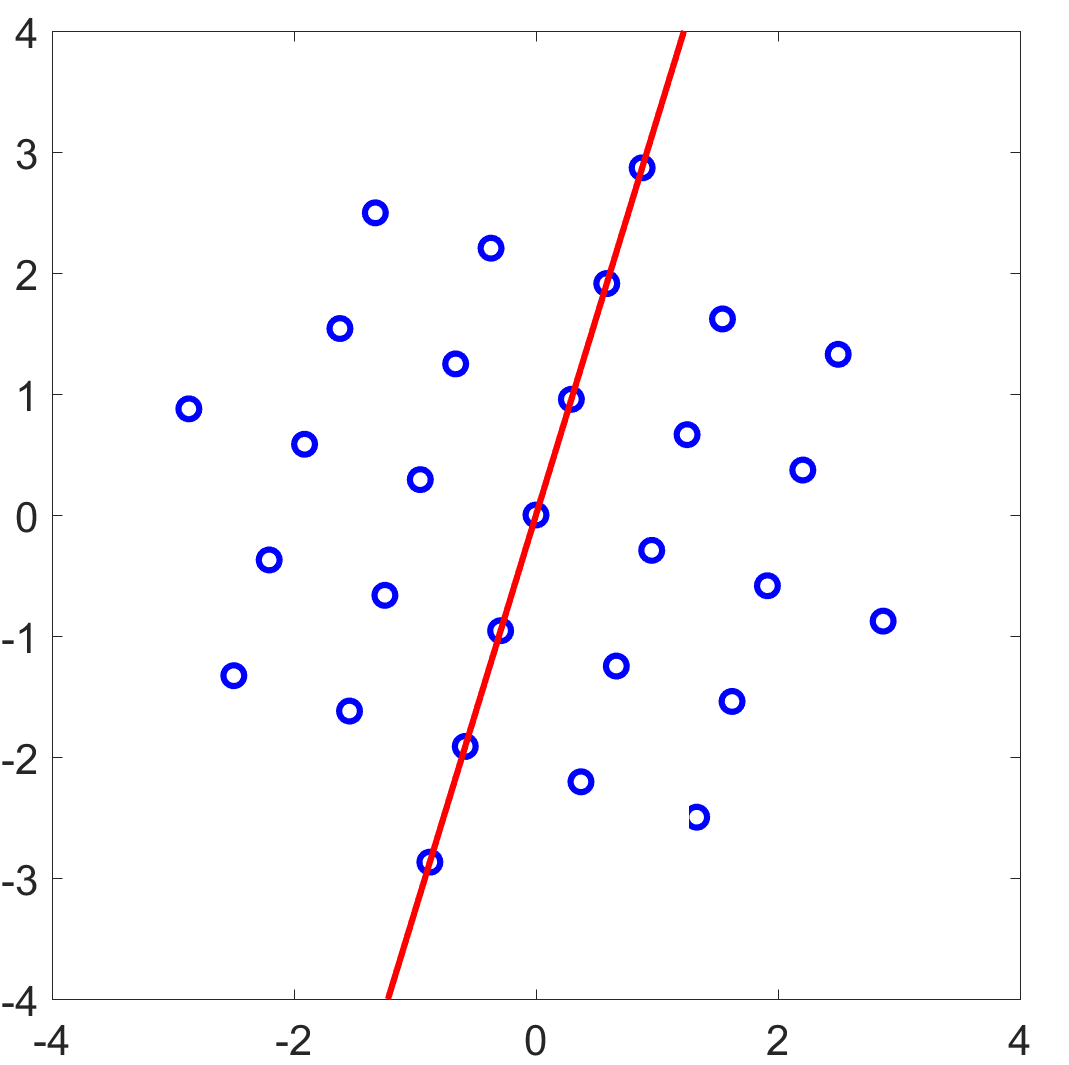}} &
\subfloat[Illustration of the (normalized) weights \eqref{eqn:weight} for $\mu = 10$.]{\includegraphics[width=.33\linewidth]{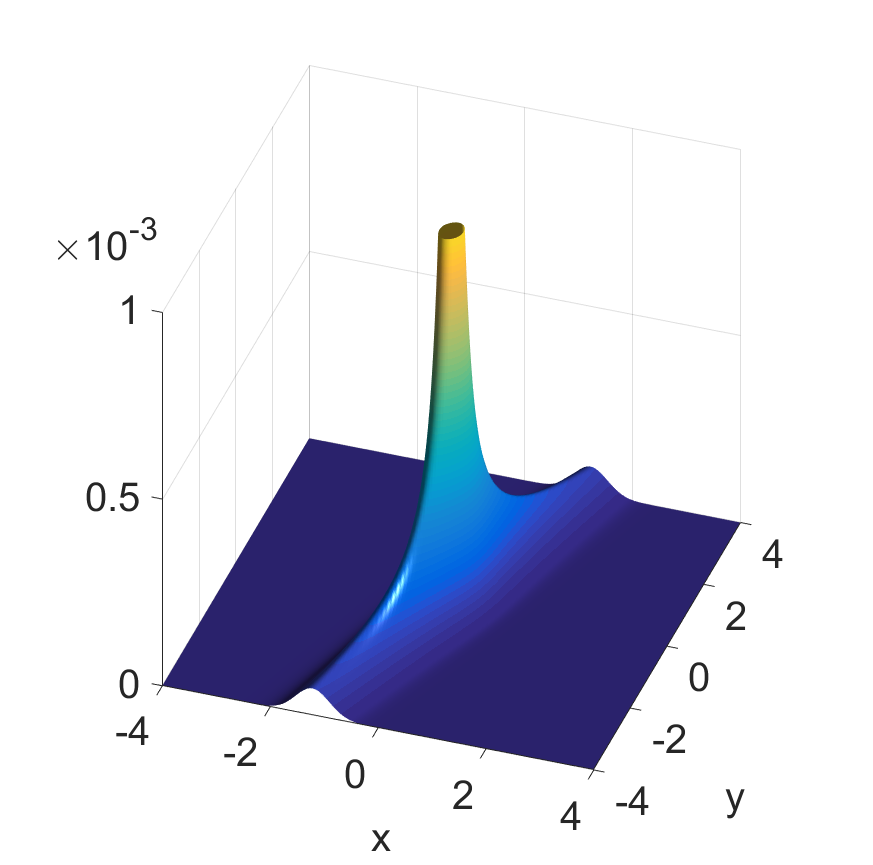}}\\
\end{tabular}
\caption{Illustration of the neighborhoods $A_{\epsilon,h}({\bf x})$ and weights \eqref{eqn:weight} used by coherence transport and Guidefill.  In each case $\epsilon = 3$px and ${\bf g}({\bf x})=(\cos73^{\circ},\sin73^{\circ})$.  Coherence transport (a) uses the lattice-aligned discrete ball $A_{\epsilon,h}({\bf x})=B_{\epsilon,h}({\bf x})$, while Guidefill (b) uses the rotated discrete ball $A_{\epsilon,h}({\bf x})=\tilde{B}_{\epsilon,h}({\bf x})$.  The ball $\tilde{B}_{\epsilon,h}({\bf x})$ is rotated so that it is aligned with the line $L$ (shown in red) passing through ${\bf x}$ parallel to ${\bf g}({\bf x})$.  In general $\tilde{B}_{\epsilon,h}({\bf x})$ contains ``ghost pixels'' lying between pixel centers, which are defined using bilinear interpolation of their ``real'' pixel neighbors.  Both use the same weights \eqref{eqn:weight} illustrated in (c).  The parameter $\mu$ controls the extent to which the weights are biased in favor of points lying on or close to the line $L$.}
\label{fig:ballRotate}
\end{figure}

Unlike in coherence transport, however, our neighbourhood $A_{\epsilon,h}({\bf x})=\tilde{B}_{\epsilon,h}({\bf x})$ is not axis aligned (unless ${\bf g}({\bf x})$ is parallel to $e_1$ or $e_2$), and this means that in general we have to evaluate $u_h$ {\em between} pixel centers, which we accomplish by extending the domain of $u_h$ at step $k$ from $\Omega_h \backslash (D^{(k)}_h \cup B_h)$ to $\Omega \backslash (D^{(k)} \cup B)$ using bilinear interpolation.  That is, we define
\begin{equation} \label{eqn:bilinearDef}
u_h({\bf x}) = \sum_{{y} \in \Omega_h } \Lambda_{{\bf y},h}({\bf x})u_h({\bf y}) \quad \mbox{ for all } {\bf x} \in \Omega \backslash (D^{(k)} \cup B),
\end{equation}
where $\{\Lambda_{{\bf y},h}\}_{{\bf y} \in \Omega_h}$ denotes the basis functions of bilinear interpolation.  Note that the continuous sets $B$ and $D^{(k)}$ have been defined so that they include a one pixel wide buffer zone around their discrete counterparts, ensuring that bilinear interpolation is well defined outside $D^{(k)} \cup B$.  The reason for the introduction of $\tilde{B}_{\epsilon,h}({\bf x})$ is to avoid a ``kinking'' phenomena whereby isophotes given a guidance direction ${\bf g}({\bf x})$ instead extrapolate along ${\bf g}^*({\bf x}) \neq {\bf g}({\bf x})$.  This is discussed in detail in Section \ref{sec:refraction} and Section \ref{sec:continuum}.  But first we describe our process of spline detection and the generation of the guide field, and how this is done in such a way as to avoid a second ``kinking'' phenomena in the computation of ${\bf g}({\bf x})$ itself. 

\subsection{Automatic Spline Detection and Creation of the Guide Field} \label{sec:guidefield}

The goal of the automatic spline detection is to position splines as straight lines in areas near the active boundary of the inpainting domain where we have detected a strong edge.  These splines are lengthened so that they extend into the inpainting domain, and may be edited by the user before being used to construct the guide field.

A one pixel wide ring $R$ is computed a small distance from $\partial_{\mbox{active}} D_h$ in the undamaged area $\Omega_h \backslash (D_h \cup B_h)$.  (As we will see in the next subsection, this dilation of $R$ from $\partial_{\mbox{active}} D_h$ is crucial for obtaining an accurate orientation of extrapolated isophotes).

We then run a version of Canny edge detection \cite{Canny86acomputational} on an annulus of pixels containing the ring, and check to see which pixels on the ring intersect a detected edge.  Portions of the annulus not labeled as belonging to the current object are ignored.  For those pixels which do intersect a detected edge, we draw a spline in the direction of the edge beginning at that pixel and extending linearly into the inpainting domain.

The direction of the edge is calculated based on the {\em structure tensor} \cite{Wei98}
\begin{equation} \label{eqn:structureTensor}
\mathcal J_{\sigma,\rho} := g_{\rho} * ( \nabla u_{\sigma} \otimes \nabla u_{\sigma} ) \qquad \mbox{ where } u_{\sigma}:=g_{\sigma} * u,
\end{equation}
(and where $g_{\sigma}$ is a Gaussian centered at ${\bf 0}$ with variance $\sigma^2$ and $*$ denotes convolution) evaluated at the point ${\bf x}_{\mbox{base}} \in R$.  In practice these convolutions are truncated to windows of size $(4\sigma+1)^2$, $(4\rho+1)^2$ respectively, so in order to ensure that $J_{\sigma,\rho}({\bf x}_{\mbox{base}})$ is computed accurately we have to ensure $R$ is far enough away from $D_h \cup B_h$ that neither patch overlaps it.  Note that our approach is different from that of coherence transport \cite{Marz2007,Marz2011} (and later adopted by Cao et al. \cite{Cao2011}) which proposes calculating a modified structure tensor directly on $\partial_{\mbox{active}} D_h$.  As we will show shortly, the modified structure tensor introduces a kinking effect and so we do not use it.  Once $J_{\sigma,\rho}({\bf x}_{\mbox{base}})$ has been calculated for a given spline $\Gamma$, we assign $\Gamma$ a direction based on the vector ${\bf g}_{\Gamma}$
$${\bf g}_{\Gamma} := \pm \tanh\left( \frac{\lambda^+-\lambda^-}{\Lambda}\right){\bf v}^-,$$ 
where $(\lambda^{\pm},{\bf v}^{\pm})$ are the maximal and minimal eigenpairs of $J_{\sigma,\rho}({\bf x}_{\mbox{base}})$ respectively, $\Lambda$ is a constant that we fix at $\Lambda = 10^{-5}$ by default, and the sign of ${\bf g}_{\Gamma}$ is chosen in order to point into $D_h$.  Then, the guide field ${\bf g}$ at a point ${\bf x} \in D_h$ is computed by first finding the spline $\Gamma_{\bf x}$ closest to ${\bf x}$, and then applying the formula
$$g({\bf x}) = {\bf g}_{\Gamma_{\bf x}} e^{- \frac{d({\bf x},\Gamma_{\bf x})^2}{2\eta^2}}$$
where $d({\bf x},\Gamma_{\bf x})$ is the distance from ${\bf x}$ to $\Gamma_{\bf x}$ and $\eta > 0$ is a constant that we set at $\eta = 3$px by default.  In practice, if $d({\bf x},\Gamma_{\bf x}) > 3\eta$ we set ${\bf g}({\bf x}) = {\bf 0}$. 

\subsubsection{Kinking Artifacts Created by the Modified Structure Tensor and their Resolution} \label{sec:2ndKink}

Coherence transport operates by computing for each ${\bf x} \in \partial D_h$ a local ``coherence direction'' ${\bf g}({\bf x})$ representing the orientation of isophotes in $\Omega_h \backslash (D_h \cup B_h)$ near ${\bf x}$.  Inspired by the success of the {\em structure tensor} \eqref{eqn:structureTensor} as a robust descriptor of the local orientation of {\em complete} images, but also noting that $\mathcal \mathcal J_{\sigma,\rho}({\bf x})$ is undefined when ${\bf x} \in \partial D_h$, the authors proposed the following {\em modified structure tensor}
\begin{equation} \label{eqn:modStructureTensor}
\hat{\mathcal J}_{\sigma,\rho}({\bf x}) := \frac{\left(g_{\rho} * \left(
1_{\Omega_h \backslash (D^{(k)}_h \cup B_h)} \nabla v_{\sigma} \otimes \nabla v_{\sigma} \right)\right)({\bf x})}{\left(g_{\rho} * 1_{\Omega_h \backslash (D^{(k)}_h \cup B_h)} \right)({\bf x})} \hskip 2mm \mbox{ where } v_{\sigma}:=\frac{g_{\sigma} * \left(  1_{\Omega_h \backslash (D^{(k)}_h \cup B_h)} u\right)}{g_{\sigma} * 1_{\Omega_h \backslash (D^{(k)}_h \cup B_h)}},
\end{equation}
which has the advantage that it is defined even for ${\bf x} \in \partial D_h$ (note the use of $v_{\sigma}$ as opposed to $u_{\sigma}$ in \eqref{eqn:modStructureTensor}.  This notation was introduced in \cite{Marz2007} because $u_{\sigma}$ is already defined in \eqref{eqn:structureTensor}).  The authors provide no theoretical justification for $\hat{\mathcal J}_{\sigma,\rho}({\bf x})$ but instead argue that it solves the problem ``experimentally''.  However, closer inspection shows that the modified structure tensor is an inaccurate description of the orientation of undamaged isophotes near ${\bf x}$ when the latter is on or near $\partial D_h$.  We illustrate this using the simple example of inpainting the lower half plane given data in the upper half plane consisting of white below the line $y = x$ and grey above ($B_h = \emptyset$ in this case).  We take $\sigma = 2$, $\rho = 4$.  This is illustrated in Figure \ref{fig:modStructureTensor}(a), where the inpainting domain is shown in red and where we also show two square neighborhoods of size $(4\sigma+1)^2$, both centered on the line $y=x$, but one centered at point $A$ on $\partial D_h$, and one at point $B \in \Omega_h \backslash D_h$ far away enough from $D_h$ that neither it nor the larger neighborhood of size $(2\rho+1)^2$ (not shown) overlap with the latter.  The core problem lies in the ``smoothed'' version $v_{\sigma}$ of $u$, which for pixel $A$ is computed based on an weighted average of pixel values {\em only in the top half of the box above $y=0$}.  Ideally, $v_{\sigma}$  sitting on the line $y=x$ should be half way between grey and white.  However, as the weights are radially symmetric and the ``angular wedge'' of the partial box centered at $A$ contains far more grey pixels than it does white, at pt $A$ we end up with a color much closer to grey.  This results in a curvature of the level curves of $v_{\sigma}$ that can be seen in Figure \ref{fig:modStructureTensor}(b).  The result is that the modified structure tensor at point $A$ has an orientation of $57 ^{\circ}$ (off by $12^{\circ}$), whereas the regular structure tensor, which is defined at point $B$ since point $B$ is far enough away from $D_h$ to be computed directly, predicts the correct orientation of $45^{\circ}$.  Figure \ref{fig:modStructureTensor}(c)-(d) show the results of inpainting using respectively the minimal eigenvalue of modified structure tensor at point $A$ and the structure tensor at point $B$ as the guidance direction.  This is why we in Section \ref{sec:guidefield} we backed our splines up from the inpainting domain and computed their orientation using the structure tensor rather than the modified structure tensor.

\begin{figure}
\centering
\begin{tabular}{cccc}
\subfloat[Points $A$, $B$ and their respective neighborhoods of size $(4\sigma+1)^2$.]{\includegraphics[width=.22\linewidth]{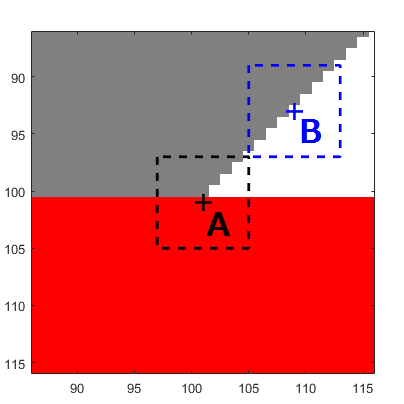}} & 
\subfloat[The isocontours of $v_{\sigma}$ used to compute the modified structure tensor \eqref{eqn:modStructureTensor} bend near point $A$.]{\includegraphics[width=.22\linewidth]{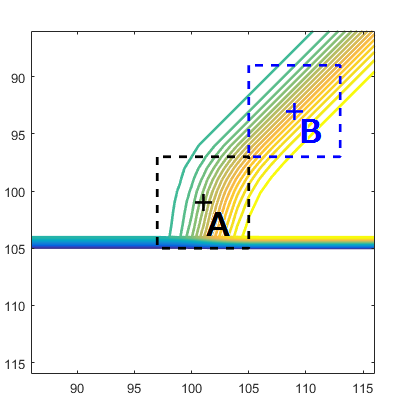}} &
\subfloat[Inpainting using Guidefill with ${\bf g}$ calculated at pt $A$ using the modified structure tensor \eqref{eqn:modStructureTensor}.]{\includegraphics[width=.20\linewidth]{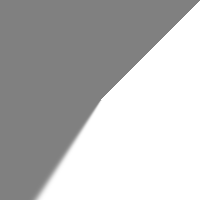}} & 
\subfloat[Inpainting using Guidefill with ${\bf g}$ calculated at pt $B$ using the ordinary structure tensor \eqref{eqn:structureTensor}.]{\includegraphics[width=.20\linewidth]{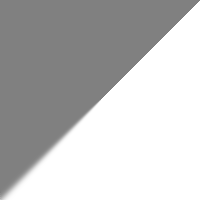}} \\
\end{tabular}
\caption{{\bf Kinking induced by the modified structure tensor.}  Consider the simple problem shown in (a) of extending a $45^{\circ}$ line into the inpainting domain (red).  A first step is to measure the orientation of this line, which coherence transport proposes to do directly on $\partial D_h$, at for example point A, using the modified structure tensor $\hat{\mathcal J}_{\sigma,\rho}$ \eqref{eqn:modStructureTensor} (with $\sigma = 2, \rho = 4$).  However, as can be seen in (b), the level lines of $v_{\sigma}$, a smoothed version of $u$ computed as an intermediate in \eqref{eqn:modStructureTensor} (as noted in the text, the notation $v_{\sigma}$ was introduced in \cite{Marz2007} because the ordinary structure tensor \eqref{eqn:structureTensor} already defines a $u_{\sigma}$), bend in the vicinity of $\partial D_h$.  The resulting guidance direction ${\bf g}_A$ (computed at $A$ using the modified structure tensor) makes an angle of $57^{\circ}$ with the horizontal, off by $12^{\circ}$ from the correct value of $45^{\circ}$ obtained by evaluating the ordinary structure tensor \eqref{eqn:structureTensor} at $B$. (c)-(d) show the results of inpainting using Guidefill with guidance directions ${\bf g}_A$ and ${\bf g}_B$ respectively.}
\label{fig:modStructureTensor}
\end{figure}

\begin{rem}

In some ways our spline-based approach resembles the earlier work by Masnou and Morel \cite{masnou1998level} and later Cao et al. \cite{Cao2011} in which level lines are interpolated across the inpainting domain by joining pairs of ``compatible T-junctions'' (level lines with the same grey value intersecting the boundary with opposite orientations).  This is done first as straight lines \cite{masnou1998level} , and later as Euler spirals \cite{Cao2011}.  An $O(N^3)$ algorithm is proposed in \cite{Cao2011} for joining compatible $T$-junctions, where $N$ is the number of such junctions.  This could be beneficial in situations such as Figure \ref{fig:guidefield}(a)-(b), where a similar process was done by hand in the editing step.  

However, our situation is different because we no longer have a simple interpolation problem - in particular, instead of an inpainting domain surrounded on both sides by useable pixels, we now typically have $D_h$ with usable pixels on one side, and bystander pixels on the other (for example, pixels belonging to some foreground object as in Figure \ref{fig:pipeline}(f)).  In some cases we might get around this by searching the perimeter of $D_h \cup B_h$, as opposed to just the perimeter of $D_h$, for compatible T-junctions.  However, this will not always work.  For example, consider the problem of inpainting a compact object in the midground partially occluded by something in the foreground.  In this case the usable pixels $\Omega_h \backslash (B_h \cup D_h)$ may be a small island entirely surrounded by $B_h \cup D_h$.  In such cases our problem is clearly no longer interpolation but extrapolation, and it doesn't make sense to talk about joining compatible T-junctions.

Nevertheless, one way of incorporating this idea would be to declare two linear splines $S_1$, $S_2$ ``compatible'' if their corresponding displacement vectors ${\bf x}_1$, ${\bf x}_2$ obey ${\bf x}_1 \cdot {\bf x}_2 \leq 0$.  Compatible splines could then be further tested by comparing patches around the base of each, with the patches rotated according to the orientation of the spline.  Those with a high match score could be joined, while the rest could remain linearly extended as before.  However, this is beyond the scope of the present work.

\end{rem}

\subsection{Resolving Additional Kinking Artifacts using Ghost Pixels} \label{sec:refraction}

\begin{figure}
\centering
\begin{tabular}{cc}
\subfloat[coherence transport ($\theta=90^{\circ}$)]{\includegraphics[width=.43\linewidth]{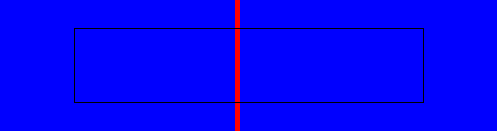}} & 
\subfloat[Guidefill ($\theta=90^{\circ}$)]{\includegraphics[width=.43\linewidth]{noKink.png}} \\
\subfloat[coherence transport ($\theta=73^{\circ}$)]{\includegraphics[width=.43\linewidth]{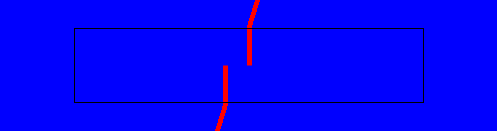}} & 
\subfloat[Guidefill ($\theta=73^{\circ}$)]{\includegraphics[width=.43\linewidth]{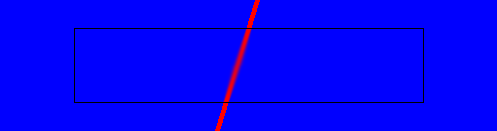}} \\
\end{tabular}
\caption{Connecting a broken lines using coherence transport (left column) and Guidefill (right column).  When the line to be extrapolated is vertical ($\theta = 90^{\circ}$), both methods are successful.  However, when the line is rotated slightly ($\theta = 73^{\circ}$) coherence transport causes the extrapolated line to ``kink'', whereas Guidefill continues to produce a successful connection.  A theoretical explanation for this phenomena is provided in Theorem \ref{thm:teaser} and illustrated in Figure \ref{fig:coherenceTheory}.}
\label{fig:connectLine}
\end{figure}

The last section showed how coherence transport can cause extrapolated isophotes to ``kink'' due to a incorrect measurement of the guidance direction ${\bf g}$, and how this is overcome in Guidefill.  In this section we briefly go over a second kinking effect that can occur even when ${\bf g}$ is known exactly, and how Guidefill overcomes this as well.  More details and a theoretical explanation are provided by our continuum analysis in Section \ref{sec:continuum}.

Figure \ref{fig:connectLine} illustrates the use of coherence transport and Guidefill - each with $\epsilon = 3$px and $\mu = 50$ - for connecting a pair of broken lines.  In each case both methods are provided the correct value of ${\bf g}$.  When the line to be extrapolated is vertical ($\theta = 90^{\circ}$), both methods are successful.  However, when the line is rotated slightly ($\theta = 73^{\circ}$) coherence transport causes the extrapolated line to ``kink'', whereas Guidefill makes a successful connection.  This happens because coherence transport is trying to bias inpainting in favor of those pixels ${\bf y}$ in the partial ball  $B_{\epsilon,h}({\bf x}) \cap (\Omega_h \backslash (D^{(k)} \cup B))$ sitting on the line $L$ passing through ${\bf x}$ in the direction ${\bf g}({\bf x})$, but in this case the {\em whole ball} $B_{\epsilon,h}({\bf x})$ contains no such pixels (other than ${\bf x}$ itself, which is excluded as it hasn't been inpainted yet) - see Figure \ref{fig:ballRotate}(a).  Instead coherence transport favors the pixel(s) {\em closest} to $L$, which in this case happens to be ${\bf y}={\bf x}+(0,h)$.  Since ${\bf y}-{\bf x}$ is in this case parallel to $(0,1)$, isophotes are extrapolated along ${\bf g}^*({\bf x}) = (0,1)$ instead of along ${\bf g}({\bf x})$ as desired.  This implies that inpainting can only be expected to succeed when ${\bf g}({\bf x})$ is of the form ${\bf g}({\bf x})=(\lambda n, \lambda m)$ for $\lambda \in \field{R}$, $n,m \in \field{Z}$ and $n^2+m^2 \leq 9$ (only a finite number of directions).

We resolve this problem by replacing $B_{\epsilon,h}({\bf x})$ with the rotated ball of ghost pixels $\tilde{B}_{\epsilon,h}({\bf x})$, which is constructed in order to contain at least one ``pixel'' on $L$ besides ${\bf x}$, as illustrated in Figure \ref{fig:ballRotate}(b).

In Figure \ref{fig:ghostReal} we also illustrate the importance of ghost pixels on a the non-synthetic example with a smoothly varying guide field shown in Figure \ref{fig:guidefield}.  When ghost pixels are not used, the extrapolated isophotes are unable to smoothly curve as only finitely many transport directions are possible.  The result is a break in the extrapolated isophote.  On the other hand, when ghost pixels are turned on we get a smoothly curving isophote with no break.

\begin{figure}
\centering
\begin{tabular}{cc}
\subfloat[Ghost pixels disabled.]{\includegraphics[width=.43\linewidth]{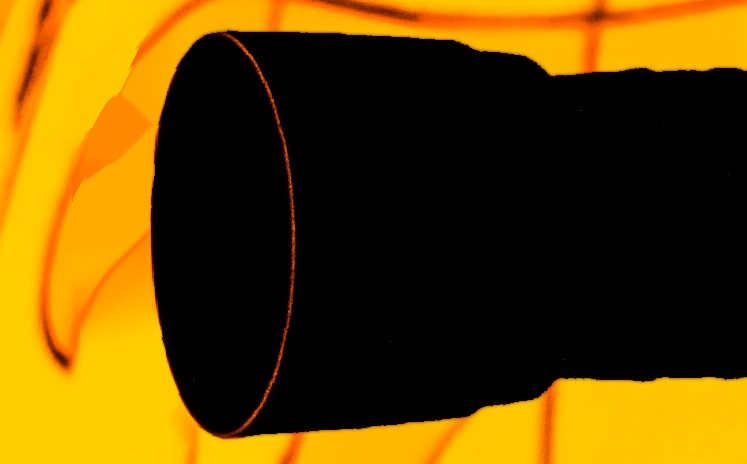}} & 
\subfloat[Ghost pixels turned on.]{\includegraphics[width=.43\linewidth]{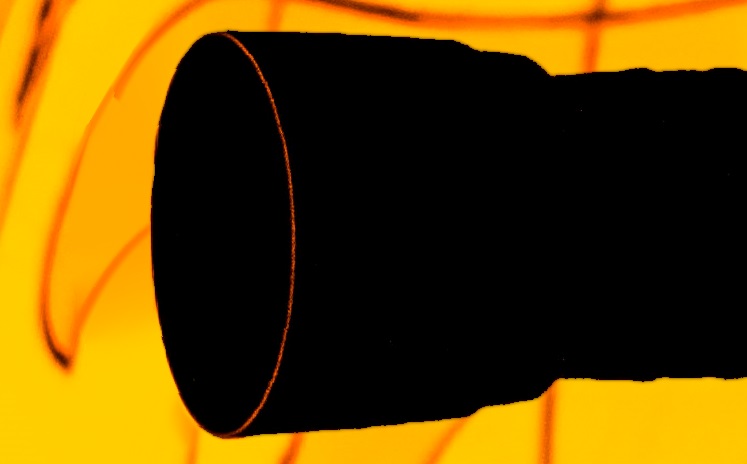}} \\
\end{tabular}
\caption{The effect of ghost pixels on a non-synthetic example ($\epsilon = 3$px, $\mu=50$).  When ghost pixels are disabled, the extrapolated isophotes are unable to smoothly curve as only finitely many transport directions are possible.}
\label{fig:ghostReal}
\end{figure}

\subsection{Automatic Determination of a Good Pixel Order (Smart Order)} \label{sec:ordering}

\begin{figure}
\centering
\begin{tabular}{cccc}
\subfloat[Onion shell order (synthetic).]{\includegraphics[width=.21\linewidth]{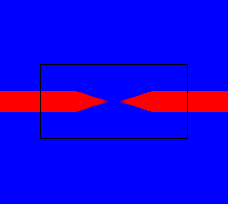}} &
\subfloat[Smart order (synthetic).]{\includegraphics[width=.21\linewidth]{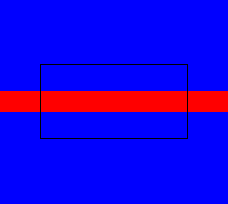}} &
\subfloat[Onion shell order (non-synthetic).]{\includegraphics[width=.21\linewidth]{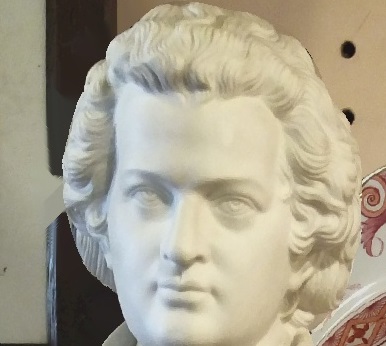}} &
\subfloat[Smart order (non-synthetic).]{\includegraphics[width=.21\linewidth]{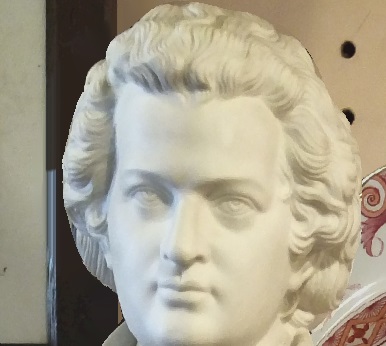}} \\
\end{tabular}
\caption{{\bf Importance of Pixel Order.} When pixels are filled in a simple ``onion shell'' order (i.e. filled as soon as they appear on the boundary of the inpainting domain), this creates artifacts including ``clipping'' of isophotes.  Our smart order \eqref{eqn:ready} avoids this by using information from the pre-computed guide field to automatically decide when pixels should be filled.}
\label{fig:goodbadOrder}
\end{figure}

Figure \ref{fig:goodbadOrder}(a) and (c) shows the result of inpainting using an ``onion shell'' fill order  (where pixels are filled as soon as they appear on the boundary of the inpainting domain), for a synthetic and non-synthetic example.  In these cases extrapolated lines are cut off due to certain pixels being filled too early.  Figure \ref{fig:goodbadOrder}(b) and (d) show the same examples using our improved fill order defined by the \mbox{ready} function \eqref{eqn:ready}.

\vskip 1mm

\noindent{\bf Review of pixel ordering strategies in the literature.}  There are at least three pixel ordering strategies for shell based inpainting methods currently in the literature.  Sun et al. \cite{Sun2005} proposed having the user draw critical curves over top of the image, and then filling patches centered on those curves first.  M\"arz \cite{Marz2011} suggested calculating non-standard distance from the boundary functions, and then filling pixels in an order based on those functions.  Finally, Criminisi et al. \cite{Criminisi04regionfilling} computes for each ${\bf p} \in \partial D_h$ a patch priority function $P({\bf p})$ as a product of a confidence term $C({\bf p})$ and a data term $D({\bf p})$, that is $P({\bf p}) = C({\bf p})D({\bf p})$ where
\begin{equation} \label{eqn:criminisiConfidence}
C({\bf p}) = \frac{\sum_{{\bf q} \in {\Psi_{\bf p} \cap (\mathcal I \backslash \Omega)}} C({\bf q})}{|\Psi_{\bf p}|} \quad \mbox{ and } D({\bf p})=\frac{|\nabla I^{\perp} _{\bf p} \cdot {\bf n}_{\bf p}|}{\alpha},
\end{equation}
where $\Psi_{\bf p}$ denotes the patch centered at ${\bf p}$, $\nabla^{\perp} I_{\bf p}$ is the orthogonal gradient to the image $I$ at ${\bf p}$, $\alpha =255$, and ${\bf n}_{\bf p}$ denotes the inward facing unit normal to the current boundary of the inpainting domain.

Patches are then filled sequentially, with the highest priority patch filled first (note that after a patch has been filled, the boundary has now changed, and certain patch priorities must be recomputed).  

\vskip 1mm

\noindent{\bf Our approach.}  The approach of M\"arz \cite{Marz2011} based on distance maps might seem the most natural - and indeed there are very simple ways one might imagine constructing a distance map given our already known splines and guidefield.  For example, distance could grow more ``slowly'' along or close to splines, while growing at the normal rate far away from splines where the guide field is zero.  However, we chose not to go to this route because we wanted to avoid the extra computational effort involved in computing such a map.  

Instead, our approach most closely resembles the approach in \cite{Criminisi04regionfilling}.  For each ${\bf x} \in \partial_{\mbox{active}} D_h$, we compute the ratio
\begin{equation} \label{eqn:confidence}
C({\bf x})=\frac{\sum_{{\bf y} \in \tilde{B}_{\epsilon,h}({\bf x}) \cap (\Omega \backslash (D^{(k)} \cup B) )}w_{\epsilon}({\bf x},{\bf y})}{\sum_{{\bf y} \in \tilde{B}_{\epsilon,h}({\bf x}) }w_{\epsilon}({\bf x},{\bf y})},
\end{equation}
where $w_{\epsilon}({\bf x},{\bf y})$ is the weight function \eqref{eqn:weight} depending implicitly on ${\bf g}$  We have suggestively suggestively named this ratio $C$ because it plays a role similar to the confidence term \eqref{eqn:criminisiConfidence}.  However, because our definition of $C({\bf x})$ also implicitly depends on the guide field {\bf g}({\bf x}), it will be small when not much information is available in the direction ${\bf g}({\bf x})$, even if the ball $\tilde{B}_{\epsilon,h}({\bf x})$ is relatively full.  In this sense it is also playing a role analogous to the data term \eqref{eqn:criminisiConfidence} which tries to ensure that the angle between $\nabla^{\perp} I_{\bf p}$ and ${\bf n}_{\bf p}$ is not too large.  Unlike Criminisi et al. \cite{Criminisi04regionfilling}, our algorithm is parallel and not sequential. Therefore, instead of every iteration filling the pixel ${\bf x} \in \partial D_h$ with the highest value of $C({\bf x})$, at every iteration we fill {\em all} pixels ${\bf x} \in \partial D_h$ for which $C({\bf x})$ is greater than a threshold.  That is, we define
\begin{equation} \label{eqn:ready}
\mbox{ ready}({\bf x}) = 1( C({\bf x}) > c)
\end{equation}
where $c > 0$ is some small user supplied constant ($c=0.05$) by default.

\vskip 1mm

\noindent{\bf Possible extensions.} Unlike \cite{Criminisi04regionfilling}, which assigns high priority to pixels with a large gradient, \eqref{eqn:ready} does not take into account the size of $\|{\bf g}({\bf x})\|$.  The result is that areas where ${\bf g}={\bf 0}$ fill concurrently with areas where $\|{\bf g}\| > 0$.  However, if one wanted to obtain an algorithm along the lines of Sun et al. \cite{Sun2005} where the region with $\|{\bf g}\| > 0$ filled first, one would only have to add a data term
$$D({\bf x}) = \|{\bf g}({\bf x})\|$$
and then modify \eqref{eqn:ready} as
\begin{equation} \label{eqn:ready2}
\mbox{ ready}({\bf x}) = 1( D({\bf x}) > c_2) 1( C({\bf x}) > c_1),
\end{equation}
where $c_1 = c = 0.05$ by default.  For $c_2$, one would take $c_2 = 0 $ initially, until it was detected that the entire region $\|{\bf g}\| > 0$ had been filled, after which point one could revert back to \eqref{eqn:ready}.

\subsection{GPU Implementation} \label{sec:GPU}

Here we sketch two GPU implementations of Guidefill, differing in how they assign GPU threads to pixels in $\Omega_h$.  In Section \ref{sec:complexity} we will analyze the time and processor complexity of these algorithms, and show that they belong to different complexity classes.  The motivation behind these algorithms is the observation that a typical HD image contains millions of pixels, but the maximum number of concurrent threads in a typical GPU is in the tens of thousands\footnote{For example, the GeForce GTX Titan X is a flagship NVIDIA GPU at the time of writing and has a total of 24 multiprocessors \cite{NvidiaHomepage} each with a maximum of 2048 resident threads \cite[Appendix G.1]{CUDAGuide}.}.  Hence, it can be advantageous to ensure that GPU threads are only assigned to the subset of pixels being currently worked on.
\begin{enumerate}
\item {\bf Guidefill without tracking.}  This implementation assigns one GPU thread per pixel in $\Omega_h$, regardless of whether or not that pixel is currently being worked on.  This implementation is simplest, but for the reason above does not scale well to very large images.
\item {\bf Guidefill with tracking.}  This implementation maintains a list of the coordinates of every pixel in $\partial_{\mbox{active}}D_h$, which it updates every iteration using a method that requires $O(|\partial_{\mbox{active}}D_h|)$ threads to do $O(\log|\partial_{\mbox{active}}D_h|)$ work each.  This extra overhead means a longer runtime for very small images, but leads to massive savings for large images as we can assign GPU threads only to pixels in $\partial_{\mbox{active}}D_h$.
\end{enumerate}
\noindent Implementation details of both methods are in the online supplementary material.

\section{Continuum Limit} \label{sec:continuum}

Here we present a special case of the analysis in our forthcoming paper \cite{Forthcoming}, which aims to provide a rigorous justification of the discussion in Section \ref{sec:refraction}.  This is accomplished by considering the continuum limit of $u_h$ as $h \rightarrow 0$ with $r := \epsilon/h \in \field{N}$, the radius of the neighborhood $A_{\epsilon,h}({\bf x})$ measured in pixels, fixed.  Note that this is different from the limit considered in \cite{Marz2007}, where first $h \rightarrow 0$ and then $\epsilon \rightarrow 0$ - see Remark \ref{remark:differences}.  Our objective is to assume enough complexity to explain the phenomena we have observed, but otherwise to keep our analysis as simple as possible.  To that end, we assume firstly a constant guide direction 
$${\bf g}({\bf x}) := {\bf g},$$
as this is all that is required to capture the phenomena in question.  More precisely, we aim to prove convergence of $u_h$, when computed by inpainting using coherence transport or Guidefill with constant guidance direction ${\bf g}$ to $u$ obeying a (weak) transport equation
\begin{equation} \label{eqn:transport}
\nabla u \cdot {\bf g}^* = 0,
\end{equation}
where ${\bf g}^* \neq {\bf g}$ in general (indeed this inequality is source of our observed ``kinking'').  We will define convergence relative to discrete $L^p$ norms defined shortly by \eqref{eqn:Lp}, and we will see that convergence is always guaranteed for $p < \infty$, but not necessarily when $p = \infty$.  We then connect this latter point back to a known issue of progressive blurring when bilinear interpolation is iterated \cite[Sec. 5]{VariationalGradient}.  

\vskip 1mm

\noindent{\bf Assumptions.}  We assume no bystander pixels ($B = \emptyset$), and that the image domain $\Omega$, inpainting domain $D$, and undamaged region $\mathcal U := \Omega \backslash D$ are all simple rectangles
$$\Omega=(0,1] \times (-\delta, 1] \qquad D = (0,1]^2 \qquad \mathcal U = (-\delta,0] \times (0,1]$$
equipped with periodic boundary conditions at $x=0$ and $x=1$.  We discretize $D = (0,1]^2$ as an $N \times N$ array of pixels $D_h=D \cap (h \cdot \field{Z}^2)$ with width $h:=1/N$.  We assume that $h < \delta/r$ so that $\epsilon < \delta$.  Given $f_h : D_h \rightarrow \field{R}$, we introduce the following discrete $L^p$ norm for $p \in [0,\infty]$
\begin{equation} \label{eqn:Lp}
\|f_h \|_p := \big(\sum_{{\bf x} \in D_h } |f_h({\bf x})|^p h^2\big)^{\frac{1}{p}}, \qquad \qquad \|f_h \|_{\infty} := \max_{{\bf x} \in D_h } |f_h({\bf x})|.
\end{equation}
We similarly define $\Omega_h = \Omega \cap (h \cdot \field{Z}^2)$, $\mathcal U_h = \mathcal U \cap (h \cdot \field{Z}^2)$,
and assume that the pixels are ordered using the default onion shell ordering, so that at each iteration $D^{(k)}_h = \{ (ih,jh) : j>k \}_{i=1}^{N}$.

\vskip 1mm

\noindent{\bf Regularity.}  In \cite{Forthcoming} we consider general boundary data $u_0 : \mathcal U \rightarrow \field{R}^d$ with low regularity assumptions, including but not limited to nowhere differentiable boundary data with finitely many jump discontinuities.  Here, we limit ourselves to piecewise $C^2$ boundary data because this is the case most relevant to image processing.  To be more precise, we assume that $u_0$ is $C^2$ everywhere on $\mathcal U$ except for on a (possibly empty) finite set of smooth curves $\{\mathcal C_i\}_{i=0}^N$ where $N \geq 0$.  We assume that the $\mathcal C_i$ intersect neither themselves nor each other, and moreover that within $0 < x \leq 1$, $-\delta < y \leq 0$ each $\mathcal C_i$ can be parametrized as a smooth monotonically increasing function $y=f_i(x)$ each of which intersects the line $y=0$ with a strictly positive angle.

\vskip 1mm

\noindent{\bf Weak Solution.}  As we have allowed discontinuous boundary data $u_0$, the solution to \eqref{eqn:transport} given boundary data $u_0$ must be defined in a weak sense.  Since we have assumed a constant guidance direction ${\bf g}({\bf x}):={\bf g}$ and due to the symmetry of the situation, the resulting transport direction ${\bf g}^*$ will also be constant (we will prove this), so this is simple.  So long as ${\bf g}^* \cdot e_2 \neq 0$, we simply define the solution to the transport problem \eqref{eqn:transport} with boundary conditions $u(x,0)=u_0(x,0)$, u(0,y)=u(1,y) to be 
\begin{equation} \label{eqn:weak}
u(x,y) = u_0(x-\cot(\theta^*)y \mbox{ mod } 1,0)
\end{equation}
where the $\mbox{mod}$ $1$ is due to our assumed periodic boundary conditions and where
$$\theta^* = \theta({\bf g}^*) \in (0,\pi)$$
is the counterclockwise angle between the x-axis and a line parallel to ${\bf g}^*$.

\begin{theorem} \label{thm:teaser}  Let the image domain $\Omega$, inpainting domain $D$, undamaged region $\mathcal U$, as well as their discrete counterparts, be defined as above.  Similarly, assume the boundary data $u_0 : \mathcal U \rightarrow \field{R}^d$ obeys the assumed regularity conditions above, in particular that it is $C^2$ except for on a finite, possibly empty set of smooth curves $\{ \mathcal C_i \}_{i=1}^N$, $N \geq 0$ with the assumed properties.

Assume $D_h$ is inpainted using Algorithm 1, with neighbourhood $$A_{\epsilon,h}({\bf x}) \in \{ B_{\epsilon,h}({\bf x}), \tilde{B}_{\epsilon,h}({\bf x}) \},$$ (that is, either the neighborhood used by coherence transport or the one used by Guidefill).  Let $w_{\epsilon}({\bf x},{\bf y})$ be given by \eqref{eqn:weight} with guidance direction ${\bf g}({\bf x}):={\bf g}$ constant.  Suppose we fix $r := \epsilon/h \in \field{N}$, assume $r \geq 2$ (that is, the radius of $A_{\epsilon,h}({\bf x})$ is at least two pixels) and consider $h \rightarrow 0$.  Define the transport direction ${\bf g}^*$ by
\begin{equation} \label{eqn:limitDirection}
{\bf g}^* = \frac{\sum_{{\bf y} \in A^-_r} w_r({\bf 0},{\bf y}){\bf y}}{\sum_{{\bf y} \in A^-_r} w_r({\bf 0},{\bf y})} \qquad A^-_r:=\{(y_1,y_2) \in \frac{1}{h}A_{\epsilon,h}(0) : y_2 \leq -1 \}.
\end{equation}
Note that $A^-_r$ depends only on $r$.  Also note that we have written $w_r$ to mean the weights \eqref{eqn:weight} with $\epsilon$ replaced by $r$.  Let $u : (0,1]^2 \rightarrow \field{R}^d$ denote the weak solution \eqref{eqn:weak} to the transport PDE \eqref{eqn:transport} with transport direction ${\bf g}^*$ and with boundary conditions $u(x,0) = u_0(x,0)$, u(0,y)=u(1,y).

Then $u$ exists and for any $p \in [1,\infty]$ and for each channel\footnote{Remember, $u$, $u_0$, and $u_h$ are all vector valued.  We could have made this more explicit by writing $u^{(i)}-u^{(i)}_h$ in \eqref{eqn:Lpbound}, \eqref{eqn:uniformBound} to emphasize that it holds channel-wise for each $i=1,\ldots,d$, but felt that this would lead to too much clutter.} of $u$, $u_h$ we have the bound
\begin{equation} \label{eqn:Lpbound}
\|u-u_h\|_p \leq K h^{\frac{1}{2p}}
\end{equation}
if $\{ \mathcal C_i \}$ is non-empty and 
\begin{equation} \label{eqn:uniformBound}
\|u-u_h\|_p \leq K h.
\end{equation}
independent of $p$ otherwise (that is, if $u_0$ is $C^2$ everywhere).  Here $K$ is a constant depending only on $u_0$ and $r$.
\end{theorem}

\begin{rem}
The transport direction ${\bf g}^*$ predicted by Theorem \ref{thm:teaser} has a simple geometric interpretation.  It is the average position vector {\em or center of mass} of the set $A^-_r$ with respect to the normalized weights $w_r$ \eqref{eqn:weight}.  This is true regardless of whether or not $A^-_r$ is axis aligned.  For coherence transport and Guidefill, we give the set $A^-_r$ the special names $b^-_r$ and $\tilde{b}^-_r$ respectively.  For ${\bf g} \neq {\bf 0}$ they are given by
\begin{eqnarray*} 
b^-_r  &:=& \{ (n,m) \in \field{Z}^2 : n^2+m^2 \leq r^2, m \leq -1 \}  \\
\tilde{b}^-_r &:=& \{ n\hat{\bf g}+m\hat{\bf g}^{\perp} : (n,m) \in \field{Z}^2, n^2+m^2 \leq r^2,  n\hat{\bf g}\cdot e_2+m\hat{\bf g}^{\perp} \cdot e_2 \leq -1 \}. 
\end{eqnarray*}
where $\hat{\bf g} := {\bf g}/\|{\bf g}\|$ (if ${\bf g} = {\bf 0}$ we set $\tilde{b}^-_r = b^-_r$).  These sets can be visualized by looking at the portion of the balls in Figure \ref{fig:ballRotate}(a)-(b) below the line $y=-1$.  The limiting transport directions for coherence transport and Guidefill - denoted by ${\bf g}^*_{\mbox{c.t.}}$ and ${\bf g}^*_{\mbox{g.f.}}$ respectively - are then given by
\begin{equation} \label{eqn:ctgf_limits}
{\bf g}^*_{\mbox{c.t.}} = \frac{\sum_{{\bf y} \in b^-_r} w_r({\bf 0},{\bf y}){\bf y}}{\sum_{{\bf y} \in b^-_r} w_r({\bf 0},{\bf y})} \qquad \mbox{and} \qquad {\bf g}^*_{\mbox{g.f.}} = \frac{\sum_{{\bf y} \in \tilde{b}^-_r} w_r({\bf 0},{\bf y}){\bf y}}{\sum_{{\bf y} \in \tilde{b}^-_r} w_r({\bf 0},{\bf y})}.
\end{equation}
Although these formulas differ only in the replacement of a sum over $b^-_r$ with a sum of over $\tilde{b}^-_r$, this difference is significant, as is explored in Figure \ref{fig:coherenceTheory}.
\end{rem}

\begin{proof}
Here we prove the easy case where $u_0$ is $C^2$ everywhere and $A_{\epsilon,h}({\bf x})$ contains no ghost pixels, that is $A_{\epsilon,h}({\bf x}) \subset \field{Z}^2_h$.  For the case where $A_{\epsilon,h}({\bf x})$ contains ghost pixels lying between pixel centers and for $u_0$ with lower regularity, we refer the reader to \cite{Forthcoming}.  We also only prove the case $p=\infty$, as $p < \infty$ follows trivially as the bound is independent of $p$ in this case.  We use the notation ${\bf x}:=(ih,jh)$ interchangeably throughout.

First note that the symmetry of the situation allows us to rewrite the formula in Algorithm 1 as
$$u_h({\bf x}) = \frac{\sum_{{\bf y} \in A^-_r }w_r({\bf 0},{\bf y})u_h({\bf x}+{\bf y}h)  }{\sum_{{\bf y} \in A^-_r }w_r({\bf 0},{\bf y})}.$$
Next we note that $A^-_r$ is nonempty, which follows from our assumption $A_{\epsilon,h}({\bf x}) \in \{ B_{\epsilon,h}({\bf x}), \tilde{B}_{\epsilon,h}({\bf x}) \}$ and $r \geq 2$ (we leave it as an exercise to the reader that no matter how we rotate $\tilde{B}_{\epsilon,h}({\bf x})$, this is always true).
Since $A^-_r \neq \emptyset$, it follows that ${\bf g}^*$ \eqref{eqn:limitDirection} is defined, and moreover ${\bf g}^* \cdot e_2 \neq 0$.  This was the condition we needed to ensure that $u$ is defined.

Now that we know $u$ exists, let us define $e_h := u_h - u$.  Then it suffices to prove 
\begin{equation} \label{eqn:linearBound}
|e_h({\bf x})| \leq K h
\end{equation}
for all ${\bf x} \in D_h$, where $K > 0$ is a constant independent of ${\bf x}$.  To prove this, we make use of the fact that since $u_0$ is $C^2$, $u$ is as well and so there is a $D > 0$ s.t. $\|Hu\|_2 \leq D$ uniformly on $(0,1]^2$, where $Hu$ denotes the Hessian of $u$ and $\| \cdot \|_2$ is usual operator norm induced by the vector 2-norm (moreover, this $D$ depends only on $u_0$).  We will use this to prove the stronger condition that for any $1\leq i,j\leq N$ we have
\begin{equation} \label{eqn:strongerCondition}
|e_h(ih,jh)| \leq j Dr^2h^2,
\end{equation}
from which \eqref{eqn:linearBound} follows with $K = Dr^2$ since $j \leq N = 1/h$.

We proceed by induction, supposing that \eqref{eqn:strongerCondition} holds for all $(i'h,j'h)$ with $1\leq i' \leq N$ and $j'<j$ (the base case $j=0$ is obvious).  Applying our inductive hypothesis and expanding $u$ to second order we obtain:
\begin{eqnarray*}
|e_h(ih,jh)| &\leq& \frac{\sum_{{\bf y} \in A^-_r }w_r({\bf 0},{\bf y})|e_h({\bf x}+{\bf y}h)|  }{\sum_{{\bf y} \in A^-_r }w_r({\bf 0},{\bf y})}+\left|\frac{\sum_{{\bf y} \in A^-_r }w_r({\bf 0},{\bf y})u({\bf x}+{\bf y}h)-u({\bf x})  }{\sum_{{\bf y} \in A^-_r }w_r({\bf 0},{\bf y})}\right| \\
& \leq & (j-1)Dr^2h^2 + \left| \nabla u({\bf x}) \cdot \frac{\sum_{{\bf y} \in A^-_r} w_r({\bf 0},{\bf y}){\bf y}h}{\sum_{{\bf y} \in A^-_r} w_r({\bf 0},{\bf y})}\right| + Dr^2h^2 \\
& = & jDr^2h^2 + |h\underbrace{\nabla u({\bf x}) \cdot {\bf g}^*}_{=0}|,
\end{eqnarray*}
where we have used the fact that when ${\bf x}=(ih,jh)$ and ${\bf y} \in A^-_r$, then $({\bf x}+{\bf y}h)$ is of necessary form $(i'h,j'h)$ with $1\leq i' \leq N$ and $j' < j$ needed for our inductive hypothesis to hold.
\end{proof}

\begin{figure}
\centering
\begin{tabular}{cc}
\subfloat[coherence transport, $r=3$.]{\includegraphics[width=.43\linewidth]{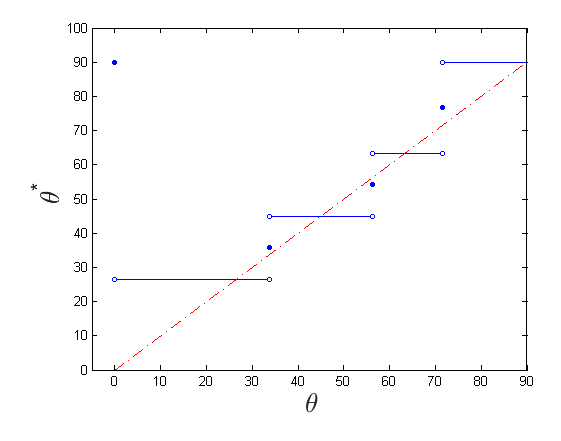}} & 
\subfloat[coherence transport, $r=5$.]{\includegraphics[width=.43\linewidth]{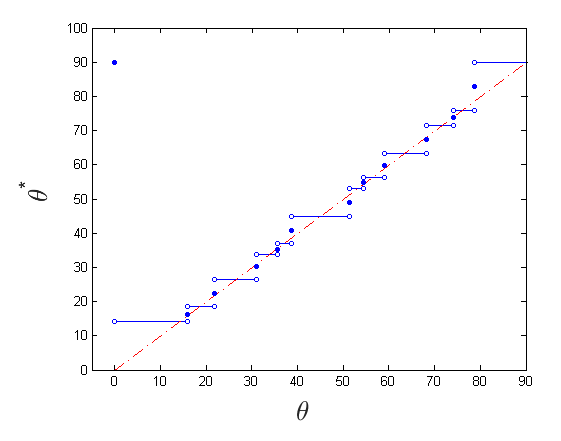}} \\
\subfloat[Guidefill, $r=3$.]{\includegraphics[width=.43\linewidth]{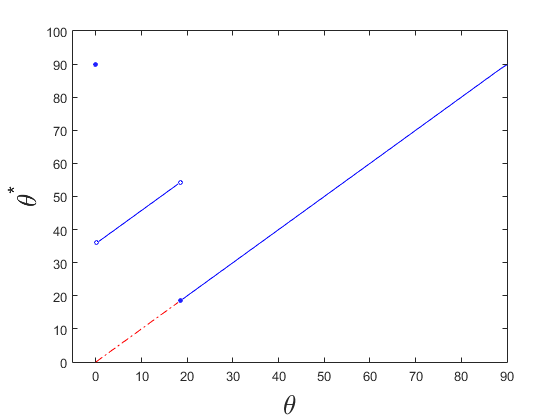}} & 
\subfloat[Guidefill, $r=5$.]{\includegraphics[width=.43\linewidth]{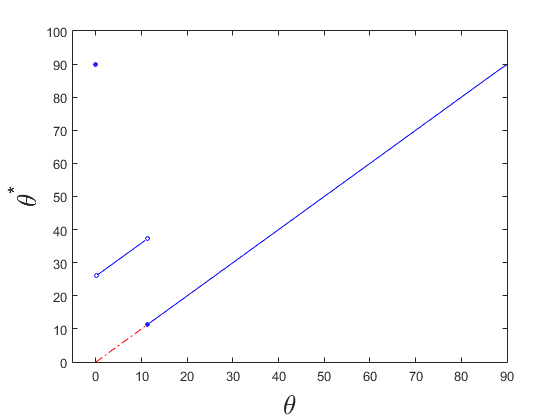}} \\
\end{tabular}
\caption{The theoretical limiting curves $\theta^* =\theta({\bf g}^*_{\mbox{c.t.}})$ (coherence transport (a)-(b)) and $\theta^* =\theta({\bf g}^*_{\mbox{g.f.}})$ (Guidefill (c)-(d)) as a function of $\theta = \theta({\bf g})$, with ${\bf g}^*_{\mbox{c.t.}}$ and ${\bf g}^*_{\mbox{g.f.}}$ given by \eqref{eqn:ctgf_limits}, and where ${\bf g}$ is the desired guidance direction fed into the weights \eqref{eqn:weight}.  We set $r:=\epsilon/h=3,5$ and consider $\mu \rightarrow \infty$.  The ideal curve $\theta^*=\theta$ is highlighted in red.  The limiting guide directions ${\bf g}^*_{\mbox{c.t.}}$ and ${\bf g}^*_{\mbox{g.f.}}$ are related by \eqref{eqn:limitDirection} to the weights \eqref{eqn:weight} as well as the distribution of sample points within $A_{\epsilon,h}({\bf x})$.   Coherence transport makes the choice $A_{\epsilon,h}({\bf x})=B_{\epsilon,h}({\bf x})$, leading to the ``kinking'' observed in (a)-(b), where $\theta^* \neq \theta$ for all but finitely many angles.  The choice $A_{\epsilon,h}({\bf x})=\tilde{B}_{\epsilon,h}({\bf x})$ made by Guidefill is largely able to avoid this and exhibits no kinking for all angles greater than a critical minimum - see Remark \ref{remark:smartOrder} as well as \cite{Forthcoming} for more details.}
\label{fig:coherenceTheory}
\end{figure}

\subsection{Consequences} \label{sec:consequences}

Theorem \ref{thm:teaser} helps us to understand two important features of the class of algorithms under study.  Firstly, it helps us to understand a kinking phenomena Guidefill aims to overcome.  Secondly, it will help us to understand a new phenomena Guidefill introduces (and indeed a limitation of the method) - the gradual degradation of signal due to repeated bilinear interpolation.

\subsubsection{Kinking}  Figure \ref{fig:coherenceTheory} illustrates the significance of Theorem \ref{thm:teaser} by plotting the phase $\theta({\bf g}^*_{\mbox{c.t.}})$ and $\theta({\bf g}^*_{\mbox{g.f.}})$ of the theoretical limiting transport directions of coherence transport and Guidefill respectively \eqref{eqn:ctgf_limits} as a function of the phase $\theta({\bf g})$ of the guidance direction ${\bf g}$.  The cases $\epsilon = 3h$ and $\epsilon = 5h$ are considered (coherence transport \cite{Marz2007} recommends $\epsilon = 5h$ by default) with $\mu \rightarrow \infty$.  For coherence transport we have $\theta({\bf g}^*) \neq \theta({\bf g})$ except for finitely many angles, explaining the kinking observed in practice.  On the other hand, for Guidefill we have $\theta({\bf g}^*) = \theta({\bf g})$ (in other words no kinking) for all angles greater than a minimum value.  We refer the reader to \cite{Forthcoming} for additional details.

\begin{rem} \label{remark:smartOrder}

In order to understand the kinking of Guidefill shown in Figure \ref{fig:coherenceTheory}(c)-(d) at small angles for ${\bf g} \neq {\bf 0}$ and $\mu \gg 1$, it is helpful to consider the decomposition
$$\tilde{B}_{\epsilon,h}({\bf x}) = \ell_{\epsilon,h}({\bf x}) \cup (\tilde{B}_{\epsilon,h}({\bf x}) \backslash \ell_{\epsilon,h}({\bf x}) ) \quad \mbox{ where } \quad \ell_{\epsilon,h}({\bf x}) := \left \{ \epsilon k \hat{\bf g} \right \}_{k=-r}^r,$$
where $\hat{\bf g}:= {\bf g}/\|{\bf g}\|$,  and where $r:=\epsilon/h \in \field{N}$ as usual.  The kinking observed for low angles in Figure \ref{fig:coherenceTheory}(c)-(d) occurs when $\ell_{\epsilon,h}({\bf x})$ contains no readable pixels, that is
\begin{equation} \label{eqn:emptyline}
\ell_{\epsilon,h}({\bf x}) \cap ( \Omega_h \backslash (D^{(k)}_h \cup B_h) ) = \emptyset.
\end{equation}
Since the weights \eqref{eqn:weight} concentrate most of their mass on $\ell_{\epsilon,h}({\bf x})$ when $\mu$ is large, in this case we can expect the confidence term \eqref{eqn:confidence} to be small and the smart order proposed in Section \ref{sec:ordering} will tell the algorithm to delay the filling of ${\bf x}$.  This is guaranteed so long as $\mu$ is sufficiently large.  If at a later iteration \eqref{eqn:emptyline} no longer holds, 
\eqref{eqn:ready} should be satisfied and inpainting can resume with no kinking.
\end{rem}

\begin{rem} \label{remark:differences}
The limiting transport direction ${\bf g}^*$ predicted by Theorem \ref{thm:teaser} is similar to the transport direction predicted by M\"arz in \cite{Marz2007} (Theorem 1).  The key difference is that while M\"arz considered the double limit where $h \rightarrow 0$ and then $\epsilon \rightarrow 0$, we consider the single limit $(h,\epsilon) \rightarrow (0,0)$ with $r=\epsilon/h$ fixed, which we argue in \cite{Forthcoming} is more relevant.  The result is that whereas \cite{Marz2007} obtains a formula for ${\bf g}^*$ as an integral over a (continuous) half-ball, our ${\bf g}^*$ is a finite sum over a discrete half-ball.  In particular, when $A_{\epsilon,h}({\bf x})=B_{\epsilon,h}({\bf x})$ as in coherence transport, the following predictions are obtained for the limiting transport direction (note that we write $w_r$ and $w_1$ to mean the weights \eqref{eqn:weight} with $\epsilon$ replaced by $r$ and $1$ respectively):
$${\bf g}^*_{\mbox{M\"arz}}=\frac{\int_{ {\bf y} \in B^-_1({\bf 0}) } w_1({\bf 0},{\bf y}){\bf y}d{\bf y}}{\int_{ {\bf y} \in B^-_1({\bf 0}) } w_1({\bf 0},{\bf y})} \qquad {\bf g}^*_{\mbox{ours}}= \frac{\sum_{{\bf y} \in b^-_r} w_r({\bf 0},{\bf y}) {\bf y}}{\sum_{{\bf y} \in b^-_r} w_r({\bf 0},{\bf y})},$$
where
\begin{eqnarray*}
B_1^-({\bf 0}) &:=& \{ (x,y) \in \field{R}^2 : x^2+y^2 \leq 1 \mbox{ and } y <0\} \\
 b^-_r &:=& \{ (n,m) \in \field{Z}^2 : n^2+m^2 \leq r^2 \mbox{ and } m \leq -1\}.
\end{eqnarray*}
Our discrete sum ${\bf g}^*_{\mbox{ours}}$ predicts the kinking observed by coherence transport in practice, whereas the integral $ {\bf g}^*_{\mbox{M\"arz}}$ does not.
\end{rem}

\vskip 1mm

\subsubsection{ Signal degradation}  \label{sec:signalDegredation}
 Theorem \ref{thm:teaser} says that when $u_0$ has jump discontinuities, we can expect convergence in $L^p$ for all $1 \leq p < \infty$, but potentially with a gradually deteriorating rate and with no guarantee of convergence when $p = \infty$.  This suggests that our method may have a tendency to gradually blur an initially sharp signal.  Indeed, our method is based on bilinear interpolation and a known property of the repeated application of bilinear interpolation is to do just that - see for example \cite[Sec. 5]{VariationalGradient}.  Moreover, this blurring is plainly visible in Figure \ref{fig:modStructureTensor}(c)-(d).
 
To explore this phenomena, we considered the continuum problem of inpainting the line
$$\tan(73^{\circ}) - 0.1 \leq y \leq \tan(73^{\circ}) + 0.1$$
over the image domain $\Omega = [-1,1] \times [-0.5,0.5]$ and inpainting domain $D = [-0.8,0.8] \times [-0.3,0.3]$.  We then used Guidefill to solve the discrete inpainting problem at four different image resolutions - $200\times 100$px, $400\times 200$px, $4000\times 2000$px, and $8000\times 4000$px.  In each case, we examined a horizontal cross section of the solution at three places - $y= 0.3$, the boundary of the inpainting domain where the signal is perfect, $y=0.25$, a short distance inside the inpainting domain, and $y=0$, the midpoint of the domain where we can expect maximal deterioration.  The results are given in Figure \ref{fig:degradation}.

There is indeed signal degradation, most significant for low resolution problems, and it does indeed get worse as we move further into the inpainting domain.  This is a limitation of our method, especially for thin edges to be extrapolated across long distances.  However, also note that as predicted by Theorem \ref{thm:teaser}, when we increase the image resolution the degree of degradation drops significantly, even though we are applying many more bilinear interpolation operations.

\begin{figure}
\centering
\begin{tabular}{cc}
\subfloat[]{\includegraphics[width=.43\linewidth]{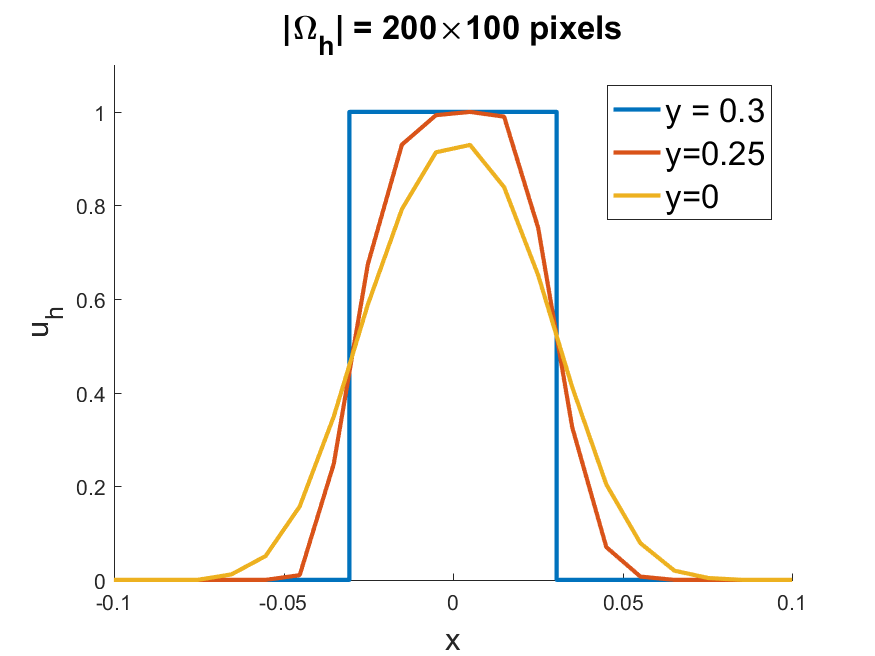}} & 
\subfloat[]{\includegraphics[width=.43\linewidth]{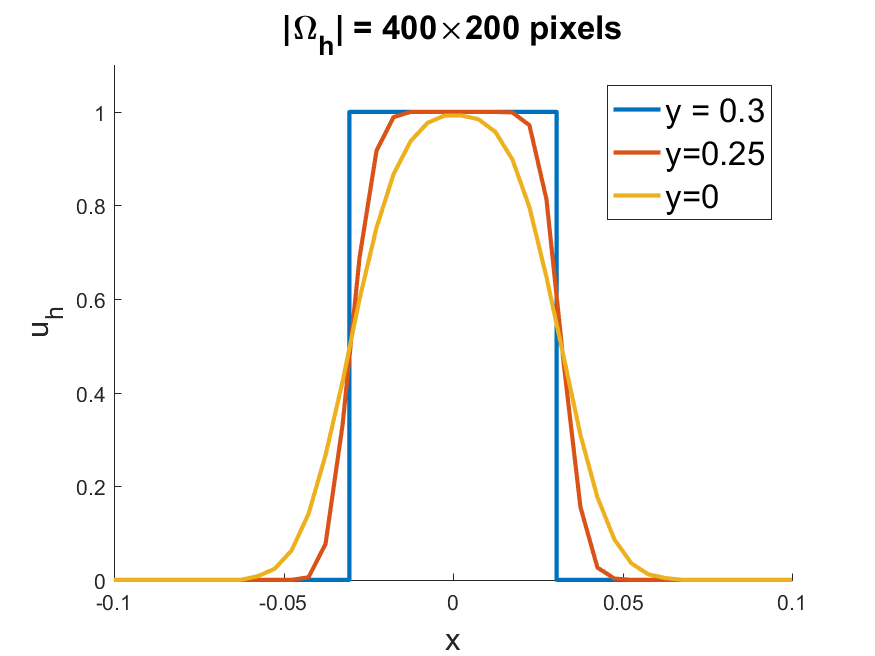}} \\
\subfloat[]{\includegraphics[width=.43\linewidth]{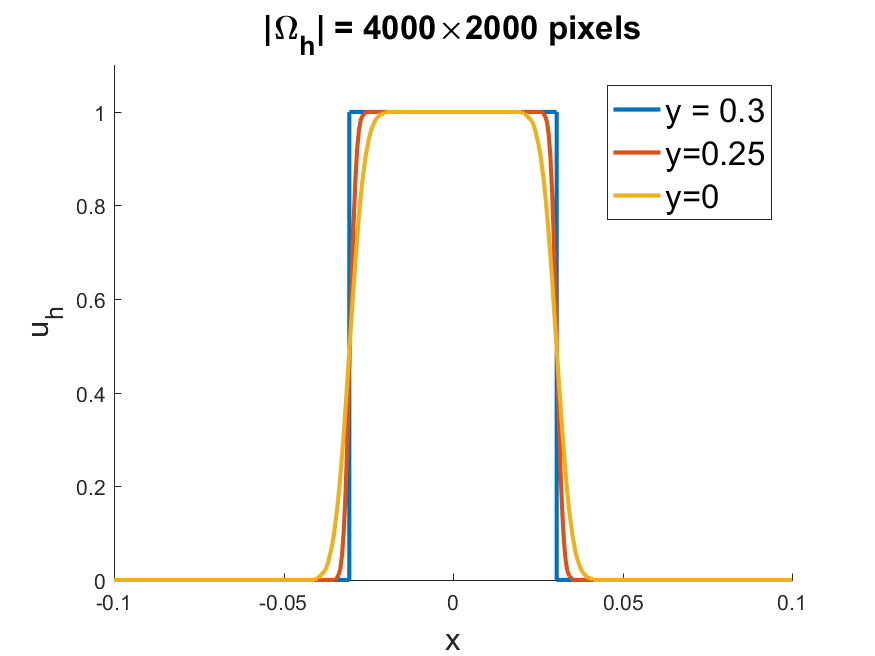}} & 
\subfloat[]{\includegraphics[width=.43\linewidth]{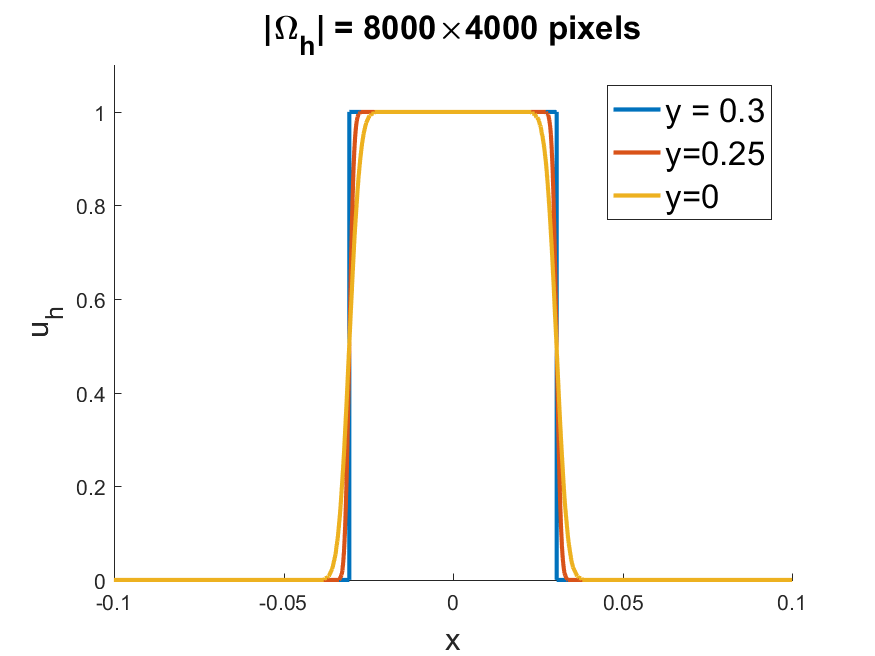}} \\
\end{tabular}
\caption{{\bf Signal degradation and $L^p$ convergence of Guidefill:}  The continuum problem of inpainting the line $\tan(73^{\circ}) - 0.1 \leq y \leq \tan(73^{\circ}) + 0.1$ with image domain $\Omega=[-1,1] \times [-0.5,0.5]$ and inpainting domain $D = [-0.8,0.8] \times [-0.3,0.3]$ is rendered at a variety of resolutions and inpainted each time using Guidefill.  Examining cross-sections of $u_h$ at $y=0.3$ (on the boundary of $D_h$), $y=0.25$ (just inside), and $y=0$ (in the middle of $D_h$) we notice a gradual deterioration of the initially sharp signal.  This deterioration is to be expected as our method is based on iterated bilinear interpolation, which is known to have this effect \cite[Sec. 5]{VariationalGradient}.  However, also note that in accordance with Theorem \ref{thm:teaser}, the signal is less degraded in higher resolution images, even though we have applied more bilinear interpolation operators.}
\label{fig:degradation}
\end{figure}

\section{Algorithmic Complexity} \label{sec:complexity}

In this section we analyze the complexity the two implementations of Guidefill sketched in Section \ref{sec:GPU} as parallel algorithms.  Specifically, we analyze how both the {\em time complexity} $T(N,M)$ and {\em processor complexity} $P(N,M)$ vary with $N = |D_h|$ and $M=|\Omega_h \backslash D_h |$, where a time complexity of $T(N,M)$ and processor complexity of $P(N,M)$ means that the algorithm can be completed by $O(P(N,M))$ processors in $O(T(N,M))$ time per processor.  See for example \cite[Ch. 5]{ParallelBook} for a more detailed discussion of the time and processor complexity formalism for parallel algorithms.

We assume that Guidefill is implemented on a parallel architecture consisting of $p$ processors working at the same time in parallel.  We further assume that when Guidefill attempts to run $P > p$ parallel threads such that there are not enough available processors to comply, the $P$ threads are run in $\lceil P/p \rceil$ sequential steps.  In reality, GPU architecture is not so simple  - see for example \cite[Ch. 4]{CUDAGuide} for a discussion of GPU architecture, and for example \cite{GPUAnalysis} for a more realistic theoretical model.  We do not consider these additional complexities here.

In Theorem \ref{thm:complexity} we derive a relationship between the time and processor complexities $T(N,M)$, $P(N,M)$ and the number of iterations $K(N)$ required for Guidefill to terminate.  This relationship is valid in general but does not allow us to say anything about $K(N)$ itself.  Next, in Theorem \ref{thm:complexity2} we establish bounds on $K(N)$ under a couple of simplifying assumptions.  Firstly, we assume that the inpainting domain is surrounded entirely by readable pixels - that is $(\partial_{\mbox{outer}} D_h) \cap B_h = \emptyset$.  In particular this means that we assume the inpainting domain does not include the edge of the image and is not directly adjacent to pixels belonging to another object (such as an object in the foreground).  Secondly, we assume that the smart ordering of Section \ref{sec:ordering} is turned off.  We also include a discussion in the online supplementary material of what to expect in the general case.  Our analysis considers only the filling step of Guidefill after the guide field has already been constructed.

\begin{theorem} \label{thm:complexity}

Let $N = |D_h|$, $M=|\Omega_h \backslash D_h |$ denote the problem size and let $T(N,M)$ and $P(N,M)$ denote the time complexity and processor complexity of the filling step of Guidefill implemented on a parallel architecture as described above with $p$ available processors.  Let $K(N)$ denote the number of iterations before Guidefill terminates.  Then the processor complexity of Guidefill with  and without  boundary tracking is given by
$$P(N,M) =  \begin{cases}   
O(N+M) & \mbox{ without tracking} \\ 
O(\sqrt{N+M}) & \mbox{ with tracking}
\end{cases}$$
while the time complexity is given by
$$ T(N,M) =  \begin{cases}   
O(K(N)) & \mbox{ if } P(N,M) \leq p \\ 
O((N+M)K(N)) & \mbox{ if } P(N,M) > p
\end{cases} \quad \mbox{ without tracking}$$
$$T(N) =  \begin{cases}   
O((\sqrt{N}+K(N))\log(N)) & \mbox{ if } P(N,M) \leq p \\ 
O((N+K(N))\log(N)) & \mbox{ if } P(N,M) > p
\end{cases} \quad \mbox{with tracking}.$$
\end{theorem}
\begin{proof}  For the case of no boundary tracking Guidefill allocates one thread per pixel in $\Omega_h$, hence $P(N,M)=O(|\Omega_h|)=O(N+M)$.  In this case if $|\Omega_h| := N+M < p$, then each thread fills only one pixel, and hence does $O(1)$ work.  On the other hand, if $N+M > p$, each thread must fill $\lceil \frac{N+M}{p}\rceil$ pixels.  It follows that 
$$T(N,M) \leq \sum_{k=1}^{K(N)} \left \lceil \frac{N+M}{p}\right \rceil  \leq \begin{cases} K(N) & \mbox{ if } N+M < p \\ \frac{2}{p}(N+M)K(N) & \mbox{ otherwise.} \end{cases}.$$
On the other hand, Guidefill with tracking allocates $O(|\partial D^{(k)}_h|)$ threads per iteration of Guidefill, each of which do $O(\log|\partial D^{(k)}_h|)$ work.  This is because, as stated in Section \ref{sec:GPU}, the boundary is updated over a series of $O(\log|\partial D^{(k)}_h|)$ parallel steps.  In order to keep the processor complexity at $O(\sqrt{N+M})$, we assume that in the unlikely event that more than $\sqrt{N+M}$ threads are requested, then Guidefill runs them in $O\left(\left \lceil \frac{|\partial D^{(k)}_h|}{\sqrt{N+M}}\right \rceil\right)$ sequential steps each involving $\sqrt{N+M}$ processors.  We therefore have, for $\sqrt{N+M} < p$
$$T(N,M) \leq  \sum_{k=1}^{K(N)}  \left( \frac{|\partial D^{(k)}_h|}{\sqrt{N+M}} +1 \right) C \log|\partial D^{(k)}_h| \leq C \log(N) \sum_{k=1}^{K(N)}  \left( 1+\frac{|\partial D^{(k)}_h|}{\sqrt{N+M}}\right)$$
where the factor $C > 0$ comes from the hidden constants in the Big O notation.  But we know that $\{ \partial D^{(k)}_h\}^{K(N)}_{k=1}$ form a partition of $D_h$, so that $ \sum_{k=1}^{K(N)} |\partial D^{(k)}_h| = N$.  Therefore
\begin{displaymath}
T(N,M) \leq C \log(N)\left( K(N) + \frac{N}{\sqrt{N+M}} \right) \leq C(\sqrt{N}+K(N))\log(N).
\end{displaymath}
An analogous argument with $\sqrt{N+M}$ in the denominator replaced by $p$ handles the case $P(N,M) > p$.
\end{proof}

\begin{theorem} \label{thm:complexity2}

If we make the same assumptions as in Theorem \ref{thm:complexity} and if we further suppose $(\partial_{\mbox{outer}} D_h ) \cap B_h = \emptyset$ and the smart order test from Section \ref{sec:ordering} is turned off, then we additionally have 
\begin{equation} \label{eqn:sqrtIterations}
K(N) = O(\sqrt{N})
\end{equation}
so that, in particular, we have $T(N,M) = O(\sqrt{N})$, $T(N,M) = O(\sqrt{N}\log(N))$ for Guidefill without and with tracking given sufficient processors, and \\$T(N,M) = O((N+M)\sqrt{N})$, $T(N,M) = O(N\log(N))$ respectively when there is a shortage of processors.
\end{theorem}
\begin{proof}
Now assume $(\partial_{\mbox{outer}} D_h ) \cap B_h = \emptyset$ and \eqref{eqn:ready} is disabled.  Then after $k$ iterations all pixels ${\bf x}$ such that $\mathcal N^{(k)}({\bf x}) \cap \Omega_h \backslash D_h \neq \emptyset$ will have been filled, where 
$$\mathcal N^{(k)}({\bf x}) = \bigcup_{{\bf y} \in \mathcal N^{(k-1)} ({\bf x})} \mathcal N({\bf y}), \quad \mathcal N^{(1)}({\bf x})=\mathcal N({\bf x}).$$
Therefore, if $D_h$ has not be completely filled after $k$ iterations, there must exist a pixel ${\bf x}^* \in D_h$ such that $\mathcal N^{(k)}({\bf x}^*) \subseteq D_h$.  However, it is easy to see that $|\mathcal N^{(k)}({\bf x}^*)|=(2k+1)^2$.  But since $|D_h|=N$, after $k = \lceil \sqrt{N}/2 \rceil$ iterations $\mathcal N^{(k)}({\bf x}^*)$ will contain more pixels than $D_h$ itself, and cannot possibly be a subset of the latter.

This proves that Guidefill terminates in at most $\lceil \sqrt{N}/2 \rceil$ iterations, and hence $K(N) = O(\sqrt{N})$.  
\end{proof}

\section{Numerical Experiments} \label{sec:numerical}

In this section we aim to validate our method as a practical tool for 3D conversion, and also to validate the complexity analysis of Section \ref{sec:complexity}.  We have implemented Guidefill in CUDA C and interfaced with MATLAB.  Our experiments were run on a laptop with a $3.28$GHz Intel i$7-4710$ CPU with $20$GB of RAM, and a GeForce GTX 970M GPU\footnote{The experiments involving nl-means and nl-Poisson are an exception.  Because the implementation available online does not support Windows, these experiments had to be done on a separate Linux machine with a $3.40$GHz Intel i$5-4670$ CPU with 16GB of RAM.  As a comparison, we measured the time to solve a $500 \times 500$ Poisson problem to a tolerance of $10^{-6}$ using the conjugate gradient method in MATLAB, which took $8.6s$ on our Windows laptop, and $5.2$s on the Linux box.}.

\vskip 2mm

\subsection{3D Conversion Examples}  

We tested our method on a number of HD problems, including the video illustrated in Figure \ref{fig:movie} and the four photographs shown in Figure \ref{fig:examples}.  The photographs were converted into 3D by building rough 3D geometry and creating masks for each object, as outlined in Section \ref{sec:pipeline}.  For the movie, we used a computer generated model with existing 3D geometry and masks\footnote{Downloaded from \url{http://www.turbosquid.com/} in accordance with the Royalty Free License agreement.}, as generating these ourselves on a frame by frame basis would have been far too expensive (indeed, in industry this is done by teams of artists and is extremely time consuming).  One advantage of this approach is that it gave us a ground truth to compare against, as in Figure \ref{fig:movie}(k).  Please see also the supplementary material where our results can be viewed in video form, and in anaglyph 3d (anaglyph glasses required).  Timings for Guidefill are given both with and without the boundary tracking as described in Section \ref{sec:GPU}.

As has been noted in the related work, the literature abounds with depth-guided variants of Criminisi's method \cite{Criminisi04regionfilling} designed for the disocclusion step arising in 3D conversion using the depth-map based pipeline discussed in Section \ref{sec:alternativePipelines}, see for example \cite{CriminisiVarient1,YoonEtAl,DepthAidedInpainingForNovelViewSynthesis,DepthGuidedInpaintingAlgorithmForFreeViewPointVideo,DepthImageBasedRenderingWithAdvancedTextureSynthesisfor3DVideo, OtherPipelineDoneWell}.  As we have already noted in the introduction, none of these methods are designed to make explicit use the bystander set set $B_h$ available to us and instead rely on heuristics.  In Section \ref{sec:alternativePipelines} Figure \ref{fig:careless2}, we have shown a simple example where with the exception of \cite{OtherPipelineDoneWell} these heuristics are likely fail.  Moreover, adapting these methods to our pipeline where depth-map inpainting is unnecessary would require considerable effort and find tuning.  Therefore, rather than comparing with these methods, we considered it more natural to compare with our own ``bystander-aware'' variant of Criminisi, adapted in an extremely simple way to incorporate the set $B_h$.  We simply modify Criminisi's algorithm by setting the priority equal to $0$ on $\partial D^{(k)}_h \backslash \partial_{\mbox{active}} D^{(k)}_h$ and restricting the search space to patches that do not overlap $\Omega_h \backslash (D_h \cup B_h)$.  However, we acknowledge that many of these methods also make further optimizations to Criminisi et al. from the point of view of running time - for example \cite{OtherPipelineDoneWell} incorporates the running time improvements originally published in their earlier work \cite{ExemplarNewHeurestics}.  We as well could have based our ``bystander-aware'' Criminisi on the improvement in \cite{ExemplarNewHeurestics}, however, instead we note that the running time published in \cite{OtherPipelineDoneWell} is about $1500$px/s, which is still much slower than Guidefill, especially for high resolution problems (see Table \ref{table:times}).

\begin{figure}
\centering
\begin{tabular}{cc}
\subfloat[Wine]{\includegraphics[width=.25\linewidth]{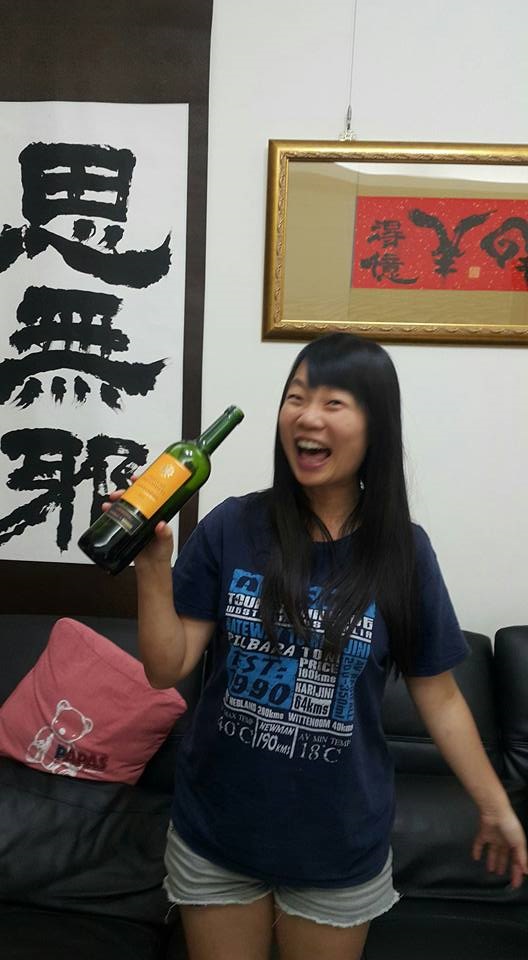}} &
\subfloat[Bust]{\includegraphics[width=.45\linewidth]{OriginalEye.jpg}} \\
\subfloat[Pumpkin]{\includegraphics[width=.45\linewidth]{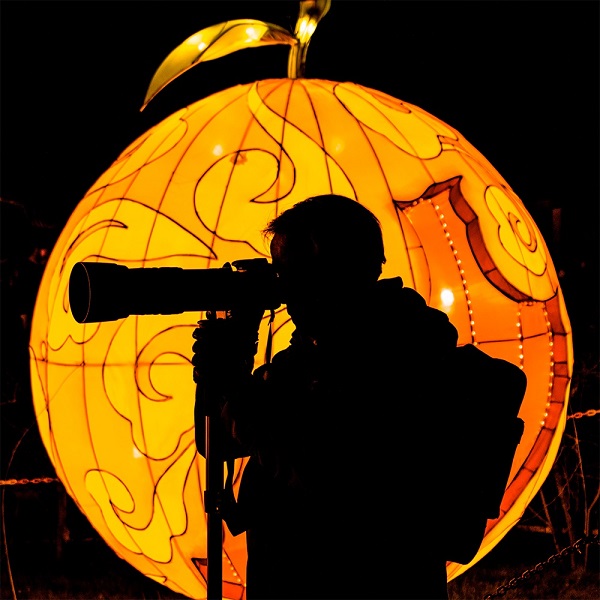}} &
\subfloat[Planet]{\includegraphics[width=.45\linewidth]{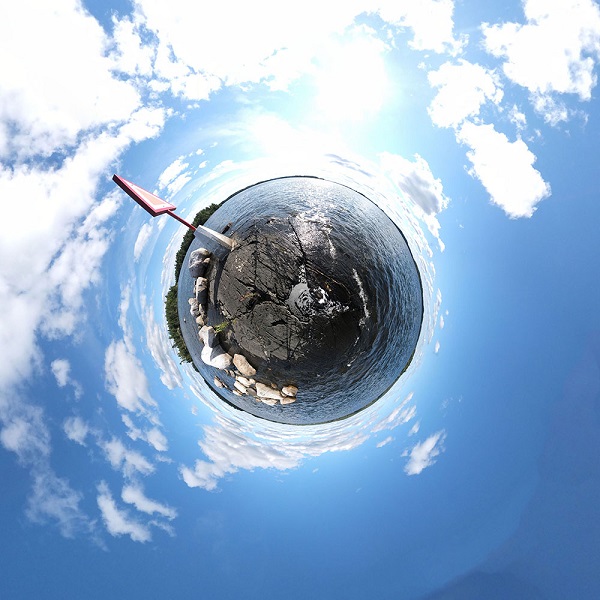}} \\
\end{tabular}
\caption{Example photographs used for 3D conversion, of different sizes  (a)  $528 \times 960$px  (b)  $1500 \times 1125$px (c)  $4000 \times 4000$px (d)  $5000 \times 5000$px.}
\label{fig:examples}
\end{figure}

For the photographs, we also compare the output of Guidefill with four other inpainting methods - coherence transport \cite{Marz2007,Marz2011}, the variational exemplar-based methods nl-means and nl-Poisson from Arias et al. \cite{Arias2011}, and Photoshop's Content-Aware fill.  For the movie, we compare with the exemplar-based video inpainting method of Newson et al. \cite{Newson2014,FastGeneric}.  However, generating ``bystander-aware'' versions of all of these methods would have been a significant undertaking, so we arrived at a compromise.  To avoid bleeding-artifacts like Figure \ref{fig:careless2}(c), we first ran each method using $D_h \cup B_h$ as the inpainting domain, giving the results shown.  However, as this led to an unfair running time due to the need to fill $B_h$, we then ran each method again using only $D_h$ as the inpainting domain, in order to obtain the given timings.  All methods are implemented in MATLAB + C (mex) and available for download online\footnote{Coherence transport: \url{http://www-m3.ma.tum.de/bornemann/InpaintingCodeAndData.zip}, Criminisi's method: \url{https://github.com/ikuwow/inpainting\_criminisi2004}, nl-means and nl-Poisson: \url{http://www.ipol.im/pub/art/2015/136/}, Newson's method: \url{http://perso.telecom-paristech.fr/~gousseau/video_inpainting/}.}.

Figure \ref{fig:movie} shows a few frames of a $1280\mbox{px} \times 960\mbox{px} \times 101 \mbox{fr} $ video, including the inpainting domain and the results of inpainting with both Guidefill and Newson's method.  With the exception of a few artifacts such as those visible in Figure \ref{fig:movie}(j), Newson's method produces excellent results.  However, it took $5$hr$37$min to run, and required more than $16$GB of RAM.  In comparison Guidefill produces a few artifacts, including the incorrectly completed window shown in Figure \ref{fig:movie}(e).  In this case the failure is because the one pixel wide ring described in Section \ref{sec:guidefield} fails to intersect certain edges we would like to extend.  However, Guidefill requires only $19$s (if boundary tracking is employed, $31$s if it is not) to inpaint the entire video and these artifacts can be corrected as in Figure \ref{fig:movie}(f).  However, due to the frame by frame nature of the computation, the results do exhibit some flickering when viewed temporally, an artifact which Newson's method avoids.  

Timings for the images are reported in Table \ref{table:times}, with the exception of Content-Aware fill which is difficult to time as we do not have access to the code.  We also do not provide timings for Bystander-Aware Criminisi, nl-means, and nl-Poisson for the ``Pumpkin'' and ``Planet'' examples as the former ran out of memory while nl-means and nl-Poisson did not finish within two hours.  However, for the ``Pumpkin'' example we do provide the result of nl-Poisson run on a small region of interest.  Results are given in Figures \ref{fig:wine}, \ref{fig:bust}, and \ref{fig:planet}.  We do not show the output of every method and have included only the most significant.

The first example, ``Wine'', is a $528 \times 960$px  photo.  Timings are reported only for the background object, which has an inpainting domain containing $15184$px.  Figure \ref{fig:splines} shows the detected splines for the background object and illustrates the editing process.  Results are then shown in Figure \ref{fig:wine} in two particularly challenging areas.  In this case the highest quality results are provided by nl-means and nl-Poisson, but both are relatively slow.  Bystander-Aware Criminisi and Content-Aware fill each produce noticeable artifacts.  Guidefill also has problems, most notably in the area behind the wine bottle, where the picture frame is extended incorrectly (this is due to a spline being too short) and where additional artifacts have been created next to the Chinese characters.  These problems however are mostly eliminated by lengthening the offending spline and editing some of the splines in the vicinity of Chinese characters as illustrated in Figure \ref{fig:splines}.  Guidefill is also the fastest method, although in this case the gains aren't as large as for bigger images.

The second example ``Bust'' is a $1500 \times 1125$px image.  Timings are reported only for inpainting the background object, which has an inpainting domain containing $111277$px, and results are shown in Figure \ref{fig:bust}(a)-(f).  In this case we chose to edit the automatically detected splines, in particular rotating one that was crooked.  Once again, the nicest result is probably nl-Poisson, but an extremely long computation time is required.  All other algorithms, including Bystander-Aware Criminisi and nl-means which are not shown, left noticeable artifacts.  The fully automatic version of Guidefill also leaves some artifacts, but these are largely eliminated by the adjustment of the splines.  The exception is a shock visible in the inpainted picture frame in Figure \ref {fig:bust}(f).  As we noted in Section \ref{sec:shocks}, shut artifacts are an unfortunate feature of the class of methods under consideration.

\begin{figure}
\centering
\begin{tabular}{ccc}
\subfloat[Frame 0 (pre inpainting).]{\includegraphics[width=.30\linewidth]{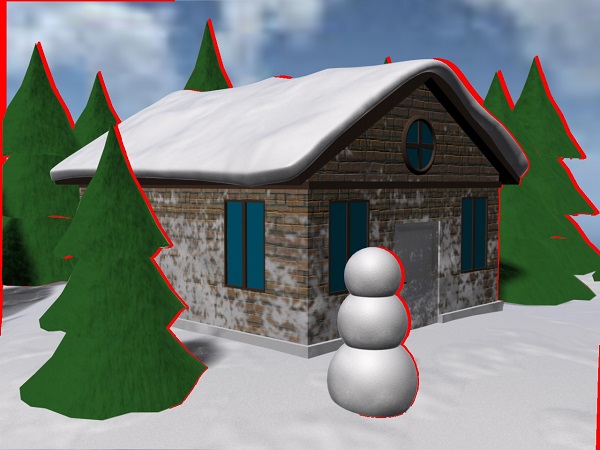}} &
\subfloat[Frame 26 (pre inpainting).]{\includegraphics[width=.30\linewidth]{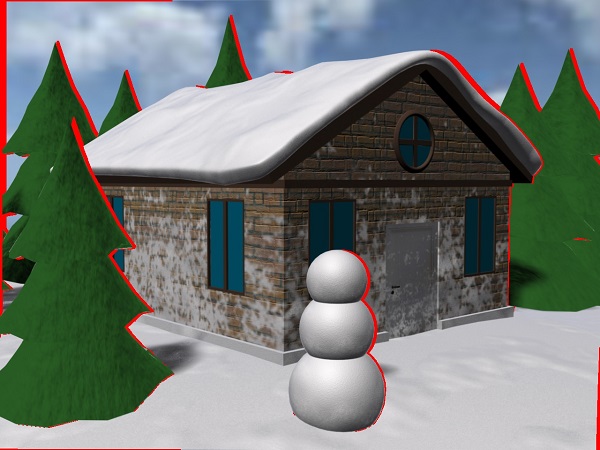}} &
\subfloat[Frame 100 (pre inpainting).]{\includegraphics[width=.30\linewidth]{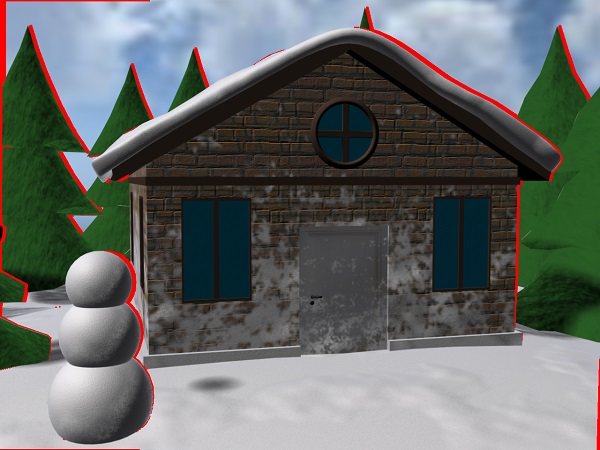}} \\
\end{tabular}
\begin{tabular}{cccc}
\subfloat[Frame 26 detail.]{\includegraphics[width=.216\linewidth]{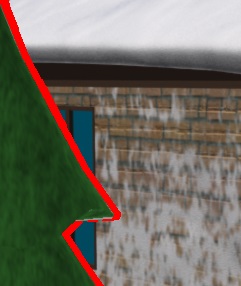}} &
\subfloat[Guidefill (pre edit).]{\includegraphics[width=.216\linewidth]{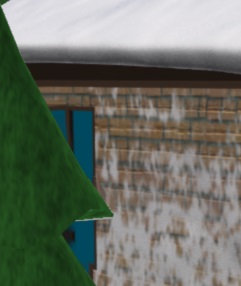}} &
\subfloat[Guidefill (post edit).]{\includegraphics[width=.216\linewidth]{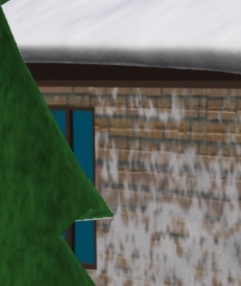}} &
\subfloat[Newson's Method.]{\includegraphics[width=.216\linewidth]{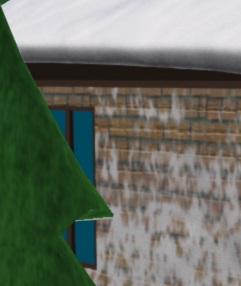}} \\
\subfloat[Frame 100 detai.l]{\includegraphics[width=.216\linewidth]{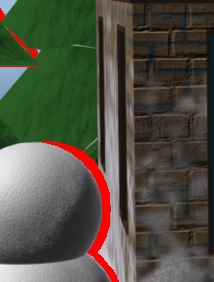}} &
\subfloat[Guidefill (no edit.)]{\includegraphics[width=.216\linewidth]{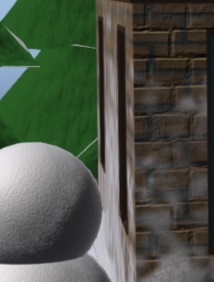}} &
\subfloat[Newson's Method.]{\includegraphics[width=.216\linewidth]{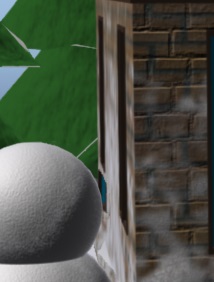}} &
\subfloat[Ground Truth.]{\includegraphics[width=.216\linewidth]{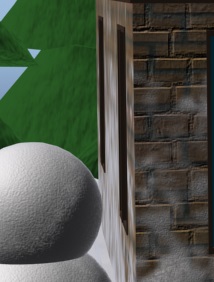}} \\
\end{tabular}
\caption{Comparison of Guidefill (19s with tracking, 31s without) and Newson's method (5hr37min) for inpainting the ``cracks'' (shown in red) arising in the 3d conversion of a HD video ($1280\mbox{px} \times 960\mbox{px} \times 101 \mbox{fr} $).  Guidefill produces artifacts such as the incorrectly extrapolated window in (e), but these can be corrected as in (f) and it is several orders of magnitude faster than Newson's method (which also required more than $16$GB of RAM in this case).  The latter produces produces very high quality results, but is prohibitively expensive and still produces a few artifacts as in (j), which the user has no recourse to correct.  A disadvantage of Guidefill is a flickering as the video is viewed through time due to the frames being inpainted independently.  The video is provided in the online supplementary material.}
\label{fig:movie}
\end{figure}

\begin{figure}
\centering
\begin{tabular}{ccc}
\subfloat[Detail one.]{\includegraphics[width=.3\linewidth]{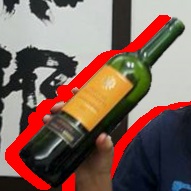}} &
\subfloat[Coherence transport - note the bending of the picture frame.]{\includegraphics[width=.3\linewidth]{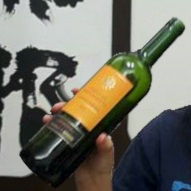}} &
\subfloat[nl-means - good result, but slow.]{\includegraphics[width=.3\linewidth]{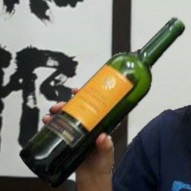}} \\
\subfloat[nl-Poisson - Chinese characters are a solid block.]{\includegraphics[width=.3\linewidth]{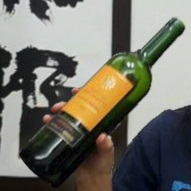}} &
\subfloat[Content-Aware Fill - distorted picture frame.]{\includegraphics[width=.3\linewidth]{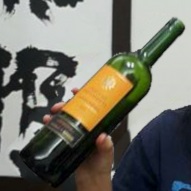}} &
\subfloat[Guidefill (before spline adjustment) - numerous issues.]{\includegraphics[width=.3\linewidth]{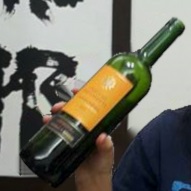}} \\
\subfloat[Guidefill (after adjustment) - issues are mostly resolved.]{\includegraphics[width=.3\linewidth]{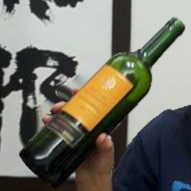}} &
\subfloat[Detail two.]{\includegraphics[width=.3\linewidth]{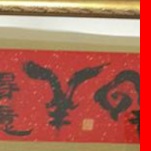}} &
\subfloat[Bystander-Aware Criminisi - a piece of the picture frame is used to extrapolate the drawing.]{\includegraphics[width=.3\linewidth]{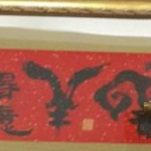}} \\
\subfloat[nl-Poisson - good result, but slow.]{\includegraphics[width=.3\linewidth]{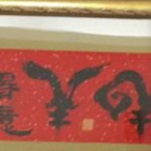}} &
\subfloat[Guidefill (before spline adjustment) - extension of drawing does not look natural.]{\includegraphics[width=.3\linewidth]{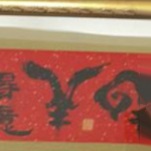}} &
\subfloat[Guidefill (after adjustment) - more believable extrapolation.]{\includegraphics[width=.3\linewidth]{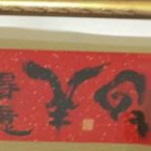}} \\
\end{tabular}
\caption{Comparison of different inpainting methods for the ``Wine'' example.  Two challenging areas are shown.}
\label{fig:wine}
\end{figure}

\begin{figure}
\centering
\begin{tabular}{ccc}
\subfloat[Detail of ``Bust''.]{\includegraphics[width=.3\linewidth]{GapsCrop.jpg}} &
\subfloat[Coherence transport.]{\includegraphics[width=.3\linewidth]{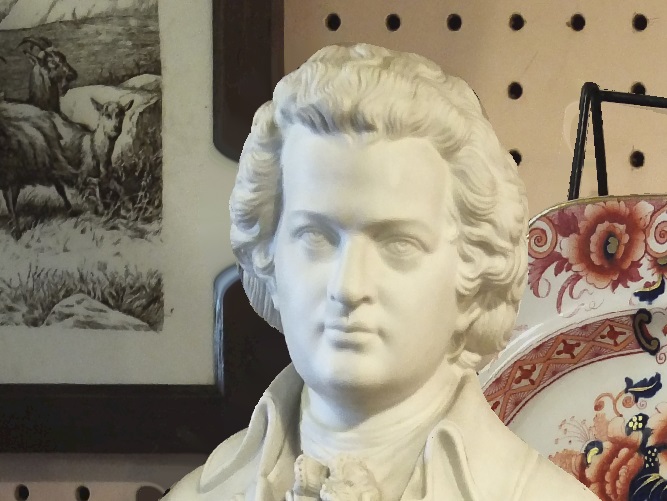}} &
\subfloat[Content-Aware Fill.]{\includegraphics[width=.3\linewidth]{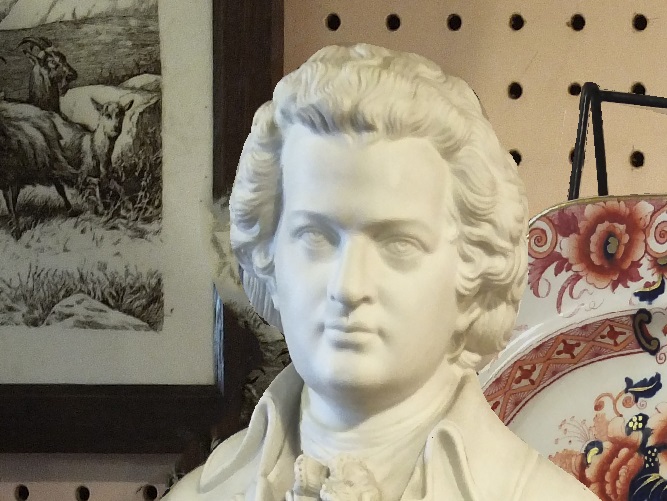}} \\
\subfloat[nl-Poisson]{\includegraphics[width=.3\linewidth]{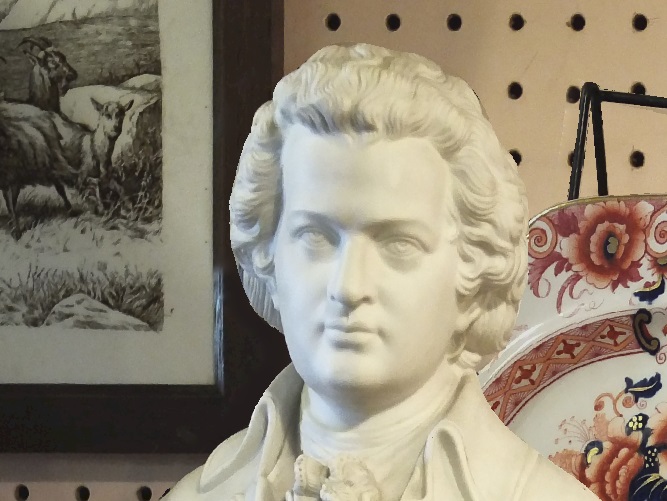}} &
\subfloat[Guidefill (before spline adjustment).]{\includegraphics[width=.3\linewidth]{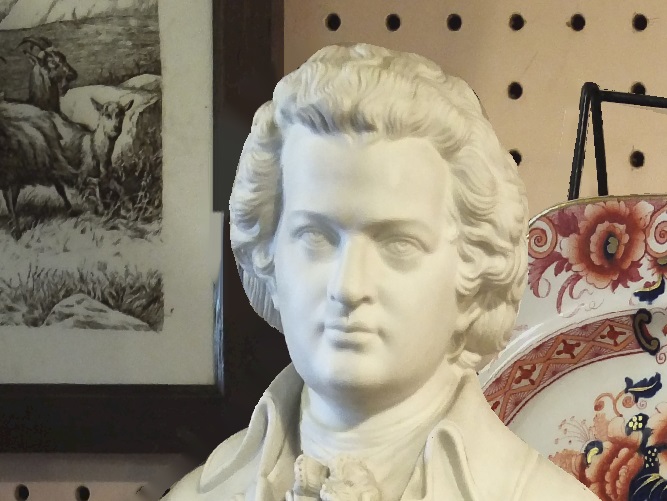}} &
\subfloat[Guidefill (after adjustment).]{\includegraphics[width=.3\linewidth]{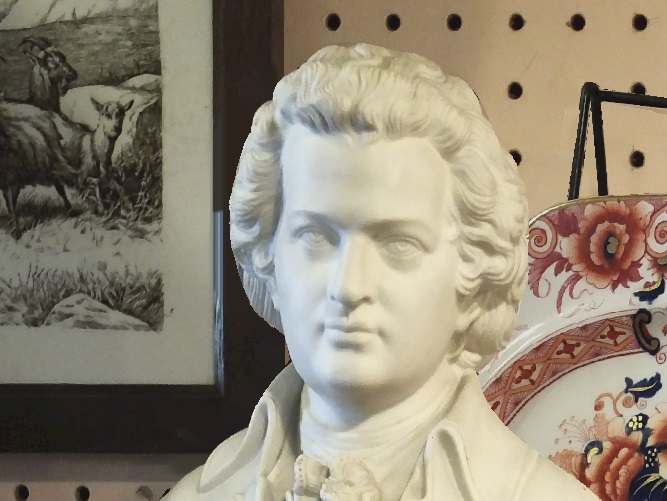}} \\
\subfloat[Detail of ``Pumpkin''.]{\includegraphics[width=.3\linewidth]{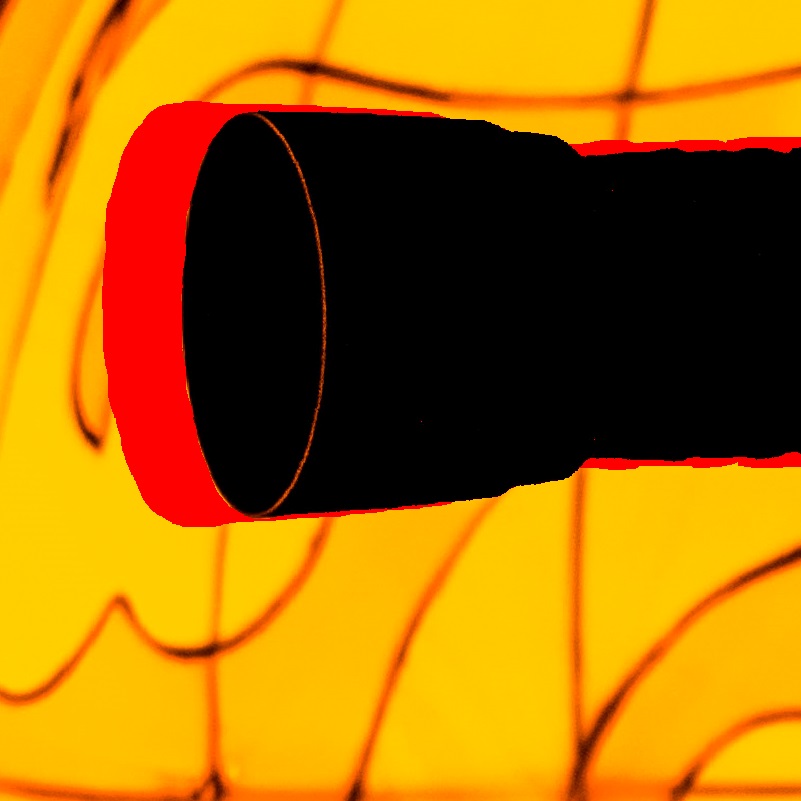}} &
\subfloat[coherence transport.]{\includegraphics[width=.3\linewidth]{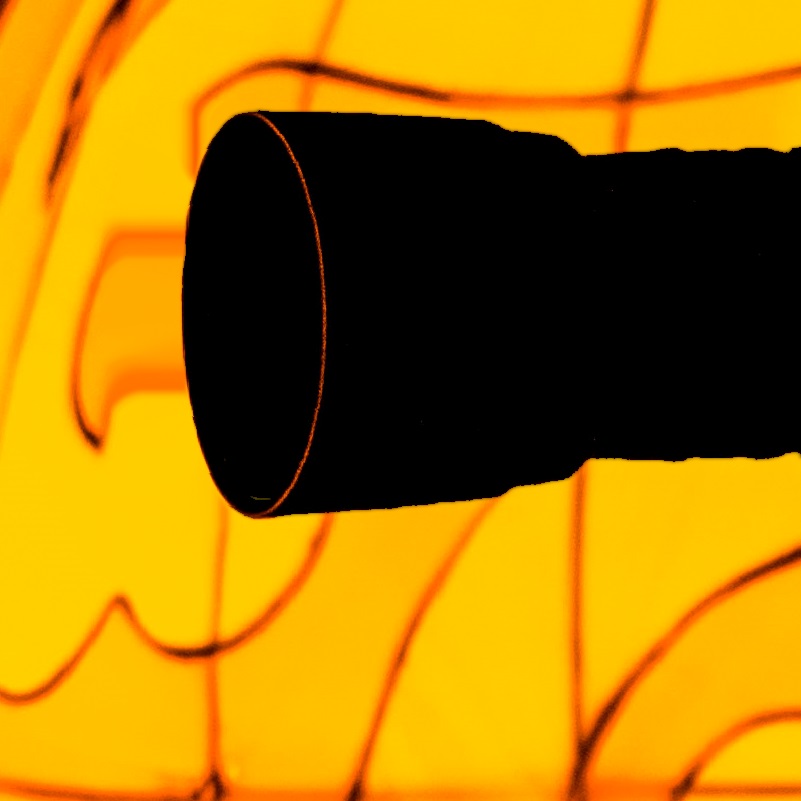}} &
\subfloat[Content-Aware Fill.]{\includegraphics[width=.3\linewidth]{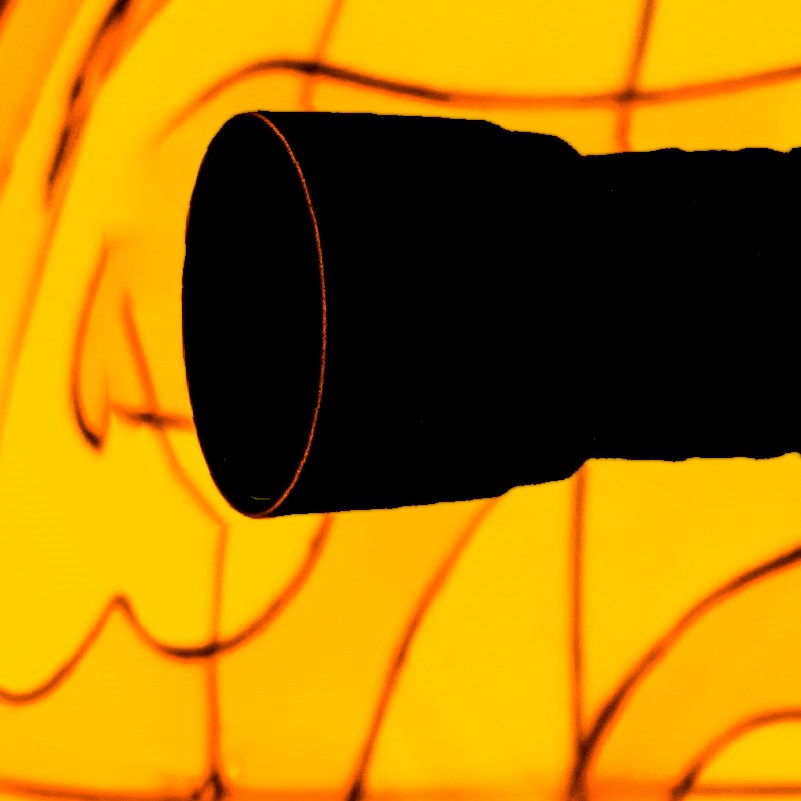}} \\
\subfloat[nl-Poisson.]{\includegraphics[width=.3\linewidth]{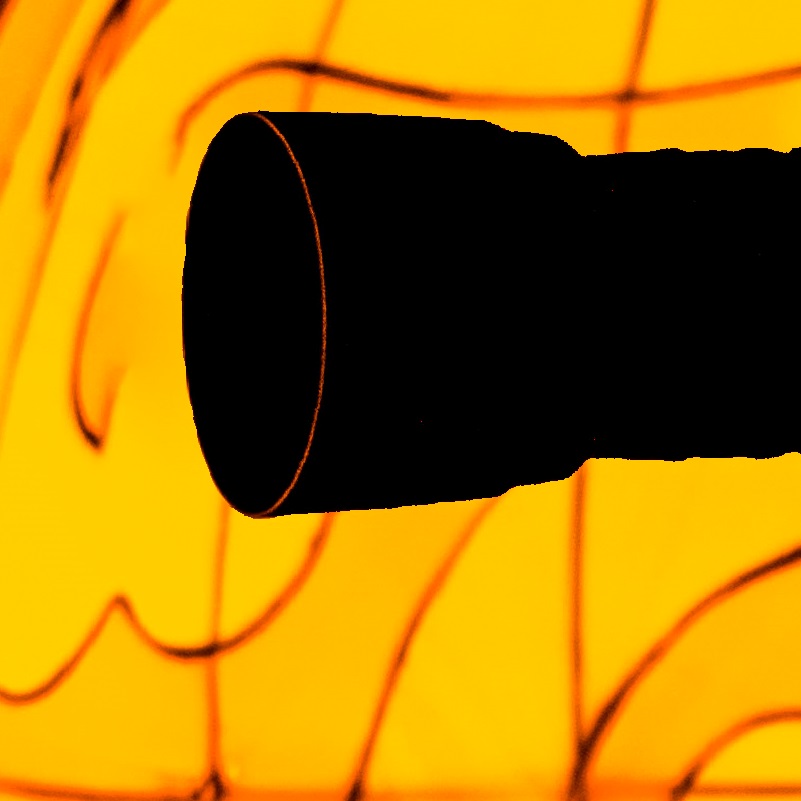}} &
\subfloat[Guidefill (before spline adjustment).]{\includegraphics[width=.3\linewidth]{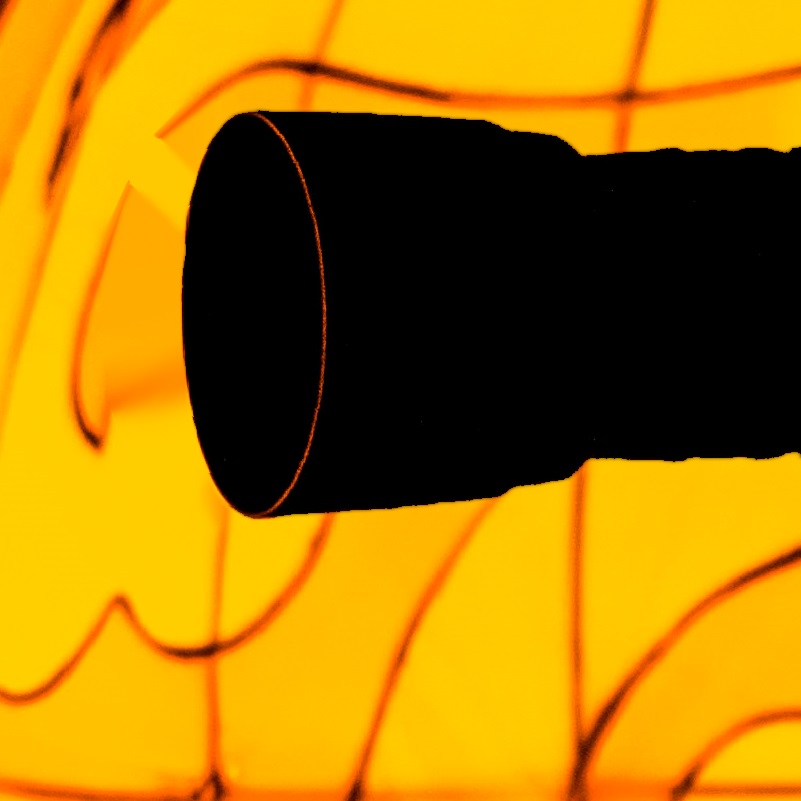}} &
\subfloat[Guidefill (after adjustment).]{\includegraphics[width=.3\linewidth]{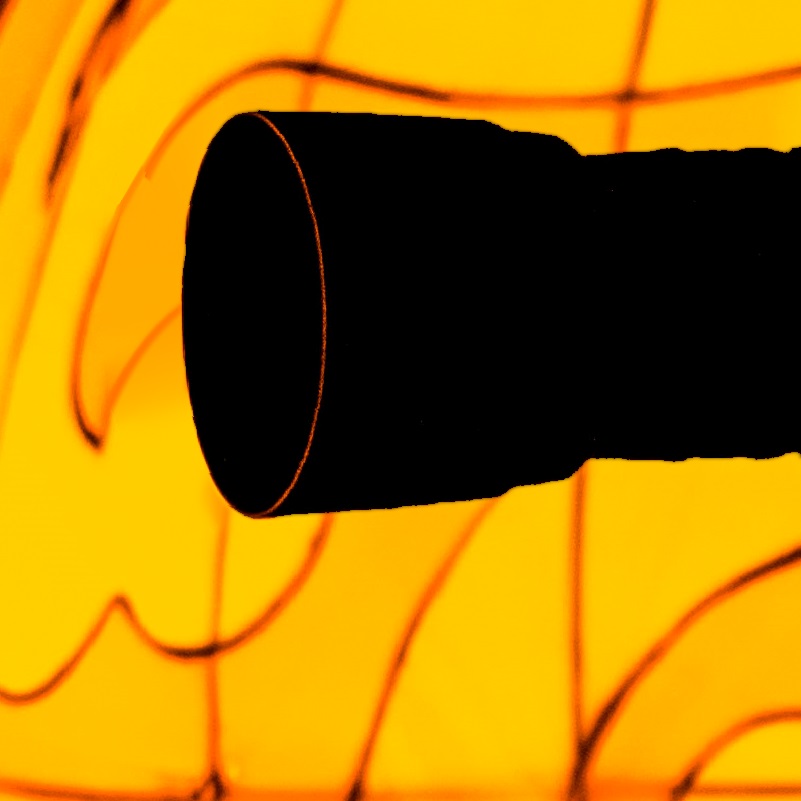}} \\
\end{tabular}
\caption{Comparison of different inpainting methods for the ``Bust'' and ``Pumpkin'' examples.  Note the shock visible in the inpainted picture frame in (f), as discussed in Section \ref{sec:shocks}.}
\label{fig:bust}
\end{figure}

\begin{table}
\caption{Timings of different inpainting algorithms used in the conversion of the three examples in Figure \ref{fig:examples}.  The inpainting domains of ``Wine'', ``Bust'', and ``Pumpkin'' and ``Planet'' contain $15184$px, $111277$px, $423549$px, and $1160899$px respectively.  ``Guidefill n.t.'' refers to Guidefill without boundary tracking, ``B.A.C.'' stands for Bystander-Aware Criminisi and ``C.T.'' refers to coherence transport.} 
\centering 
\begin{tabular}{|c | c | c | c | c | c | c |} 
\hline\hline 
 & C.T. & B.A.C. & nl-means & nl-Poisson & Guidefill n.t. & Guidefill \\ [0.5ex] 
\hline 
Wine & 340ms & 1 min 40s & 41s & 2min11s & {\bf 233ms} & 261ms \\ 
\hline 
Bust & 2.13s & 37min & 23min & 1hr 10min & 1.34s & {\bf 559ms} \\ 
\hline 
Pumpkin & 15.7s & $-$ & $-$ & $-$ & 6.66s & {\bf 1.14s} \\ 
\hline 
Planet & 28.5s & $-$ & $-$ & $-$ & 4.27s & {\bf 923ms} \\ 
\hline 
\end{tabular}
\label{table:times} 
\end{table}

\begin{figure}
\centering
\begin{tabular}{cccc}
\subfloat[]{\includegraphics[width=.21\linewidth]{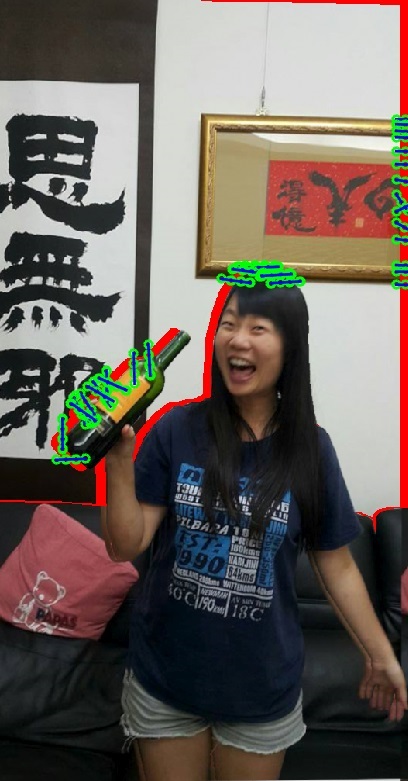}} &
\subfloat[]{\includegraphics[width=.21\linewidth]{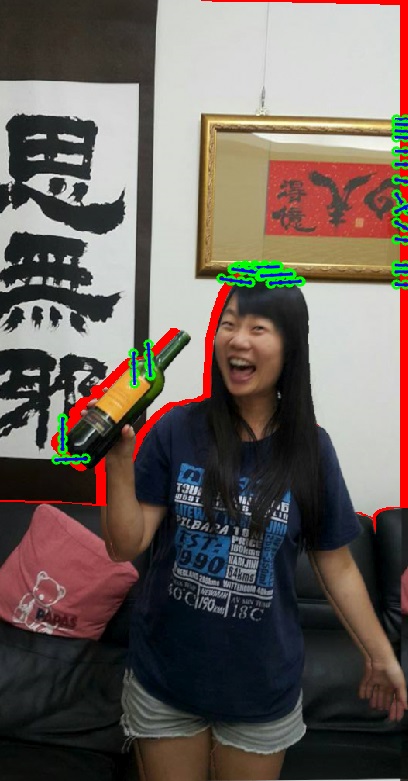}} &
\subfloat[]{\includegraphics[width=.21\linewidth]{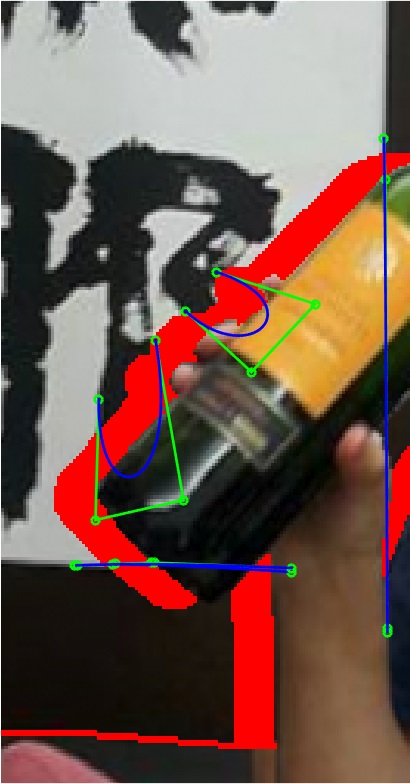}} &
\subfloat[]{\includegraphics[width=.21\linewidth]{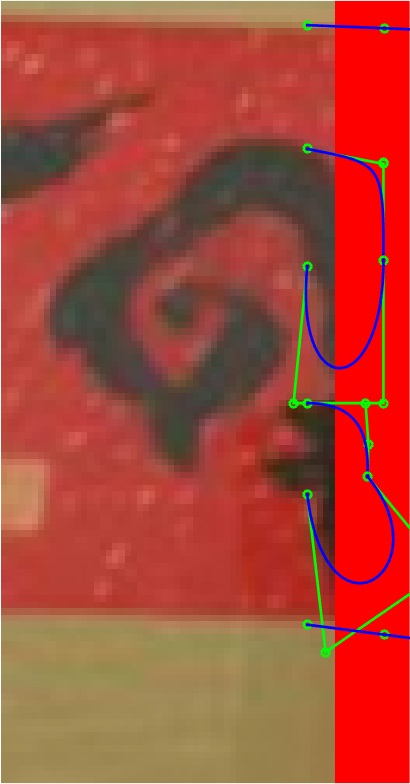}} \\
\end{tabular}
\caption{Stages of spline adjustment for the ``Wine'' example:  (a)  The automatically detected splines for the background object. (b) Undesirable splines are deleted.  (c) Deleted splines are replaced with new splines, drawn by hand, which form a plausible extension of the disoccluded characters.  (d)  Some of the remaining splines on the painting in the upper right corner are edited to form a more believable extension.}
\label{fig:splines}
\end{figure}

Our third example ``Pumpkin'' is a very large $4000 \times 4000$px image.  Timings are reported only for the pumpkin object, which has an inpainting domain containing $423549$px.  Results are shown in Figure \ref{fig:bust}(g)-(l).  We ran nl-Poisson on only the detail shown in Figure \ref{fig:bust}(g), because it did not finish within two hours when run on the image as a whole.  In this case we edited the automatically detected splines as shown in Figure \ref{fig:guidefield}(a)-(b).  In doing so we are able to recover smooth arcs that most fully automatic methods would struggle to produce.  Guidefill in this case is not only the fastest method by far, it also produces the nicest result.  In this example we also see the benefits of our boundary tracking algorithm, where it leads to a speed up of a factor of $2-3$.  The gains of boundary tracking are expected to be greater for very large images where the pixels greatly outnumber the available processors.

\begin{figure}[H]
\centering
\begin{tabular}{cccc}
\subfloat[detail of ``Planet''.]{\includegraphics[width=.21\linewidth]{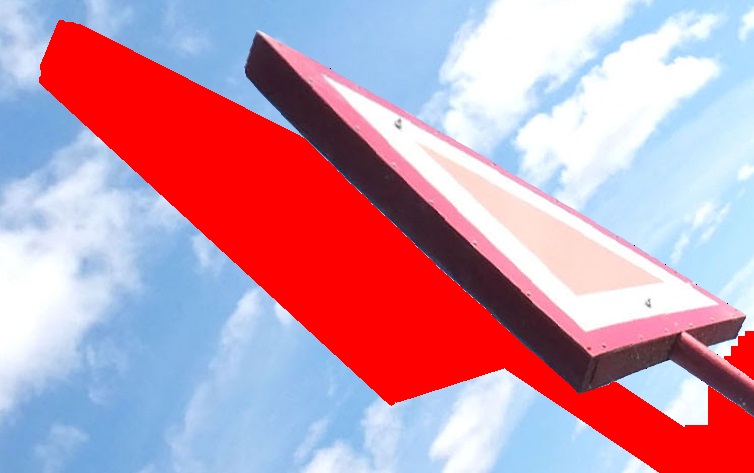}} &
\subfloat[Content-Aware Fill.]{\includegraphics[width=.21\linewidth]{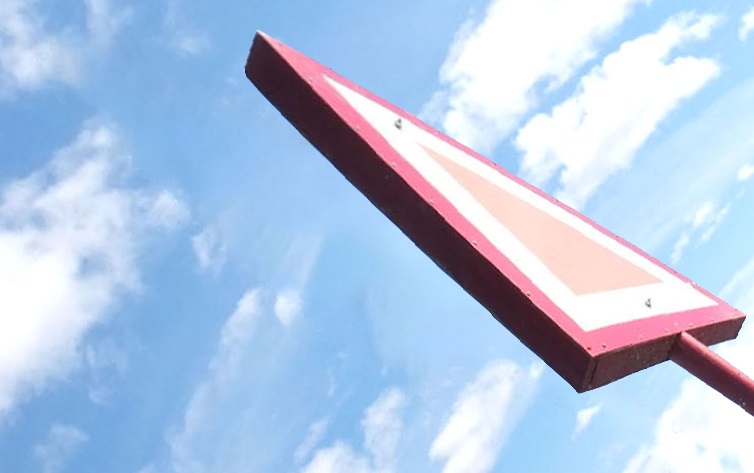}} &
\subfloat[coherence transport.]{\includegraphics[width=.21\linewidth]{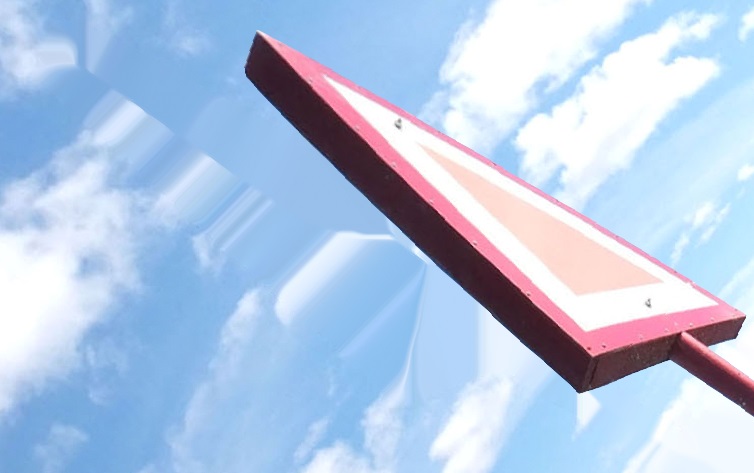}} &
\subfloat[Guidefill (no spline adjustment).]{\includegraphics[width=.21\linewidth]{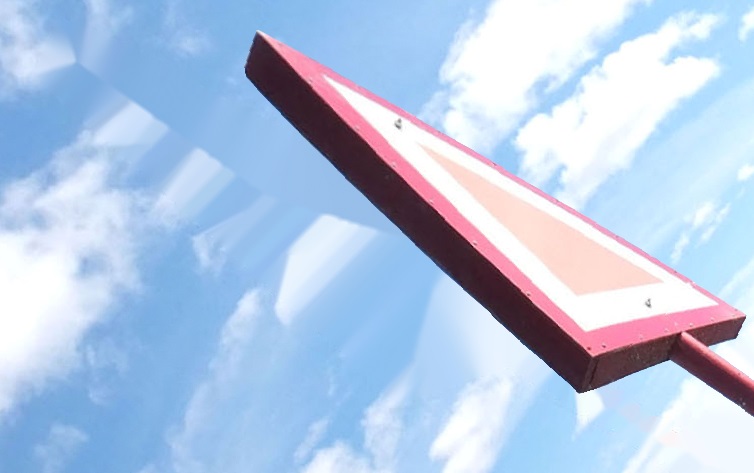}} \\
\end{tabular}
\caption{Comparison of different inpainting methods for the ``Planet'' example.  In this case geometric methods leave noticeable artifacts and exemplar-based methods like Content-Aware Fill are a better choice.}
\label{fig:planet}
\end{figure}

Our final example is a fun example that illustrates how 3D conversion may be used to create ``impossible'' 3D scenes.  In this case the image is a $5000 \times 5000$px ``tiny planet'' panorama generated by stitching together dozens of photographs.  The choice of projection creates the illusion of a planet floating in space - however, a true depth map would appear as an elongated finger, as in reality the center of the sphere is only a few feet from the camera, while its perimeter is at a distance of several kilometers. In order to preserve the illusion we created fake spherical 3D geometry.  See the online supplementary material for the full 3D effect - here we show only a detail in Figure \ref{fig:planet}.  In this example the inpainting domain is relatively wide and the image is dominated by texture.  As a result geometric methods are a bad choice and exemplar-based methods are more suitable.

\subsection{Validation of Complexity Analysis} \label{sec:complexityValidation}
As stated in Section \ref{sec:complexity}, our analysis assumes that Guidefill is implemented on a parallel architecture consisting of $p$ identical processors acting in parallel.  In reality, GPU architecture is more complex than this, but as a rough approximation we assume $p=20480$, the maximum number of resident threads allowed on our particular GPU.  See the online supplementary material for a deeper discussion.

In order to explore experimentally the time and processor complexity of Guidefill, we considered the continuum problem of inpainting the line $ 0.45 \leq y \leq 0.55$ across the inpainting domain $D = [0.4,3.96] \times [0.2, 0.8]$ with image domain $\Omega = [0,4] \times [0,1]$.  This continuum problem was then rendered at a series of resolutions varying from as low as $280 \times 70$px all the way up to $4000 \times 1000$px.  The resulting series of discrete inpainting problems were solved using Guidefill.  For simplicity, smart order was disabled and splines were turned off.  In each case we measured the execution time $T(N)$ of Guidefill as well as the maximum number of requested threads $P(N)$, with and without tracking.  Results are shown in Figure \ref{fig:complexityGraphs} - note the loglog scale.  In Figure \ref{fig:complexityGraphs}(b) we have also indicated the value of $p$ for comparison - note that for Guidefill without tracking we have $P(N) \gg p$ for all but the smallest problems, but for Guidefill with tracking we have $P(N) < p$ up until $N \approx 2 \times 10^5$.

\begin{figure}[H]
\centering
\begin{tabular}{cc}
\subfloat[Time Complexity]{\includegraphics[width=.45\linewidth]{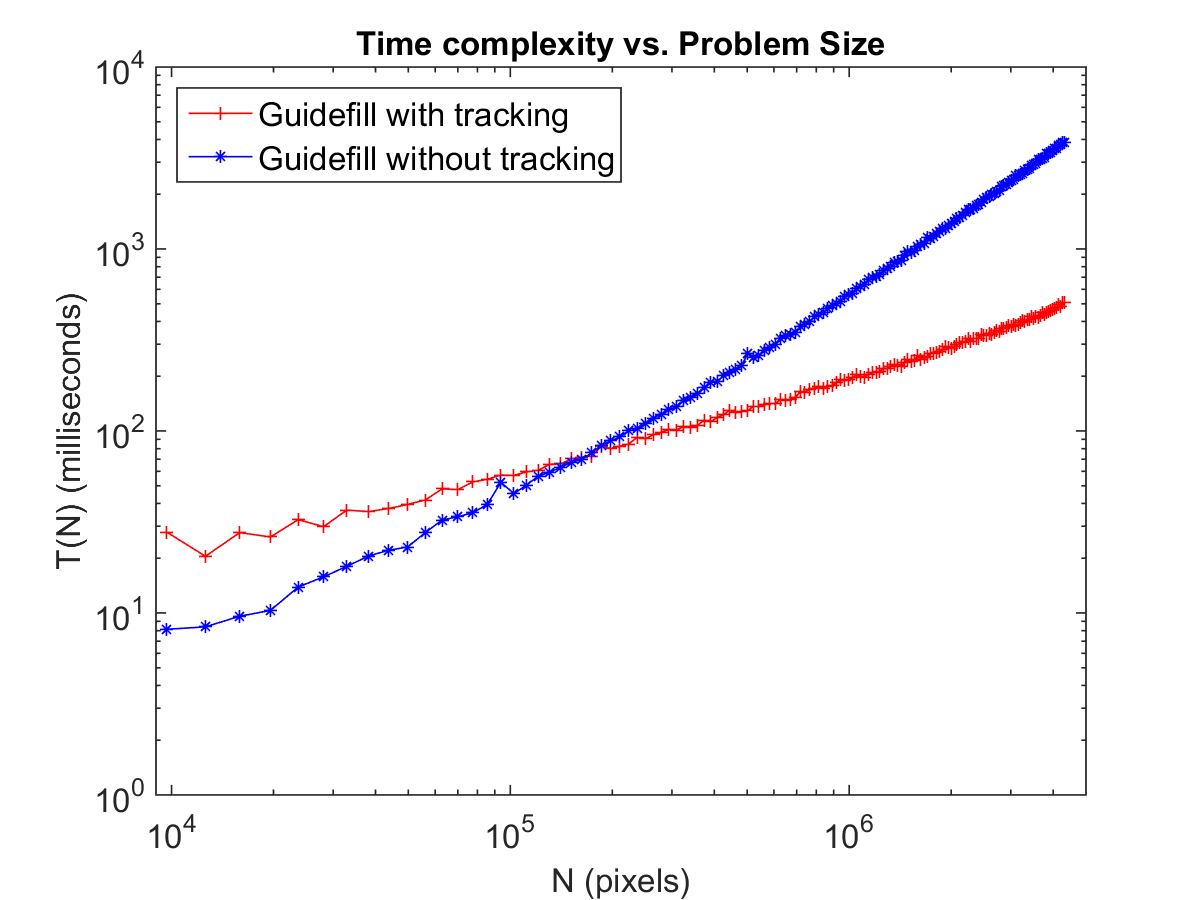}} &
\subfloat[Processor Complexity]{\includegraphics[width=.45\linewidth]{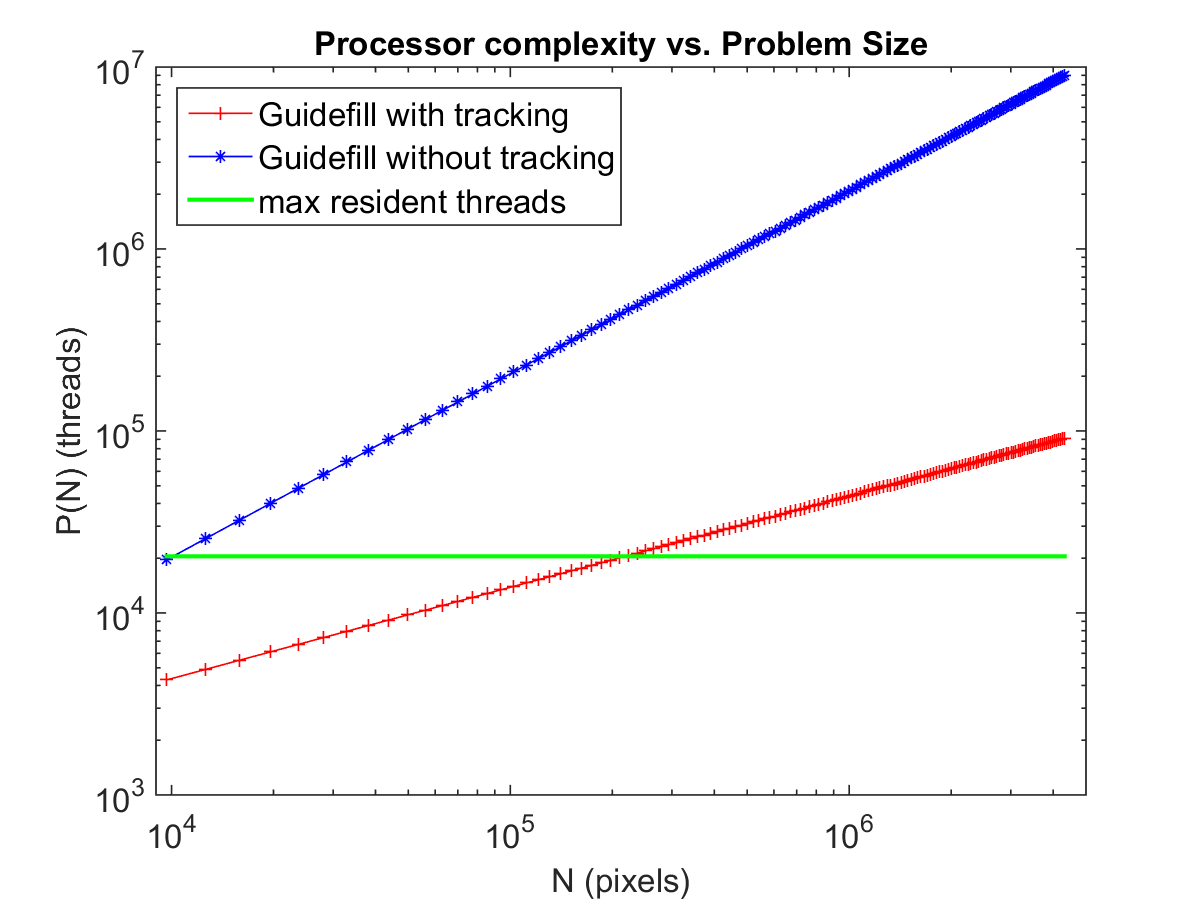}} \\
\end{tabular}
\caption{{\bf Experimental time complexity $T(N)$ and processor complexity $P(N)$ of Guidefill with and without boundary tracking:}  The continuum inpainting problem
$\Omega = [0,4] \times [0,1]$, $D = [0.4, 3.96] \times [0.2, 0.8]$ was discretized at a variety of resolutions leading to inpainting domains with $N:=|D_h|$ varying from
$N\approx 10^4$px up to $N\approx 10^6$px.  Results are given on a loglog scale to emphasize the approximate power law $T(N) \approx AN^{\alpha}$, $P(N) \approx BN^{\beta}$.  A least squares fit gives $\alpha = 1.1$, $\beta = 1.0$ without tracking, and $\alpha = 0.54$, $\beta = 0.5$ with tracking (see Section \ref{sec:GPU} for a review of these terms).  The superior scaling law of Guidefill with tracking kicks in around $N \approx 2\cdot 10^5$. Processor complexity is compared with the maximum number of resident threads (green line).}
\label{fig:complexityGraphs}
\end{figure}

Based on Theorem \ref{thm:complexity} and Theorem \ref{thm:complexity2} for Guidefill without tracking we expect $T(N) \in O(N^{1.5})$ for all $N$, but for Guidefill with tracking we expect \\$T(N) \in O(N^{0.5} \log(N))$ for $N$ up to about $10^5$px (where $P(N) \approx p$), with somewhat worse performance as $N$ grows larger, converging to $O(N\log(N))$ when $P(N) \gg p$.  To test these expectations we assume a power law of $T(N) \approx A N^{\alpha}$ and solve for $\alpha$ using least squares.  The results are $\alpha = 0.54$ and $\alpha = 1.10$ for Guidefill with and without tracking respectively.  Assuming a similar power law $P(N) \approx B N^{\beta}$ gives $\beta = 1.0$, $\beta = 0.5$ for Guidefill without and with tracking respectively.  These results suggest that the analysis in Section \ref{sec:complexity} do a reasonable job of predicting the rough behaviour of our method in practice.  

\section{Conclusions} \label{sec:conclusion}

We have presented a fast inpainting method suitable for use in the hole-filling step of a 3D conversion pipeline used in film, which we call Guidefill. Guidefill is non-texture based, exploiting the fact that the inpainting domains in 3D conversion tend to be in the form of a thin ``crack'' such that texture can often be neglected. Its fast processing time and its setup allowing intuitive, user-guided amendment of the inpainting result render Guidefill into a user-interactive inpainting tool. A version of Guidefill is in use by the stereo artists at the 3D conversion company Gener8, where it has been used in major Hollywood blockbusters such as Mockingjay, Pan, and Maleficent.  In those cases where it is suitable, especially scenes dominated by structure rather than texture and/or thin inpainting domains, Guidefill produces results that are competitive with alternative algorithms in a tiny fraction of the time. In practice, Guidefill was found to be particularly useful for movies with many indoor scenes dominated by structure, and less useful for movies taking place mainly outdoors, where texture dominates. Because of its speed, artists working on a new scene may apply our method first. If the results are unsatisfactory, they can edit the provided splines or switch to a more expensive method.

In addition to its use as an algorithm for 3D conversion, Guidefill belongs to a broader class of fast geometric inpainting algorithms also including Telea's Algorithm \cite{Telea2004} and coherence transport \cite{Marz2007, Marz2011}.  Similarly to these methods, Guidefill is based on the idea of filling the inpainting domain in shells while extrapolating isophotes based on a transport mechanism.  However, Guidefill improves upon these methods in several important respects including the elimination of two forms of kinking of extrapolated isophotes. In one case this is done by summing over a non-axis aligned balls of ``ghost pixels'', which as far as we know have never been done in the literature.  

We have also presented a theoretical analysis of our method and methods like it, by considering a relevant continuum limit.  Our limit, which is different from the one explored in \cite[Theorem 1]{Marz2007}, is able to theoretically explain some of the advantages and disadvantages of both our method and coherence transport.  In particular, our analysis predicts a kinking phenomena observed in coherence transport in practice but not accounted for by the analysis in \cite{Marz2007}.  It is also able to explain how our ghost pixels are able to fix this problem, but also sheds light on a new problem that they introduce - the progressive blurring of an extrapolated signal due to repeated bilinear interpolation operations.  Nonetheless, our analysis predicts that this latter effect becomes less and less significant as the image resolution increases, and our method is designed with HD in mind.  More details of our analytic framework are explored in our forthcoming paper \cite{Forthcoming}

In order to make our method as fast as possible we have implemented it on the GPU where we consider two possible implementations. A naive implementation, suitable for small images, simply assigns one GPU thread per pixel.  For our second implementation, we propose an algorithm to track the inpainting interface as it evolves, facilitating a massive reduction in the number of threads required by our algorithm. This does not lead to speed up by a constant factor - rather, it changes the complexity class of our method, leading to improvements that become arbitrarily large as $N=|D_h|$ increases. In practice we observed a slight decrease in speed (compared with the naive implementation) for small images ($N \lesssim 10^5$px), and gains ranging from a factor of $2-6$ for larger images. 

A current disadvantage of our method is that, in order to keep execution times low, temporal information is ignored. In particular, splines are calculated for each frame separately, and inpainting is done on a frame by frame basis without consideration for temporal coherence. As a result of the former, artists must perform separate spline adjustments for every frame. In practice we find that only a minority of frames require adjustment, however one potential direction for improvement is to design a system that proposes a series of {\em animated} splines to the user, which they may then edit over time by adjusting control points and setting key frames. Secondly, a procedure for enforcing temporal coherence, if it could be implemented without significantly increasing the runtime, would be beneficial. However, these improvements are beyond the scope of the present work.

\section{Acknowledgements}
The authors are grateful to the insightful comments of three anonymous referees, which greatly improved the clarity of the text.  LRH acknowledges support from the Cambridge Commonwealth Trust and the Cambridge Center for Analysis.  CBS acknowledges support from Leverhulme Trust project on Breaking the non-convexity barrier, EPSRC grant Nr. EP/M00483X/1, the EPSRC Centre Nr. EP/N014588/1 and the Cantab Capital Institute for the Mathematics of Information.

\bibliographystyle{siamplain}
\bibliography{SIAMBib2}
\end{document}

%% file: SIAM_Revised_shared.tex
\usepackage{algorithm}
\usepackage{algorithmic}

\usepackage{subfig}
\usepackage{subfloat}
\usepackage{graphicx}
\usepackage{amsfonts,amsmath,yfonts}
\usepackage{amssymb}
\usepackage{wrapfig}
\usepackage{color}
\usepackage[bordercolor=white,backgroundcolor=gray!30,linecolor=black,colorinlistoftodos]{todonotes}
\usepackage{hyperref}
\usepackage{MnSymbol}

\newcommand{\field}[1]{\mathbb{#1}}


\newcounter{parentnumber}

\newtheorem{thm}{Theorem}[section]
\newtheorem{rem}[thm]{Remark}

\usepackage{epsfig}

  \newcommand{\bi}{\begin{itemize}}
  \newcommand{\ei}{\end{itemize}}
  \newcommand{\be}{\begin{enumerate}}
  \newcommand{\ee}{\end{enumerate}}
  \newcommand{\beq}{\begin{equation}}
  \newcommand{\eeq}{\end{equation}}
  \newcommand{\beqa}{\begin{eqnarray}}
  \newcommand{\eeqa}{\end{eqnarray}}

\numberwithin{equation}{section}

\newcommand{\TheTitle}{Guidefill:  GPU Accelerated, Artist Guided Geometric Inpainting for 3D Conversion of Film} 
\newcommand{\TheAbbreviatedTitle}{Guidefill} 
\newcommand{\TheAuthors}{L. Robert Hocking, Russell MacKenzie, and Carola-Bibiane Sch\"onlieb}

\headers{\TheAbbreviatedTitle}{\TheAuthors}

\title{{\TheTitle}\thanks{Submitted to the editors \today.
\funding{This work was supported by the Cambridge Commonwealth Trust and the Cambridge Center for Analysis.}}}

\author{
  L. Robert Hocking\thanks{Department of Applied Mathematics and Theoretical Physics, University of Cambridge
    (\email{lrh30@cam.ac.uk}).}
    \and
  Russell MacKenzie\thanks{Gener8 Media corp, (\email{russell@gener8.com}).}
  \and
  Carola-Bibiane Sch\"onlieb\thanks{Department of Applied Mathematics and Theoretical Physics, University of Cambridge, (\email{cbs31@cam.ac.uk}).}
}

\usepackage{amsopn}
